\def\year{2022}\relax
\documentclass[letterpaper]{article} 
\usepackage{aaai22}  
\usepackage{times}  
\usepackage{helvet}  
\usepackage{courier}  
\usepackage[hyphens]{url}  
\usepackage{graphicx} 
\usepackage{natbib}  
\usepackage{caption} 
\DeclareCaptionStyle{ruled}{labelfont=normalfont,labelsep=colon,strut=off} 
\usepackage{algorithm}
\usepackage{algorithmic}

\usepackage{newfloat}
\usepackage{listings}
\floatstyle{ruled}
\newfloat{listing}{tb}{lst}{}
\floatname{listing}{Listing}
\usepackage{hyperref}

\usepackage{savesym,multirow,url,xspace,graphicx,hyperref,enumitem,color}
\usepackage[titletoc]{appendix}
\usepackage{subcaption}
\usepackage{dsfont}
\usepackage{multicol}

\usepackage{etoolbox}

\usepackage{tikz}
\usetikzlibrary{automata, positioning, arrows}
\tikzset{
	->, 
	every state/.style={thick, fill=gray!10}, 
	initial text=$ $, 
}
\usetikzlibrary{arrows.meta}
\usepackage{subcaption}
\usepackage{pgfplots}
\pgfplotsset{width=10cm,compat=1.9}
\usepackage{diagbox}
\usepackage{caption}

\usepackage{nameref}

\def\mycomments{1}
\if\mycomments1
\newcommand{\adish}[1]{{\textcolor{blue}{{\bf AS:} #1}}}
\newcommand{\kiarash}[1]{{\textcolor{cyan}{{\bf KB:} #1}}}
\newcommand{\goran}[1]{{\textcolor{green}{{\bf GR:} #1}}}
\newcommand{\jiarui}[1]{{\textcolor{purple}{{\bf JG:} #1}}}
\newcommand{\todo}[1]{{\textcolor{red}{{\bf TODO:} #1}}}
\else
\newcommand{\adish}[1]{\iffalse{\textcolor{blue}{{\bf AS:} #1}}\fi}
\newcommand{\kiarash}[1]{\iffalse{\textcolor{cyan}{{\bf KB:} #1}}\fi}
\newcommand{\goran}[1]{\iffalse{\textcolor{green}{{\bf GR:} #1}}\fi}
\newcommand{\jiarui}[1]{\iffalse{\textcolor{purple}{{\bf JG:} #1}}\fi}
\newcommand{\todo}[1]{\iffalse{\textcolor{red}{{\bf TODO:} #1}}\fi}
\fi
\makeatletter
\newif\if@restonecol
\makeatother

\usepackage{amsmath,amssymb,xfrac,amsthm}
\usepackage{mathtools}
\setitemize{noitemsep,topsep=1pt,parsep=1pt,partopsep=1pt, leftmargin=12pt}
\newtheorem{lemma}{Lemma}
\newtheorem{theorem}{Theorem}
\newtheorem*{theorem*}{Theorem}

\newtheorem{corollary}{Corollary}
\newtheorem{proposition}{Proposition}

\newtheorem{remark}{Remark}
\newtheorem{claim}{Claim}
\makeatletter

\newcommand{\Rmnum}[1]{\expandafter\@slowromancap\romannumeral #1@}
\makeatother
\usepackage{caption}

\usepackage{enumitem}
\usepackage{wasysym}

\DeclareMathOperator*{\argmin}{arg\,min}
\DeclareMathOperator*{\argmax}{arg\,max}

\newcommand{\opt}{\textsc{{Opt}}\xspace}

\newcommand{\qgreedy}{\textsc{QGreedy}}

\newcommand{\coopalgo}{\textsc{Constrain\&Optimize}}

\newcommand{\adm}{\textnormal{adm}}

\usepackage{bm}
\usepackage{bbm}

\newcommand{\norm}[1]{\left\lVert#1\right\rVert}

\newcommand{\vecdot}[2]{\left<#1, #2\right>}

\newcommand{\expct}[1]{\mathbb{E}\left[#1\right]}
\newcommand{\expctu}[2]{\mathbb{E}_{#1}\left[#2\right]}
\renewcommand{\Pr}[1]{\ensuremath{\mathbb{P}\left[#1\right] }}

\newcommand{\ind}[1]{\mathds{1}\left[#1\right]}
\newcommand{\pos}[1]{\left[#1\right]^+}
\newcommand{\CasesIf}{\quad\text{if}\quad}
\newcommand{\CasesOW}{\quad\text{o.w.}\quad}

\newcommand{\targetpi}{\ensuremath{{\pi_{\dagger}}}}

\newcommand{\score}{{\rho}}

\newcommand{\occstate}{{\mu}}

\newcommand{\optpi}{\pi^{*}}

\newcommand{\optallowpi}{\pi^{*}_{\textnormal{adm}}}

\newcommand{\allowedpi}{\Pi^{\textnormal{adm}}_{\textnormal{det}}}
\newcommand{\optEpsDet}{\opt^{\epsilon}_{\textnormal{det}}}

\newcommand{\PiDet}{\Pi_{\text{det}}}
\newcommand{\RsolA}{\widehat R_1}
\newcommand{\RsolB}{\widehat R_2}
\newcommand{\pisolB}{\pi_2}

\newcommand{\relval}{\Phi}
\newcommand{\maxgapqval}{{\Delta_Q^{\pi}}}
\newcommand{\minmaxgapqval}{{\Delta_Q}}
\newcommand{\scoregap}{{\Delta_{\score}^{\pi}}}
\newcommand{\minscoregap}{{\Delta_{\score}}}
\newcommand{\qgapqval}{{\delta}}

\newcommand{\posStates}{S_{\textnormal{pos}}}

\newcommand{\mincostallowpi}{\pi_{\textnormal{min}_c}}

\newcommand{\piqg}{\pi_{\textnormal{qg}}}

\newcommand{\pico}{\pi_{\textnormal{co}}}

\usepackage[utf8]{inputenc} 
\usepackage[T1]{fontenc}    
\usepackage{hyperref}       
\usepackage{url}            
\usepackage{booktabs}       
\usepackage{amsfonts}       
\usepackage{nicefrac}       
\usepackage{microtype}      
\usepackage{xcolor}         

\title{Admissible Policy Teaching through Reward Design}
\author {
    Kiarash Banihashem, 
    Adish Singla, 
    Jiarui Gan, 
    Goran Radanovic
}
\affiliations {
    Max Planck Institute for Software Systems\\
    \{kbanihas, adishs, jrgan, gradanovic\}@mpi-sws.org
}

\begin{document}

\maketitle

\newtoggle{longversion}
\settoggle{longversion}{true}
\begin{abstract}
\looseness-1We study reward design strategies for incentivizing a reinforcement learning agent to adopt a policy from a set of admissible policies. The goal of the reward designer is to modify the underlying reward function cost-efficiently while ensuring that any approximately optimal deterministic policy under the new reward function is admissible and performs well under the original reward function. This problem can be viewed as a dual to the problem of optimal reward poisoning attacks: instead of forcing an agent to adopt a specific policy, the reward designer incentivizes an agent to avoid taking actions that are inadmissible in certain states. Perhaps surprisingly, and in contrast to the problem of optimal reward poisoning attacks, we first show that the reward design problem for admissible policy teaching is computationally challenging, and it is NP-hard to find an approximately optimal reward modification.  We then proceed by formulating a surrogate problem whose optimal solution approximates the optimal solution to the reward design problem in our setting, but is more amenable to optimization techniques and analysis. For this surrogate problem, we present characterization results that provide bounds on the value of the optimal solution. Finally, we design a local search algorithm to solve the surrogate problem and showcase its utility using simulation-based experiments.
\end{abstract}

\section{Introduction}
\looseness-1Reinforcement learning (RL) \cite{sutton2018reinforcement} is a framework for deriving an agent's policy that maximizes its utility in sequential decision making tasks. In the standard formulation, the utility of an agent is defined via its reward function, which determines the decision making task of interest. Reward design plays a critical role in providing sound specifications of the task goals and supporting the agent's learning process~\cite{singh2009rewards,amodei2016concrete}.

There are different perspectives on reward design, which differ in the studied objectives. A notable example of reward design is {\em reward shaping} \cite{mataric1994reward,dorigo1994robot,ng1999policy} which modifies the reward function in order to accelerate the learning process of an agent.
Reward transformations that are similar to or are based on reward shaping are not only used for accelerating learning. For example, reward penalties are often used in safe RL to penalize the agent whenever it violates safety constraints \cite{tessler2018reward}. Similarly, reward penalties can be used in offline RL for ensuring robustness against model uncertainty \cite{yu2020mopo}, while exploration bonuses can be used as intrinsic motivation for an RL agent to reduce uncertainty \cite{bellemare2016unifying}. 

In this paper, we consider a different perspective on reward design, and study it in the context of {\em policy teaching} and closely related (targeted) {\em reward poisoning attacks}. In this line of work~\cite{DBLP:conf/aaai/ZhangP08,zhang2009policy,ma2019policy,rakhsha2020policy,rakhsha2020policy-jmlr}, the reward designer perturbs the original reward function to influence the choice of policy adopted by an optimal agent. For instance, \cite{DBLP:conf/aaai/ZhangP08,zhang2009policy} studied policy teaching from a principal's perspective who provides incentives to an agent to influence its policy. In reward poisoning attacks~\cite{ma2019policy,rakhsha2020policy,rakhsha2020policy-jmlr}, an attacker modifies the reward function with the goal of forcing a specific target policy of interest. Importantly, the reward modifications do not come for free, and the goal in this line of work is to alter the original reward function in a cost-efficient manner. The associated cost can, e.g., model the objective of minimizing additional incentives provided by the principal or ensuring the stealthiness of the attack.

The focus of this paper is on a dual problem to reward poisoning attacks. Instead of forcing 
a specific target policy, the reward designer's goal is to incentivize an agent to avoid taking actions that are inadmissible in certain states, while ensuring that the agent performs well under the original reward function. As in reward poisoning attacks, the reward designer cares about the cost of modifying the original reward function. Interestingly and perhaps surprisingly, the novel reward design problem leads to a considerably different characterization results, as we show in this paper. We call this problem {\em admissible policy teaching} since the reward designer aims to maximize the agent's utility w.r.t. the original reward function, but under constraints on admissibility of state-action pairs. 
These constraints could encode additional knowledge that the reward designer has about the safety and security of executing certain actions. Our key contributions are:
\begin{itemize}
\item We develop a novel optimization framework based on Markov Decision Processes (MDPs) for finding a minimal reward modifications which ensure that an optimal agent adopts a well-performing admissible policy.
\item We show that finding an optimal solution to the reward design problem for admissible policy teaching is computationally challenging, in particular, that it is NP-hard to find a solution that approximates the optimal solution.
\item We provide characterization results for a surrogate problem whose optimal solution approximates the optimal solution to our reward design problem. For a specific class of MDPs, which we call \emph{special} MDPs, we present an exact characterization of the optimal solution. For \emph{general} MDPs, we provide bounds on the optimal solution value.
\item We design a local search algorithm for solving the surrogate problem, and demonstrate its efficacy using simulation-based experiments.
\end{itemize}

\subsection{Related Work}

\paragraph{Reward design.} 
A considerable number of works is related to designing reward functions that improve an agent's learning procedures. The optimal reward problem focuses on finding a reward function that can support computationally bounded agents \cite{sorg2010internal,sorg2010reward}.
Reward shaping \cite{mataric1994reward,dorigo1994robot}, and in particular, potential-based reward shaping \cite{ng1999policy} and its extensions (e.g., \cite{devlin2012dynamic,grzes2017reward,zou2019reward}) densify the reward function so that the agent receives more immediate signals about its performance, and hence learns faster.
As already mentioned, similar reward transformations, such as reward penalties or bonuses, are often used for reducing uncertainty or for ensuring safety constraints \cite{bellemare2016unifying, yu2020mopo, tessler2018reward}.
Related to safe and secure RL are works that study reward specification problem and negative side affects of reward misspecification \cite{amodei2016concrete,hadfield2017inverse}. 
The key difference between the above papers and our work is that we
focus on policy teaching rather than on an agent's learning procedures.

\paragraph{Teaching and steering.} As already explained, our work relates to prior work on policy teaching and targeted reward poisoning attacks~\cite{DBLP:conf/aaai/ZhangP08,zhang2009policy,ma2019policy,DBLP:conf/gamesec/HuangZ19a,rakhsha2020policy,rakhsha2020policy-jmlr,xuezhou2020adaptive,sun2020vulnerability}. Another line of related work is on designing steering strategies. For example, \cite{nikolaidis2017game,dimitrakakis2017multi,radanovic2019learning} consider two-agent collaborative settings where a dominant agent can exert influence on the other agent, and the goal is to design a policy for the dominant agent that accounts for the imperfections of the other agent. Similar support mechanisms based on providing advice or helpful interventions have been studied by \cite{amir2016interactive,omidshafiei2019learning, tylkin2021learning}. In contrast, we consider steering strategies based on reward design. When viewed as a framework for supporting an agent's decision making, this paper is also related to works on designing agents that are robust against adversaries \cite{pinto2017robust, fischer2019online, lykouris2019corruption, zhang2020robust, zhang2021robust, zhang2021robustICML,banihashem2021defense}. These works focus on agent design and are complementary to our work on reward design.

\section{Problem Setup}\label{sec.setting}

In this section we formally describe our problem setup. 

\subsection{Environment} 

The environment in our setting is described by a discrete-time Markov Decision Process (MDP) $M = (S, A, R, P, \gamma, \sigma)$, where $S$ and $A$ are the discrete finite state and action spaces respectively\footnote{This setting can encode the case where states have different number of actions (e.g., by adding actions to the states with smaller number of actions and setting the reward of newly added state-action pairs to $-\infty$).},
$R: S \times A \rightarrow \mathds R$ is a reward function, $P:S \times A \times S \rightarrow [0, 1]$ specifies the transition dynamics with $P(s, a, s')$ denoting the probability of transitioning to state $s'$ from state $s$ by taking action $a$, $\gamma \in [0, 1)$ is the discounted factor, and $\sigma$ is the initial state distribution. A deterministic policy $\pi$ is a mapping from states to actions, i.e., $\pi:S \rightarrow A$, and the set of all deterministic policies is denoted by $\PiDet$.

Next, we define standard quantities in this MDP, which will be important for our analysis. First, we
define the
{\em score} of a policy $\pi$
as the
total expected return scaled by $1-\gamma$,
i.e., 
$\score^{\pi, R} = \expct{ 
      (1-\gamma)\sum_{t=1}^{\infty} \gamma^{t-1} R(s_t, a_t) | \pi, \sigma}$.
Here states $s_t$ and actions $a_t$ are obtained by executing policy $\pi$ starting from state $s_1$, which is sampled from the initial state distribution $\sigma$. 
Score $\score^{\pi, R}$ can be obtained through state occupancy measure $\occstate^{\pi}$ by using 
the equation $\score^{\pi, R} = \sum_{s} \occstate^{\pi}(s) \cdot R(s, \pi(s))$.
Here, $\occstate^{\pi}$ is the expected discounted state visitation frequency when $\pi$ is executed, given by $\occstate^{\pi}(s) = \expct{ (1-\gamma)\sum_{t=1}^{\infty} \gamma^{t-1}  \ind{s_t = s} | \pi, \sigma}$.
Note that $\occstate^{\pi}(s)$ can be equal to $0$ for some states. Furthermore, we define $\occstate^{\pi}_{\textnormal{min}} = \min_{s|\occstate^{\pi}(s)>0} \occstate^{\pi}(s)$---the minimum always exists due to the finite state and action spaces.
Similarly, we denote by $\occstate_{\textnormal{min}}$ the minimal value of $\occstate^{\pi}_{\textnormal{min}}$  across all deterministic policies, i.e., $\occstate_{\textnormal{min}} = \min_{\pi \in \PiDet} \occstate^{\pi}_{\textnormal{min}}$.

We define the state-action value function, or $Q$ values as $Q^{\pi, R}(s, a) = \expct{\sum_{t=1}^{\infty} \gamma^{t-1} R(s_t, a_t) | \pi, s_1 = s, a_1 = a}$,
where states $s_t$ and actions $a_t$ are obtained by executing policy $\pi$ starting from state $s_1 = s$ in which action $a_1 = a$ is taken. State-action values $Q(s,a)$ relate to score $\rho$ via the equation $\score^{\pi, R} = \expctu{s \sim \sigma}{(1-\gamma) \cdot Q^{\pi, R}(s, \pi(s))}$,
where the expectation is taken over possible starting states.

\subsection{Agent and Reward Functions}

We consider a reinforcement learning agent whose behavior is specified by a deterministic policy $\pi$, derived offline using an MDP model given to the agent. We assume that the agent selects a deterministic policy that (approximately) maximizes the agent's expected utility under the MDP model specified by a reward designer. In other words, given an access to the MDP $M=(S, A, R, P, \gamma, \sigma)$, the agent chooses a policy from the set $\optEpsDet(R) = \{\pi\in \PiDet:  \score^{\pi, R} > \max_{\pi' \in \PiDet} \score^{\pi', R} - \epsilon \}$,
 where $\epsilon$ is a strictly positive number. It is important to note that the MDP model given to the agent might be different from the true MDP model of the environment. In this paper, we focus on the case when only the reward functions of these two MDPs (possibly) differ. 

Therefore, in our notation, we differentiate the reward function that the reward designer specifies to the agent, denoting it by $\widehat R$, from the original reward function of the environment, denoting it by $\overline{R}$. A generic reward function is denoted by $R$, and is often used as a variable in our optimization problems. 
We also denote by $\optpi$ a deterministic policy 
that is optimal with respect to $\overline{R}$ for any starting state, i.e.,
$Q^{\optpi, \overline{R}}(s, \optpi(s))
= \max_{\pi\in \PiDet} Q^{\pi, \overline{R}}(s, \pi(s))$ for all states $s$.

\subsection{Reward Designer and Problem Formulation}

We take the perspective of a reward designer whose goal is 
to design a reward function, $\widehat R$, such that the agent adopts a policy  from a class of admissible deterministic policies $\allowedpi \subseteq \PiDet$. Ideally, the new reward function $\widehat R$ would be close to the original reward function $\overline{R}$, thus reducing  the cost of the reward design. At the same time, the adopted policy should perform well under the original reward function $\overline{R}$, since this is the performance that the reward designer wants to optimize and represents the objective of the underlying task.
As considered in related works~\cite{ma2019policy,rakhsha2020policy,rakhsha2020policy-jmlr}, we measure the {\em cost} of the reward design by $L_2$ distance between the designed $\widehat R$ and the original reward function $\overline R$. Moreover, we measure the agent's performance with the score $\score^{\pi, \overline R}$, where the agent's policy $\pi$ is obtained w.r.t. the designed reward function $\widehat R$.  Given the model of the agent discussed in the previous subsection, and assuming the worst-case scenario (w.r.t. the tie-breaking in the policy selection), the following optimization problem specifies the reward design problem for \emph{admissible policy teaching} (APT):
\begin{align*}
\label{prob.reward_design}
\tag{P1-\textsc{APT}}
\notag
&\quad
\min_{R}\max_{\pi}
\norm{\overline{R}-R}_{2} -\lambda\cdot \score^{\pi, \overline{R}}
\\
&\quad \mbox{ s.t. }  \quad
\optEpsDet (R) \subseteq \allowedpi
\\&\quad \quad \quad\quad
\pi \in\optEpsDet(R),
\end{align*}
where $\lambda \ge 0$ is a trade-off factor. 
While in this problem formulation $\allowedpi$ can be any set of policies, we will primarily focus on admissible policies that can be described by a set of admissible actions per state. More concretely, we define sets of admissible actions per state denoted by $A^{\adm}_s \subseteq A$. Given these sets $A^{\adm}_s$, the set of admissible policies will be identified as $\allowedpi = \{ \pi | \pi(s) \in A^{\adm}_s \lor \occstate^{\pi}(s) = 0 \mbox{ for } s \in S \}$.\footnote{In practice, we can instead put the constraint that $\occstate^{\pi}(s)$ is greater than or equal to some threshold. For small enough threshold, our characterization results qualitatively remain the same. } 
In other words, these policies must take admissible actions for states that have non-zero state occupancy measure.

We conclude this section by validating the soundness of the optimization problem \eqref{prob.reward_design}. The following proposition shows that the optimal solution to the optimization problem \eqref{prob.reward_design} is always attainable.
\begin{proposition}\label{prop.solvability.reward_design}
If $\allowedpi$ is not empty, there always exists an optimal solution to the optimization problem \eqref{prob.reward_design}.
\end{proposition}
In the following sections, we analyze computational aspects of this optimization problem, showing that it is intractable in general and providing characterization results that bound the value of solutions. 
The proofs of our results are provided in the full version of the paper.

\section{Computational Challenges}\label{sec.computational_challenges}

We start by analyzing computational challenges behind the optimization problem \eqref{prob.reward_design}. 
To provide some intuition, let us first analyze a special case of \eqref{prob.reward_design} where $\lambda = 0$, which reduces to the following optimization problem:
\begin{align*}
    \label{prob.reward_design_simple}
    \tag{P2-$\textsc{APT}_{\lambda=0}$}
    \notag
    &\quad
    \min_{R}
    \norm{\overline{R}-R}_{2}\\
    &\quad \mbox{ s.t. }  \quad
    \optEpsDet (R) \subseteq \allowedpi.
\end{align*}
This special case of the optimization problem with $\lambda=0$ is a generalization of the reward poisoning attack from \cite{rakhsha2020policy,rakhsha2020policy-jmlr}. In fact, the reward poisoning attack of \cite{rakhsha2020policy-jmlr} can be written as
\begin{align*}
    \label{prob.reward_poisoning_attack}
    \tag{P3-\textsc{ATK}}
    \notag
    &\quad
    \min_{R}
    \norm{\overline{R}-R}_{2}\\
    &\quad \mbox{ s.t. }  \quad
    \optEpsDet (R) \subseteq \{ \pi | \pi(s) = \targetpi(s) \mbox { if } \mu^{\targetpi}(s) > 0\},
\end{align*}
\looseness-1where $\targetpi$ is the target policy that the attacker wants to force. However, while \eqref{prob.reward_poisoning_attack} is tractable in the setting of \citep{rakhsha2020policy-jmlr},
the same is not true for \eqref{prob.reward_design_simple}; see Remark~\ref{rm.reward_poisoning}. Intuitively, the difficulty of solving the optimization problem \eqref{prob.reward_design_simple} lies in the fact that the policy set $\allowedpi$ (in the constraints of \eqref{prob.reward_design_simple}) can contain exponentially many policies. Since the optimization problem \eqref{prob.reward_design_simple} is a specific instance of the optimization problem \eqref{prob.reward_design}, the latter problem is also computationally intractable. We formalize this result in the following theorem.
\begin{theorem}\label{thm.copmutational_hardness}
For any constant $p \in (0,1)$, it is NP-hard to distinguish between instances of \eqref{prob.reward_design_simple} that have optimal values at most $\xi$ and instances that have optimal values larger than $\xi \cdot \sqrt{(|S| \cdot |A|)^{1-p}}$.
The result holds even when the parameters $\epsilon$ and $\gamma$ in \eqref{prob.reward_design_simple} are fixed to arbitrary values subject to $\epsilon > 0$ and $\gamma \in (0,1)$.
\end{theorem}

The proof of the theorem
is based on 
a classical NP-complete problem called {\sc Exact-3-Set-Cover} (X3C) \cite{karp1972reducibility,garey1979computers}.
The result implies that it is unlikely (assuming that P = NP is unlikely) that there exists a polynomial-time algorithm that always outputs an approximate solution whose cost is at most
$\sqrt{(|S| \cdot |A|)^{1-p}}$
times that of the optimal solution for some $p > 0$.

We proceed by introducing a surrogate problem \eqref{prob.reward_design.approx}, which is more amenable to optimization techniques and analysis since the focus is put on optimizing over policies rather than reward functions. In particular, the optimization problem takes the following form:
\begin{align*}
\label{prob.reward_design.approx}
\tag{P4-\textsc{APT}}
\notag
&\quad
\min_{\pi \in \allowedpi, R} \norm{\overline{R}-R}_{2} - \lambda\cdot \score^{\pi, \overline{R}}
\\&\quad\quad\mbox{ s.t. } 
\quad\quad \optEpsDet (R) \subseteq \{ \pi' | \pi'(s) = \pi(s) \mbox { if } \mu^{\pi}(s) > 0\}.
\end{align*}
Note that \eqref{prob.reward_design.approx} differs from  \eqref{prob.reward_poisoning_attack} in that it optimizes over all admissible policies, and it includes performance considerations in its objective. 
The result in Theorem~\ref{thm.copmutational_hardness} extends to this case as well, so the main computational challenge remains the same. 

The following proposition
shows that the solution to the surrogate problem \eqref{prob.reward_design.approx} is an approximate solution to the optimization problem \eqref{prob.reward_design}, with an additive bound. More precisely: 
\begin{proposition}\label{prop.reward_design.approx}
	Let $\RsolA$ and $\RsolB$ be the optimal solutions to
	\eqref{prob.reward_design} and \eqref{prob.reward_design.approx} respectively and let $l(R)$ be a function that outputs the objective of the optimization problem
	\eqref{prob.reward_design}, i.e.,
	\begin{align}\label{eq.objective_opt_reward_design}
	l(R)= \max_{\pi\in \optEpsDet(R)}\norm{\overline{R} - R}_{2} - \lambda\score^{\pi, \overline{R}}.
	\end{align}
	Then $\RsolB$ satisfies the constraints of \eqref{prob.reward_design}, i.e, $\optEpsDet(\RsolB) \subseteq \allowedpi$, and
	\begin{align*}
	l(\RsolA) \le l(\RsolB) \le l(\RsolA) + \frac{\epsilon}{\occstate_{\min}}\cdot\sqrt{|S|\cdot|A|}.
	\end{align*}
\end{proposition}
Due to this result, in the following sections, we focus on the optimization problem \eqref{prob.reward_design.approx}, and provide characterization for it. Using Proposition \ref{prop.reward_design.approx} we can obtain analogous results for the optimization problem \eqref{prob.reward_design}.

\begin{remark}\label{rm.reward_poisoning}
The optimization problem \eqref{prob.reward_poisoning_attack} is a strictly more general version of the optimization problem studied in \cite{rakhsha2020policy-jmlr} since \eqref{prob.reward_poisoning_attack} does not require $\occstate^{\pi}(s) > 0$ for all $\pi$ and $s$. This fact also implies that the algorithmic approach presented in \cite{rakhsha2020policy-jmlr} is not applicable in our case. hm with provable guarantees.
We 
provide an efficient algorithm for finding an approximate solution to \eqref{prob.reward_poisoning_attack} with provable guarantees in the full version of the paper.
\end{remark}

\section{Characterization Results for Special MDPs}\label{sec.special_mdps}

In this section, we consider a family of MDPs where an agent's actions do not influence transition dynamics, or more precisely, all the actions influence transition probabilities in the same way. In other words, the transition probabilities satisfy $P(s, a, s') = P(s, a', s')$,
for all $s$, $a$, $a'$, $s'$. 
We call this family of MDPs \emph{special} MDPs, in contrast to \emph{general} MDPs that are studied in the next section. Since an agent's actions do not influence the future, the agent can reason myopically when  deciding on its policy. Therefore, the reward designer can also treat each state separately when reasoning about the cost of the reward design.
Importantly, the hardness result from the previous section does not apply for this instance of our setting, so we can efficiently solve the optimization problems \eqref{prob.reward_design} and \eqref{prob.reward_design.approx}.

\subsection{Forcing Myopic Policies}

We first analyze the cost of forcing a target policy $\targetpi$ in special MDPs. The following lemma plays a critical role in our analysis. 
\begin{lemma}\label{lm.special_mdp.cost_of_poisoning}
Consider a {\em special} MDP with reward function $\overline{R}$, and let $\optallowpi(s) = \argmax_{a \in \allowedpi}\overline{R}(s,a)$.
Then the cost of the optimal solution to the optimization problem \eqref{prob.reward_poisoning_attack} with $\targetpi = \optallowpi$ is less than or equal to the cost of the optimal solution to the optimization problem \eqref{prob.reward_poisoning_attack} with $\targetpi = \pi$ for any $\pi \in \allowedpi$.
\end{lemma}
In other words, Lemma \ref{lm.special_mdp.cost_of_poisoning} states that in special MDPs it is easier to force policies that are  myopically optimal (i.e., optimize w.r.t. the immediate reward) than any other policy in the admissible set $\allowedpi$. This property is important for the optimization problem \eqref{prob.reward_design.approx} since its objective includes the cost of forcing an admissible policy.

\subsection{Analysis of the Reward Design Problem}

We now turn to the reward design problem \eqref{prob.reward_design.approx} and provide characterization results for its optimal solution.
Before stating the
result,
we note that for special MDPs $\occstate^{\pi}(s)$ is independent of policy $\pi$, so we denote it by $\occstate(s)$. 
\begin{theorem}\label{thm.spec_mdp_attack_form}
Consider a special MDP with reward function $\overline{R}$. Define $\widehat R(s, a) = \overline R(s, a)$ for $\occstate(s) = 0$ and otherwise
\begin{align*}
    \widehat R(s, a) = \begin{cases}
     x_s + \frac{\epsilon}{\occstate(s)} \quad &\mbox{ if } a = \optallowpi(s)\\
     x_s \quad &\mbox{ if } a \ne \optallowpi(s) \land \overline{R}(s, a) \ge x_s \\
     \overline{R}(s, a) \quad &\mbox{ otherwise}
    \end{cases},
\end{align*}
where $x_s$ is the solution to the equation
\begin{align*}
     \sum_{a \ne \optallowpi(s)} \pos{\overline{R}(s, a) - x} = x - \overline{R}(s, \optallowpi(s)) + \frac{\epsilon}{\occstate(s)}.
\end{align*}
Then, $(\optallowpi, \widehat R)$
is an optimal solution to \eqref{prob.reward_design.approx}.
\end{theorem}
Theorem \ref{thm.spec_mdp_attack_form} provides an interpretable  solution to \eqref{prob.reward_design.approx}: for each state-action pair $(s, a \ne \optallowpi(s))$ we reduce the corresponding reward $\overline{R}(s,a)$ if it exceeds a state dependent threshold. Likewise, we increase the rewards 
$\overline{R}(s,\optallowpi(s))$.

\section{Characterization Results for General MDPs}\label{sec.general_mdps}

\looseness-1In this section, we extend the characterization results from the previous section to general MDPs for which transition probabilities can depend on actions. 
In contrast to the previous section, the computational complexity result from Theorem~\ref{thm.copmutational_hardness} showcase the challenge of deriving characterization results for general MDPs that specify the form of an optimal solution. 
We instead focus on bounding the value of an optimal solution to \eqref{prob.reward_design.approx} relative to the score of an optimal policy $\optpi$. 
More specifically, we define the relative value $\relval$ as
\begin{align*}
    \relval = \underbrace{\norm{\overline{R} - \RsolB}_{2}}_{\text{cost}} + \lambda \cdot \underbrace{[\score^{\optpi, \overline{R}} - \score^{\pisolB, \overline{R}}]}_{\text{performance reduction}},
\end{align*}
where $(\pisolB, \RsolB)$ is an optimal solution to the optimization problem \eqref{prob.reward_design.approx}. Intuitively, $\relval$ expresses the optimal value of \eqref{prob.reward_design.approx} in terms of the cost of the reward design and the agent's performance reduction. 

The characterization results in this section provide bounds on $\relval$ and are obtained by analyzing two specific policies: an optimal admissible policy $\optallowpi \in \arg\max_{\pi \in \allowedpi} \score^{\pi, \overline{R}}$ that optimizes for performance $\score$, and a min-cost policy $\mincostallowpi$ that minimizes the cost of the reward design and is a solution to the optimization problem \eqref{prob.reward_design.approx} with $\lambda = 0$. 
As we show in the next two subsections, bounding the cost of forcing $\optallowpi$ and $\mincostallowpi$ can be used for deriving bounds on $\relval$. Next, we utilize the insights of the characterization results to devise a local search algorithm for solving the reward design problem, whose utility we showcase using experiments.

\subsection{Perspective 1: Optimal Admissible Policy}

Let us consider an optimal admissible policy $\optallowpi \in \arg\max_{\pi \in \allowedpi} \score^{\pi, \overline{R}}$. Following the approach presented in the previous section, we can design $\widehat{R}$ by (approximately) solving the optimization problem \eqref{prob.reward_poisoning_attack} (see Remark \ref{rm.reward_poisoning}) with the target policy $\targetpi = \optallowpi$. While this approach does not yield an optimal solution for general MDPs, the cost of its solution can be bounded by a quantity that depends on the gap between the scores of an optimal policy $\optpi$ and an optimal admissible policy $\optallowpi$.  

In particular, for any policy $\pi$ we can define the performance gap as $\scoregap =\score^{\optpi, \overline{R}} - \score^{\pi, \overline{R}}$. 
As we will show, 
the cost of forcing policy $\pi$ can be upper and lower bounded by terms that linearly depend on $\scoregap$. Consequently,
this means that one can also bound $\relval$ with terms that linearly depend on $\minscoregap = \min_{\pi} \scoregap$, which is nothing else but the 
performance gap
of
$\pi = \optallowpi$.
Formally, we obtain the following result.

\begin{theorem}\label{thm.general_mdp.charact.bounds.opt_const}
The relative value $\relval$ is bounded by
\begin{align*}
\alpha_{\score} \cdot \minscoregap \le \relval \le  \beta_{\score} \cdot \minscoregap + \frac{\epsilon}{\occstate_{\min}} \cdot \sqrt{|S| \cdot |A|},
\end{align*}
where $\alpha_{\rho} = \left (\lambda + \frac{1-\gamma}{2} \right )$ and $\beta_{\score} = \left (\lambda + \frac{1}{\occstate_{\min}} \right )$.
\end{theorem}
Note that the bounds in the theorem can be efficiently computed from the MDP parameters. Moreover, the reward design approach based on forcing $\optallowpi$ yields a solution to \eqref{prob.reward_design.approx} whose value (relative to the score of $\optpi$) satisfies the bounds in Theorem \ref{thm.general_mdp.charact.bounds.opt_const}. We use this approach as a baseline.

\subsection{Perspective 2: Min-Cost Admissible Policy}

We now take a different perspective, and compare $\relval$ to the cost of the reward design obtained by forcing the min-cost policy $\mincostallowpi$.
Ideally, we would relate $\relval$ to the the smallest cost that the reward designer can achieve. However, this cost is not efficiently computable (due to Theorem \ref{thm.copmutational_hardness}), making such a bound uninformative.

Instead, we consider $Q$ values: as \citet{ma2019policy} showed, the cost of forcing a policy can be upper and lower bounded by a quantity that depends on $Q$ values. We introduce a similar quantity, denoted by   $\minmaxgapqval$ and defined as 
\begin{align*}
&\minmaxgapqval = 
 	\min_{\pi
 	\in \allowedpi
 	}\max_{s \in S_{\textnormal{pos}}^{\pi}}
 	\big(
 	Q^{\optpi, \overline{R}}(s, \optpi(s)) - 
 	Q^{\optpi, \overline{R}}(s, \pi(s))
 	\big),
\end{align*}
where $S_{\textnormal{pos}}^{\pi} = \{s| \occstate^{\pi}(s) > 0\}$ contains the set of states that policy $\pi$ visits with strictly positive probability.
In the full version of the paper,
we
present an algorithm 
called \qgreedy~
that efficiently computes
$\minmaxgapqval$.
The \qgreedy~algorithm also outputs a policy $\piqg$ that solves the corresponding min-max optimization problem. By approximately solving the optimization problem \eqref{prob.reward_poisoning_attack} with $\targetpi = \piqg$, we can obtain reward function $\widehat{R}$ as a solution to the reward design problem. We use this approach as a baseline in our experiments, and also for deriving the bounds on $\relval$ relative to $\minmaxgapqval$ provided in the following theorem.

\begin{theorem}\label{thm.general_mdp.charact.bounds.qgreedy}
The relative value $\relval$ is bounded by
\begin{align*}
&\alpha_Q \cdot \minmaxgapqval \le 
    \Phi
    \le 
    \beta_Q \cdot \minmaxgapqval
    + \frac{\epsilon}{\occstate_{\min}}\sqrt{|S| \cdot |A|},
\end{align*}
where $\alpha_Q = \left (\lambda\cdot \occstate_{\min} + \frac{1-\gamma}{2}\right )$ and $\beta_Q = \left( \lambda + \sqrt{|S|} \right)$.
\end{theorem}
The bounds in Theorem \ref{thm.general_mdp.charact.bounds.qgreedy} are obtained by analyzing the cost of forcing policy $\piqg$ and the score difference $(\score^{\optpi, \overline{R}} - \score^{\piqg, \overline{R}})$. The well-known relationship between  $\rho^{\pi, R} - \rho^{\pi', R}$ and $Q^{\pi, R}$ for any two policies $\pi, \pi'$ (e.g., see \cite{schulman2015trust}) relates the score difference $(\score^{\optpi, \overline{R}} - \score^{\piqg, \overline{R}})$ to $Q^{\optpi, \overline{R}}$, so the crux of the analysis lies in upper and lower bounding the cost of forcing policy $\piqg$. 
To obtain the corresponding bounds, we utilize similar proof techniques to those presented in \cite{ma2019policy} (see Theorem 2 in their paper). Since the analysis focuses on $\piqg$, the approach based on forcing $\piqg$ outputs a solution to \eqref{prob.reward_design.approx} whose value (relative to the score of $\optpi$) satisfies the bounds in Theorem \ref{thm.general_mdp.charact.bounds.qgreedy}.

\subsection{Practical Algorithm: \coopalgo}
In the previous two subsections, we discussed characterization results for the relative value $\relval$ by considering two specific cases: optimizing performance and minimizing  cost. We now utilize the insights from the previous two subsections to derive a practical algorithm for solving \eqref{prob.reward_design.approx}. The algorithm is depicted in Algorithm \ref{alg.contrain_optimize}, and it searches for a well performing policy with a small cost of forcing it.

\begin{algorithm}
\caption{\coopalgo}
\label{alg.contrain_optimize}
\begin{algorithmic}[1]
\REQUIRE MDP $\overline M$, admissible set $\allowedpi$ 
\ENSURE Reward function $\widehat{R}$, policy $\pico$
\STATE $\pico \leftarrow \arg\max_{\pi \in \allowedpi} \score^{\pi, \overline{R}}$
\STATE $cost_{\textnormal{co}} \leftarrow \text{approx. solve \eqref{prob.reward_poisoning_attack} with } \targetpi = \pico$
\STATE set $\Pi_{\textnormal{co}} \leftarrow \allowedpi$ 
\REPEAT
\STATE output$_{\text{new}} \leftarrow $false
\FOR{$s$ in {\em priority-queue}($S^{\pico}_{\textnormal{pos}}$)}
\STATE $\Pi' \leftarrow \{ \pi | \pi \in \Pi_{\textnormal{co}} \land \pi(s) \ne \pico(s)\}$
\STATE $\pi' \leftarrow \arg\max_{\pi \in \Pi'} \score^{\pi, \overline{R}}$
\STATE $cost' \leftarrow \text{approx. solve  \eqref{prob.reward_poisoning_attack} with } \targetpi = \pi'$
\IF{$cost' - \lambda \score^{\pi', \overline{R}} < cost_{\textnormal{co}} - \lambda \score^{\pico, \overline{R}}$}
 \STATE set $\pico \leftarrow \pi'$, $cost_{\textnormal{co}} \leftarrow cost'$, and $\Pi_{\textnormal{co}} \leftarrow \Pi'$
 \STATE set output$_{\text{new}} \leftarrow $true and {\bf break}
\ENDIF
\ENDFOR
\UNTIL{output$_{\text{new}} = $ false}
\STATE $\widehat{R} \leftarrow \text{approx. solve \eqref{prob.reward_poisoning_attack} with } \targetpi = \pico$
\end{algorithmic}
\end{algorithm}

\begin{figure*}[t]
        \begin{subfigure}[b]{0.33\textwidth}
                \centering
              \includegraphics[width=.99\linewidth]{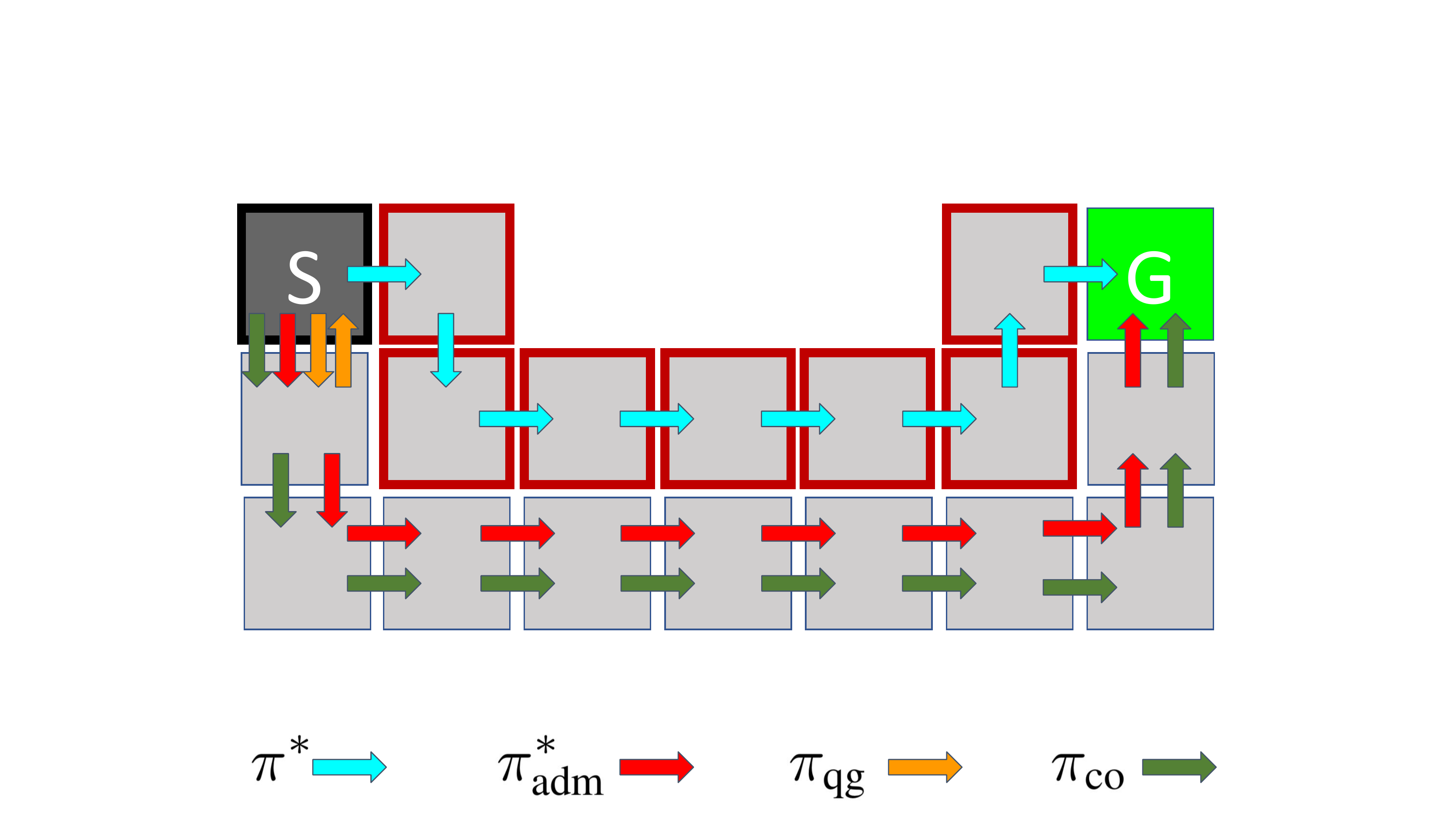}
                \caption{Cliff environment}
                \label{fig:ah_1}
        \end{subfigure}
        \begin{subfigure}[b]{0.33\textwidth}
                \centering
              \includegraphics[width=.99\linewidth]{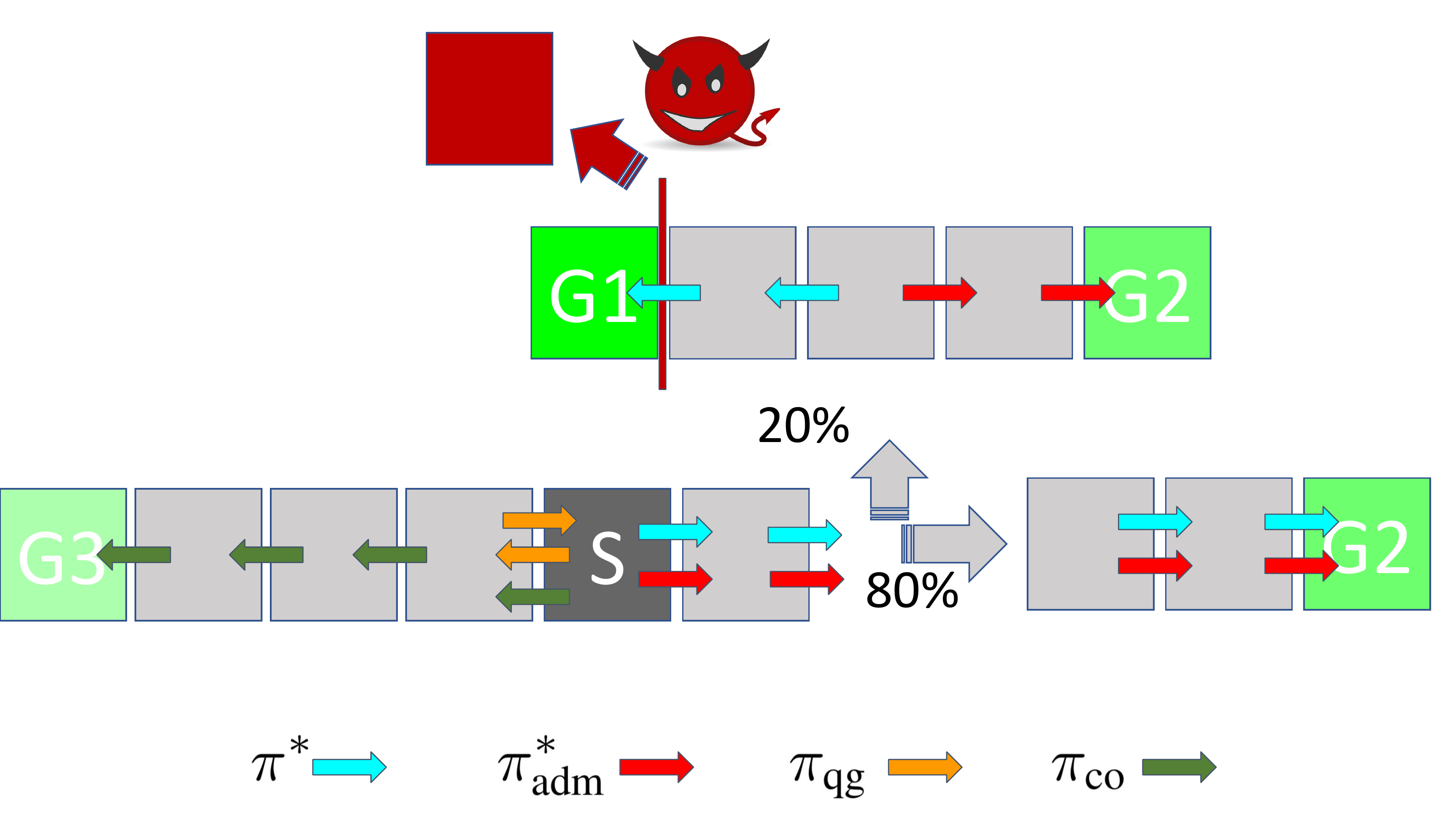}
              \caption{Action hacking environment}
                \label{fig:ah_2}
        \end{subfigure}
        \begin{subfigure}[b]{0.33\textwidth}
                \centering
              \includegraphics[width=.99\linewidth]{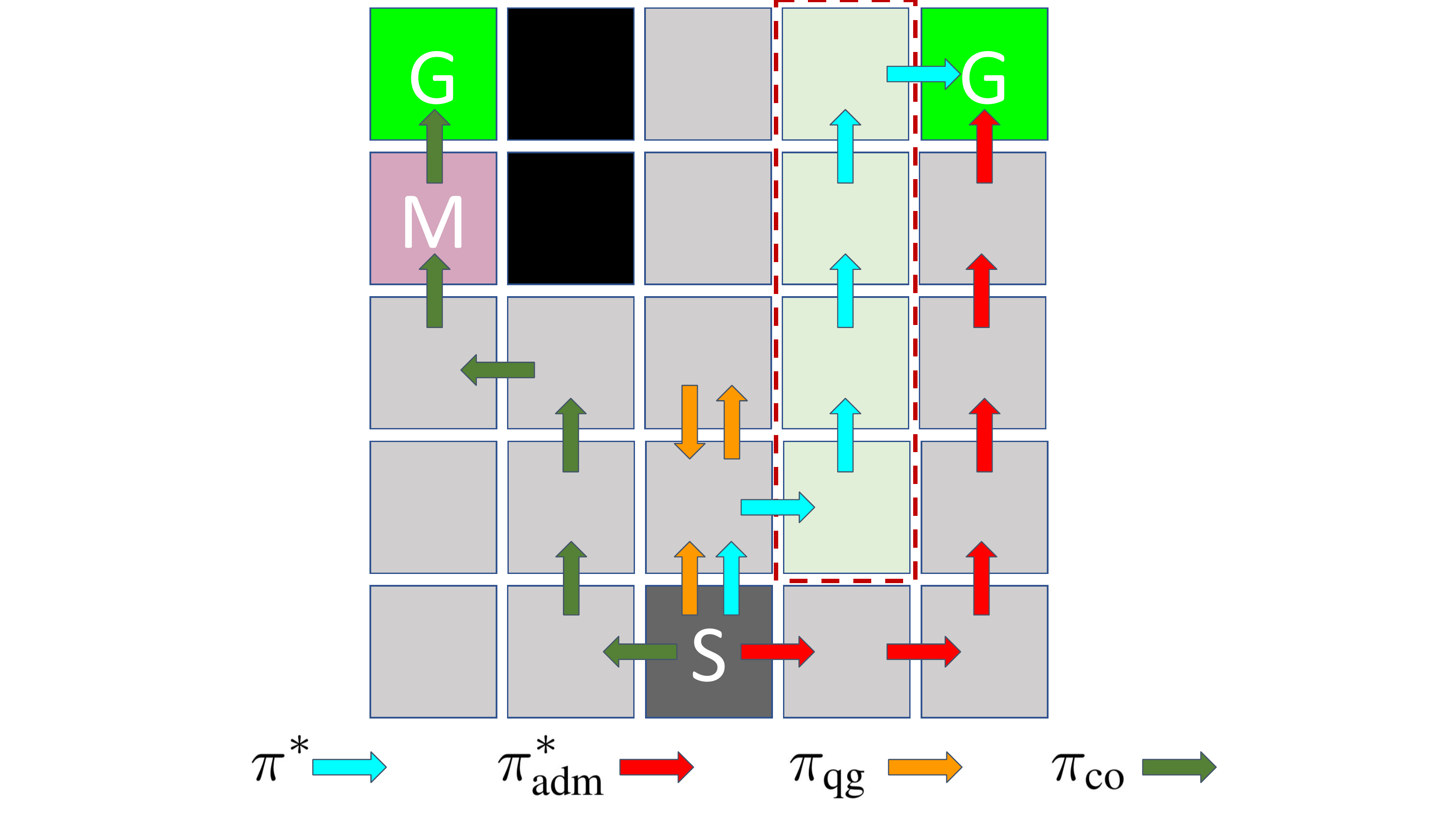}
              \caption{Grass and mud environment}
                \label{fig:ah_3}
        \end{subfigure}
        \caption{
        \textbf{Qualitative assessment.} (\textbf{a}) $\pico$ is the same as $\optallowpi$, taking the longer path to the goal state than $\optpi$ in order to avoid the cliff edge, while $\piqg$ simply alternates between the starting state and the state below it; (\textbf{b}) $\pico$ takes the path to the goal state \textrm{G3}---this behavior is less costly to incentivize than navigating to \textrm{G2} since 
        in the former case the reward designer mainly needs to 
        compensate for the difference in rewards between \textrm{G3} and \textrm{G2}
        (\textrm{G1} is reachable only 20\% of the time from \textrm{S}); 
        (\textbf{c}) $\optpi$ takes a path through the grass states, $\optallowpi$ takes a different but admissible path towards the same goal, while $\pico$ navigates the agent towards the other goal (with a lower cumulative reward, but the corresponding policy is less costly to force).
        \textbf{Quantitative assessment.} The objective values of \eqref{prob.reward_design.approx} for the approaches based on forcing $\optpi$, $\optallowpi$, $\piqg$, and $\pico$ are respectively:
        (\textbf{a}) $0.27$, $1.59$, $3.93$, and $1.59$; 
         (\textbf{b}) $-2.04$, $14.96$, $5.00$, and $3.82$;
         and (\textbf{c}) $-9.54$, $9.46$, $17.26$, and $7.92$.
        }\label{fig:exp_env}
\end{figure*}

The main blocks of the algorithm are as follows:
\begin{itemize}
    \item \textbf{Initialization (lines 1-2).} 
    The algorithm selects $\optallowpi$ as its initial solution, i.e., $\pico = \optallowpi$, and evaluates its cost by approximately solving \eqref{prob.reward_poisoning_attack}.
    \item \textbf{Local search (lines 4-15).} Since the initial policy $\pico$ is not necessarily cost effective, the algorithm proceeds with a local search in order to find a policy that has a lower value of the objective of \eqref{prob.reward_design.approx}. In each iteration of the local search procedure, it iterates over all states that are visited by the current $\pico$ (i.e., $S^{\pico}_{\textnormal{pos}}$), prioritizing those that have a higher value of $Q^{\optpi, \overline{R}}(s, \optpi(s)) -  Q^{\optpi, \overline{R}}(s, \pico(s))$ (obtained via {\em priority-queue}). The intuition behind this prioritization is that this Q value difference is reflective of the cost of forcing action $\pico(s)$ (as can be seen by setting $\lambda = 0$ in the upper bound of Theorem \ref{thm.general_mdp.charact.bounds.qgreedy}). Hence, deviations from $\pico$ that are considered first are deviations from those actions that are expected to induce high cost. 
    \item \textbf{Evaluating a neighbor solution (lines 7-12).} Each visited state $s$ defines a neighbor solution in the local search. To find this neighbor, the algorithm first defines a new admissible set of policies $\Pi'$ (line 7), obtained from the current one by making action $\pico(s)$ inadmissible. The neighbor solution is then identified as $\pi' \in \arg\max_{\pi \in \Pi'} \score^{\pi, \overline{R}}$ (line 8) and the  costs of forcing it is calculated by approximately solving \eqref{prob.reward_poisoning_attack} with $\targetpi = \pi'$ (line 9).
    If $\pi'$ yields a better value of the objective of \eqref{prob.reward_design.approx} than $\pico$ does (line 10), we have a new candidate policy and the set of admissible policies is updated to $\Pi'$ (lines 11-12). 
    \item \textbf{Returning solution (line 16).} Once the local search finishes, the algorithm outputs $\pico$ and the reward function $\widehat{R}$ found by approximately solving \eqref{prob.reward_poisoning_attack} with $\targetpi = \pico$.
\end{itemize}

In each iteration of the local search (lines 5-14), the algorithm either finds a new candidate (output$_{\text{new}}$=true) or the search finishes with that iteration (output$_{\text{new}}$=false). Notice that the former cannot go indefinitely since the admissible set reduces  between two iterations. This means that the algorithm is guaranteed to halt. 
Since the local search only accepts new policy candidates if they are better than the current $\pico$ (line 10), the output of \coopalgo~is guaranteed to be better than forcing an optimal admissible policy (i.e., approx. solving \eqref{prob.reward_poisoning_attack} with $\targetpi=\optallowpi$).


\section{Numerical Simulations}
\label{sec.numerical.simulation}

\begin{figure*}[!t]
        \begin{subfigure}[b]{0.33\textwidth}
                \centering                \includegraphics[width=.85\linewidth]{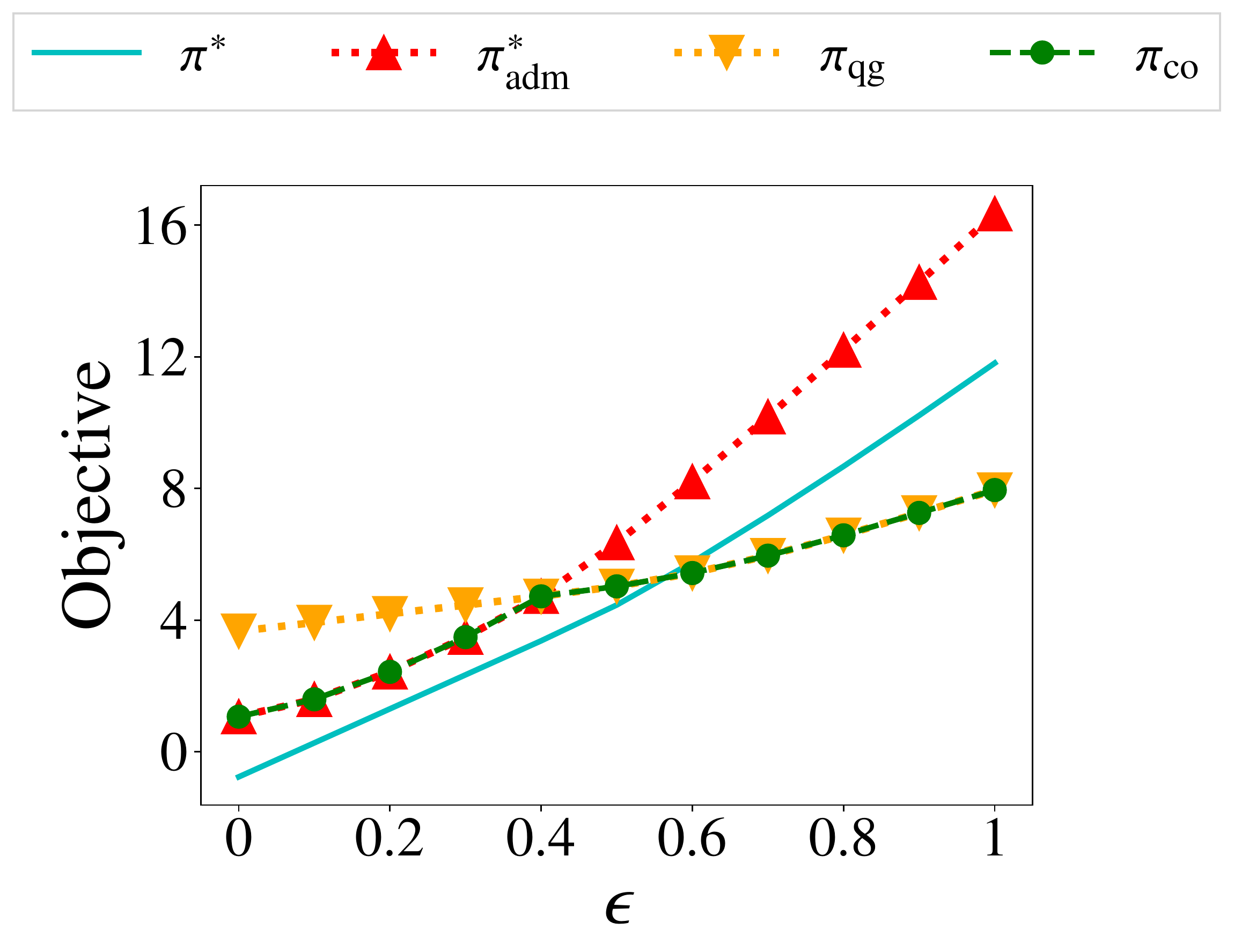}
                \caption{Cliff}
                \label{fig:eps:res1}
        \end{subfigure}
        \begin{subfigure}[b]{0.33\textwidth}
                \centering                \includegraphics[width=.85\linewidth]{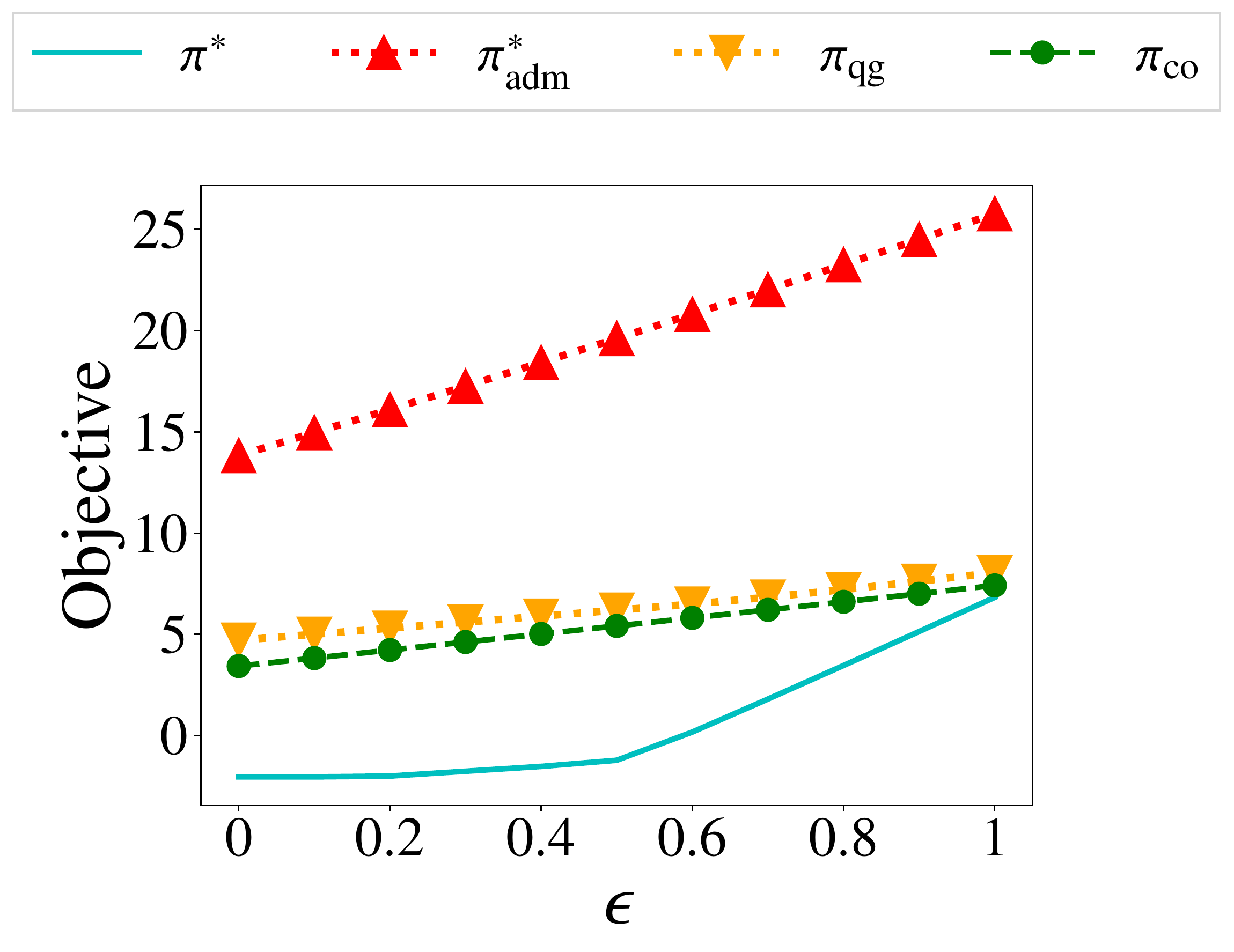}
                \caption{Action hacking}
                \label{fig:eps:res2}
        \end{subfigure}
        \begin{subfigure}[b]{0.33\textwidth}
                \centering                \includegraphics[width=.85\linewidth]{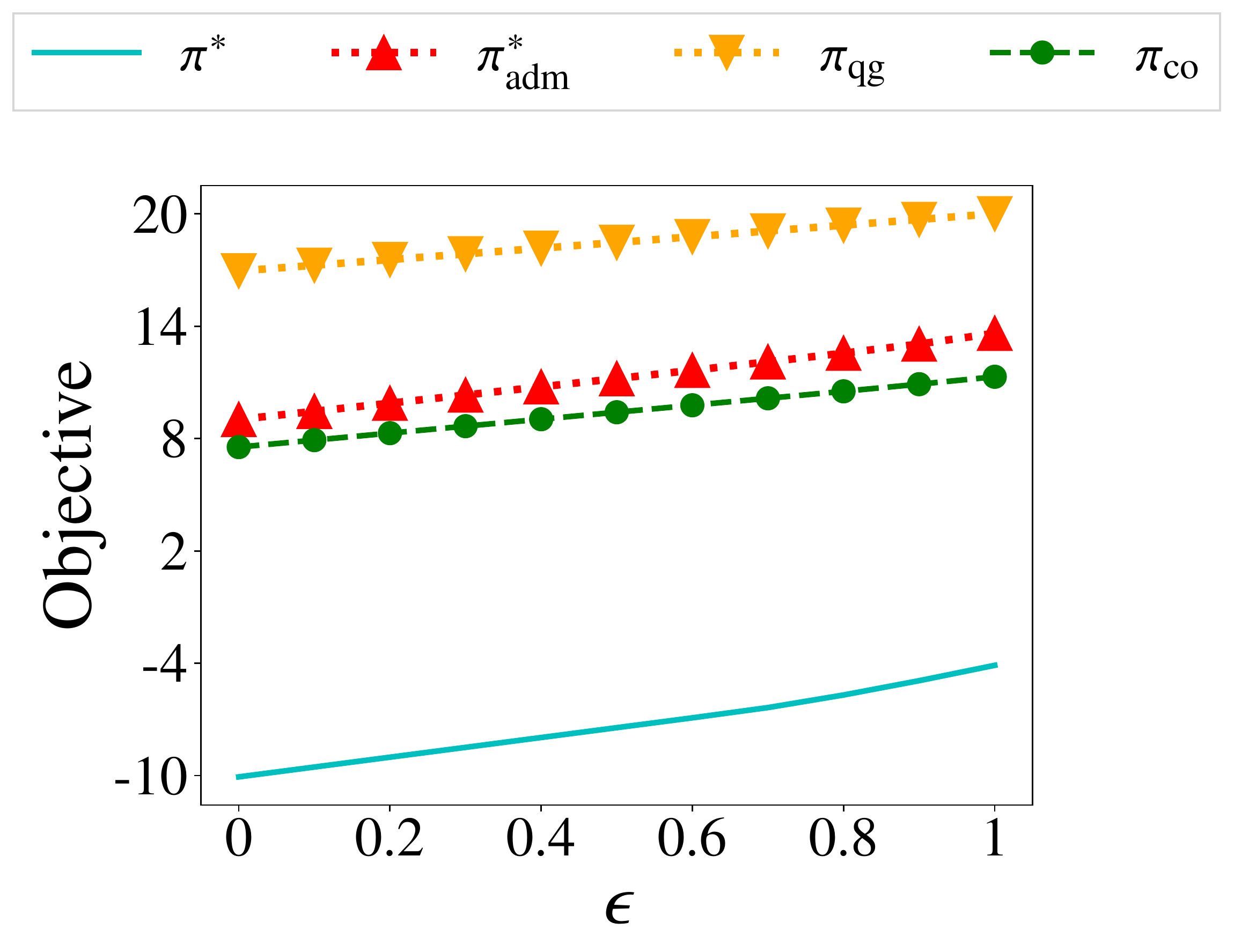}
                \caption{Grass and mud}
                \label{fig:eps:res3}
        \end{subfigure}
        \\
        \begin{subfigure}[b]{0.33\textwidth}
                \centering                \includegraphics[width=.85\linewidth]{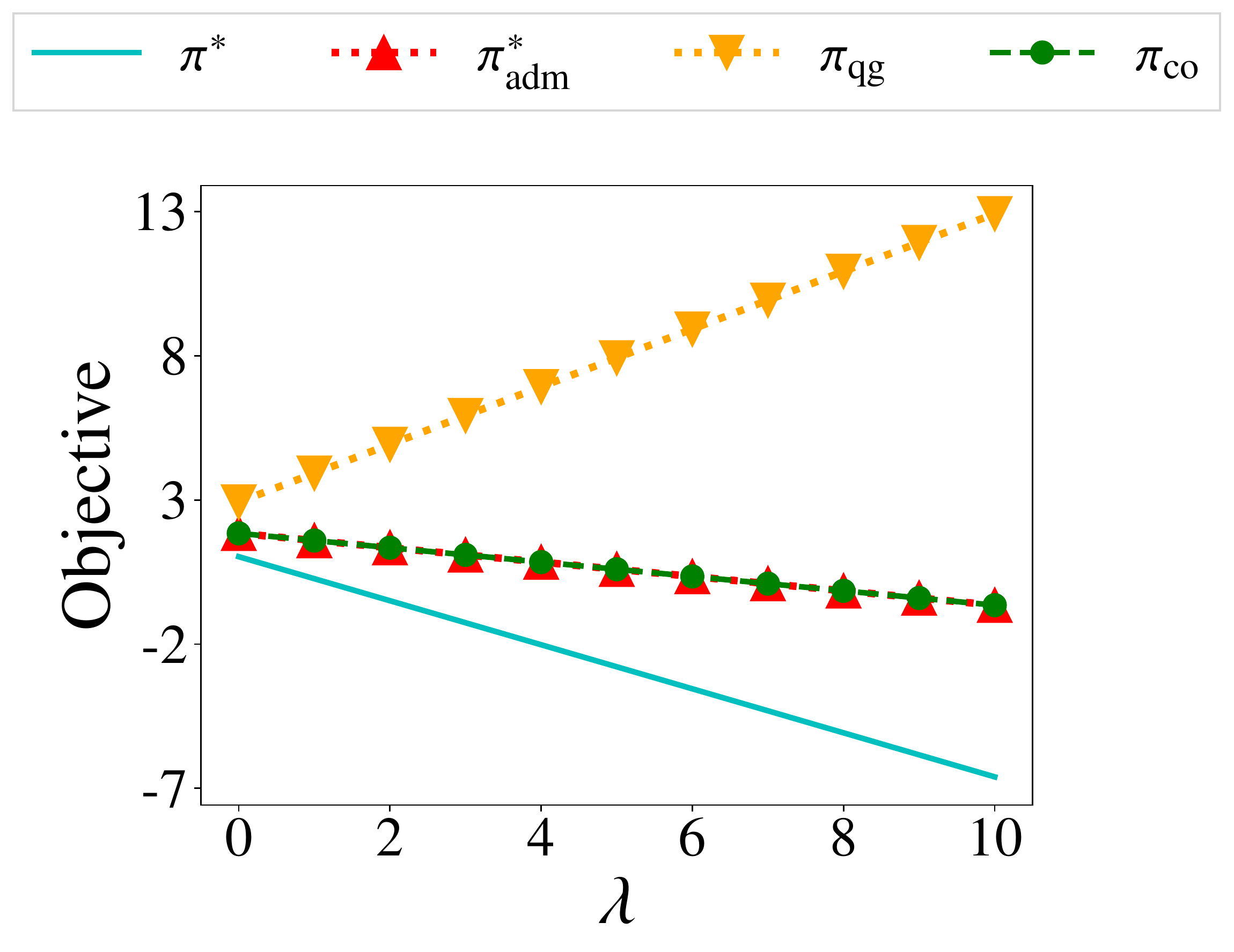}
                \caption{Cliff}
                \label{fig:lam:res1}
        \end{subfigure}
        \begin{subfigure}[b]{0.33\textwidth}
                \centering                \includegraphics[width=.85\linewidth]{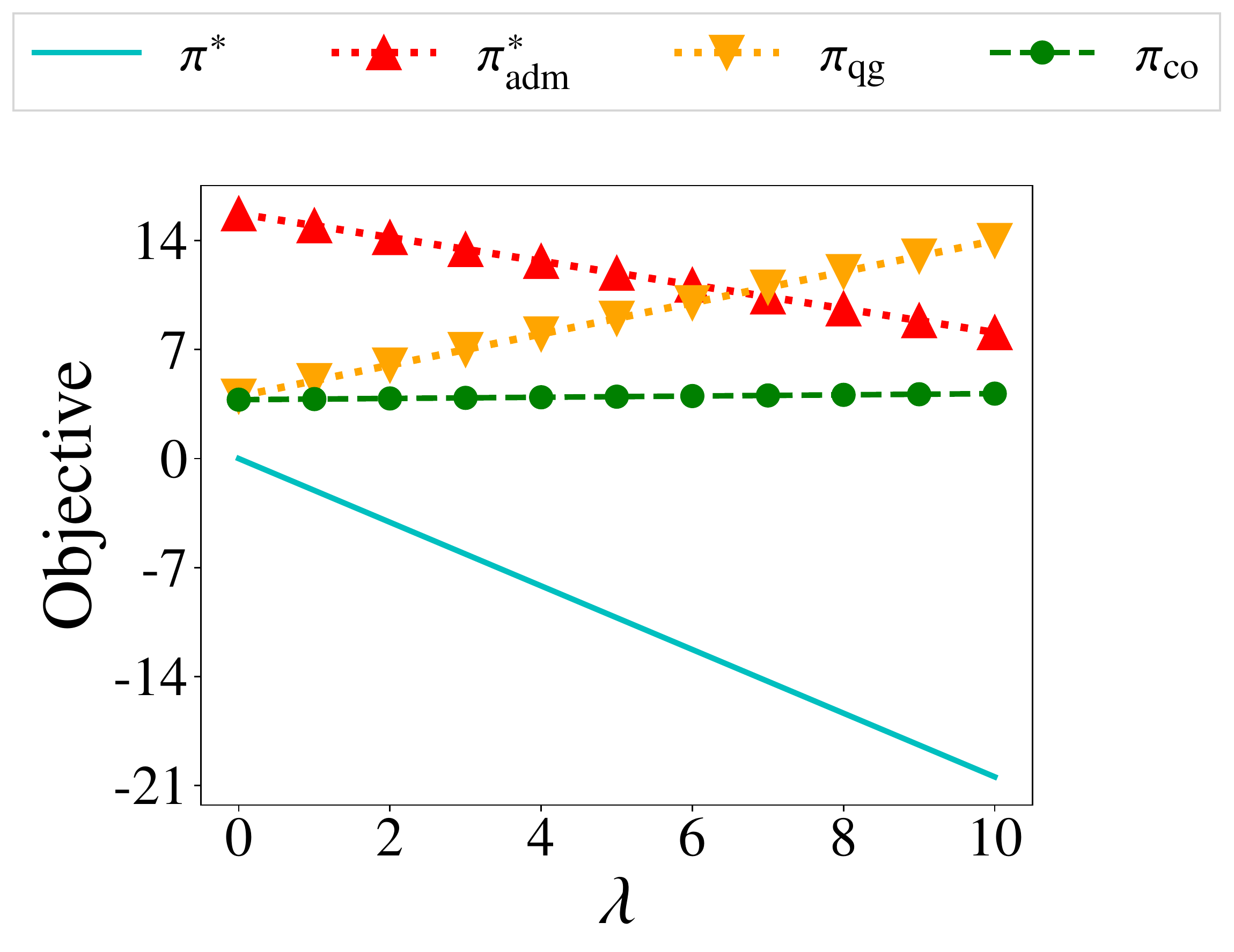}
                \caption{Action hacking}
                \label{fig:lam:res2}
        \end{subfigure}
        \begin{subfigure}[b]{0.33\textwidth}
                \centering                \includegraphics[width=.85\linewidth]{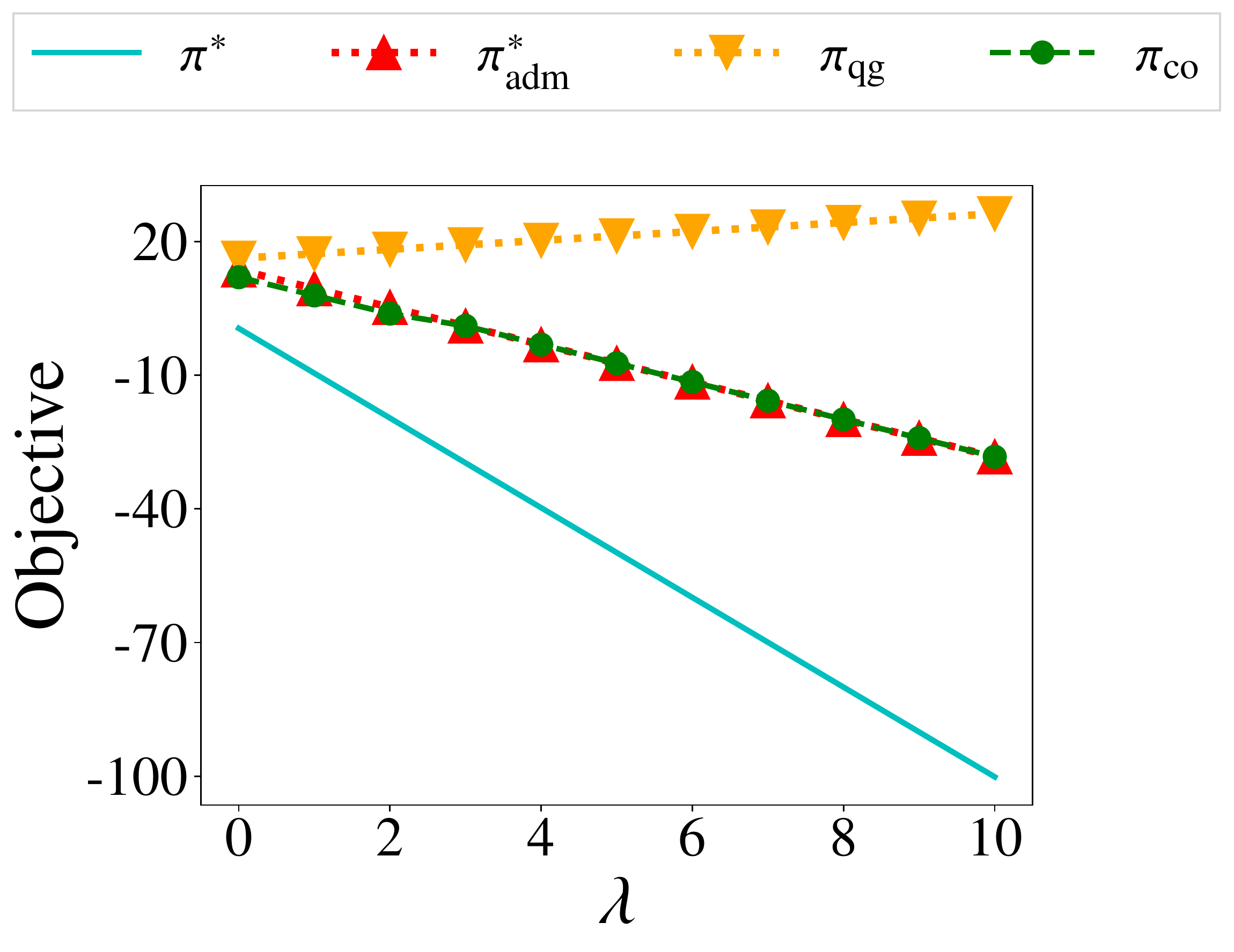}
                \caption{Grass and mud}
                \label{fig:lam:res3}
        \end{subfigure}
        \caption{\looseness-1Effect of $\lambda$ and $\epsilon$ on
        the objective value of \eqref{prob.reward_design.approx} for different approaches. \textbf{(a, b, c)} vary $\epsilon$ with $\lambda=1.0$; \textbf{(d, e, f)} vary $\lambda$ with $\epsilon=0.1$. Lower values on the y-axis denote better performance. Note that $\optpi$ is not an admissible policy in these environments; importantly, the objective value for $\pico$ is consistently better (lower) than for $\optallowpi$ and $\piqg$,  highlighting its efficacy.
        }
        \label{fig:vary_params}
\end{figure*}

We analyze the efficacy of \coopalgo~in solving the optimization problem \eqref{prob.reward_design.approx} and the policy it incentivizes, $\pico$.  
We consider three baselines, all based on approximately solving the optimization problem \eqref{prob.reward_poisoning_attack}, but with different target policies $\targetpi$: a) forcing an optimal policy, i.e., $\targetpi = \optpi$, b) forcing an optimal admissible policy, i.e., $\targetpi = \optallowpi$, c) forcing the policy obtained by the \qgreedy~algorithm, i.e., $\targetpi = \piqg$.\footnote{$\optpi$ might not be admissible; also, even though $\optpi$ is an optimal policy, there is still a cost of forcing it to create the required gap.} We compare these approaches by measuring their performance w.r.t. the objective value of \eqref{prob.reward_design.approx}---lower value is better. By default, we set the parameters $\gamma=0.9$, $\lambda = 1.0$ and $\epsilon = 0.1$.

\subsection{Experimental Testbeds}
As an experimental testbed, we consider three simple navigation environments, shown in Figure \ref{fig:exp_env}. 
Each environment contains a start state \textrm{S} and goal state(s) \textrm{G}. Unless otherwise specified, in a non-goal state, the agent can navigate in the left, right, down, and up directions, provided there is a state in that direction. In goal states, the agent has a single action which transports it back to the start state.

\looseness-1\textbf{Cliff environment (Figure \ref{fig:ah_1}).} 
This environment depicts a scenario where some of the states are potentially unsafe due to model uncertainty that the reward designer is aware of. More concretely, the states with ``red'' cell boundaries in Figure~\ref{fig:ah_1} represent the edges of a cliff and are unsafe; as such, all actions leading to these states are considered inadmissible. In this environment, the action in the goal state yields a reward $\overline{R}$ of $20$ while all other actions yield a reward $\overline{R}$ of $-1$.

\textbf{Action hacking environment (Figure \ref{fig:ah_2}).} 
This environment depicts a scenario when some of the agent's actions could be hacked at the deployment phase, taking the agent to a bad state. The reward designer is aware of this potential hacking and seeks to design a reward function so that these actions are inadmissible. More concretely, the action leading the agent to \textrm{G1} is considered inadmissible. In this environment, we consider the reward function $\overline{R}$ and dynamics $P$ as follows. Whenever an agent reaches any of the goal states (\textrm{G1}, \textrm{G2}, or \textrm{G3}), it has a single action that transports it back to the starting state and yields a reward of $50$, $10$, and $5$ for \textrm{G1}, \textrm{G2}, and \textrm{G3} respectively. In all other states, the agent can take either the left or right action and navigate in the corresponding direction, receiving a reward of $-1$. With a small probability of $0.20$, taking the right action in the state next to \textrm{S} results in the agent moving up instead of right.

\looseness-1\textbf{Grass and mud environment (Figure \ref{fig:ah_3}).} This environment depicts a policy teaching scenario where the reward designer and the agent do not have perfectly aligned preferences (e.g., the agent prefers to walk on grass, which  the reward designer wants to preserve).  The reward designer wants to incentivize the agent not to step on the grass states, so actions leading to them are considered inadmissible.
In addition to the starting state and two goal states, the environment contains four grass states, one mud state, and $16$ ordinary states, shown by ``light green'', ``light pink'', and ``light gray'' cells respectively. The ``black'' cells in the figure represent inaccessible blocks.
The reward function $\overline{R}$ is as follows: the action in the goal states yields a reward of $50$; the actions in the grass and mud states yield rewards of $10$ and $-2$ respectively; all other actions have a reward of $-1$.

\subsection{Results}
Figure~\ref{fig:exp_env} provides an assessment of different approaches by visualizing the agent's policies obtained from the designed reward functions $\widehat{R}$. For these results, we set the parameters $\lambda = 1.0$ and $\epsilon = 0.1$. 
In order to better understand the effect of the parameters $\lambda$ and $\epsilon$, we vary these parameters and solve \eqref{prob.reward_design.approx} with the considered approaches. The results are shown in Figure \ref{fig:vary_params} for each environment separately.
We make the following observations based on the experiments. First, the approaches based on forcing $\optpi$ and $\optallowpi$ benefit more from increasing $\lambda$. This is expected as these two policies have the highest scores under $\overline{R}$; the scores of $\piqg$ and $\pico$ is less than or equal to the score of $\optallowpi$. Second, the approaches based on forcing $\piqg$ and $\pico$ are less susceptible to increasing $\epsilon$. This effect is less obvious, and we attribute it to the fact that \qgreedy~and \coopalgo~output $\piqg$ and $\pico$ respectively by accounting for the cost of forcing these policies. Since this cost clearly increases with $\epsilon$---intuitively, forcing a larger optimality gap in \eqref{prob.reward_poisoning_attack} requires larger reward modifications---we can expect that increasing $\epsilon$ deteriorates more the approaches based on forcing $\optpi$ and $\optallowpi$. Third, the objective value of \eqref{prob.reward_design.approx} is consistently better (lower) for $\pico$ than for $\optallowpi$ and $\piqg$, highlighting the relevance of \coopalgo.

\section{Conclusion}
The characterization results in this paper showcase the computational challenges of optimal reward design for admissible policy teaching. In particular, we showed that it is computationally challenging to find minimal reward perturbations that would incentivize an optimal agent into adopting a well-performing admissible policy. To address this challenge, we derived a local search algorithm that outperforms baselines which either account for only the agent's performance or for only the cost of the reward design. On the flip side, this algorithm is only applicable to tabular settings, so one of the most interesting research directions for future work would be to consider its extensions based on function approximation. In turn, this would also make the optimization framework of this paper more applicable to practical applications of interest, such as those related to safe and secure RL.


\newpage
\bibliography{main}

\iftoggle{longversion}{
\clearpage
\onecolumn
\appendix 
{\allowdisplaybreaks

\section{Appendix: Table of Contents}
\label{app.toc}

Appendix is structured according to the following sections:  

\begin{itemize}
    
    \item Section \nameref{app.background} introduces additional quantities and lemmas relevant for the formal proofs.
    
    \item An approach for solving  \eqref{prob.reward_poisoning_attack}, which is used in the experiments, is provided in
    Section \nameref{app.approx.reward_poisoning_attack}. This section also analyzes this approach and provides provable guarantees. 
    
    \item Section \nameref{app.qgreedy_algo} provides the description of \qgreedy~and shows that it is sound.
    
    \item The proof or Proposition \ref{prop.solvability.reward_design} is given in section \nameref{app.setting_proofs}.
    
    \item Theorem \ref{thm.copmutational_hardness} and Proposition \ref{prop.reward_design.approx} are proven in section \nameref{app.computational_challenges}. The same section contains an additional hardness result, which proves that the optimization problem \eqref{prob.reward_design.approx} is computationally hard.
    
    \item Proofs of Lemma \ref{lm.special_mdp.cost_of_poisoning} and Theorem \ref{thm.spec_mdp_attack_form} are in given in section \nameref{app.proof.special_mdps}. 
    
    \item Proofs of Theorem \ref{thm.general_mdp.charact.bounds.opt_const} and Theorem \ref{thm.general_mdp.charact.bounds.qgreedy} are provided in section \nameref{app.proofs.general_mdps}. The same section includes additional results relevant for proving the statements. 
    
\end{itemize}


\section{Background}\label{app.background}

As explained in the main paper,
for a policy $\pi$ and reward function $R$, we define its state-action value function $Q^{\pi, R}$ as
\begin{align*}
    Q^{\pi, R}(s, a) = \expct{\sum_{t=1}^{\infty} \gamma^{t-1} R(s_t, a_t) | \pi, s_1 = s, a_1 = a},
\end{align*}
where states $s_t$ and actions $a_t$ are obtained by executing policy $\pi$ starting from state $s_1 = s$ in which action $a_1 = a$ is taken.
~\\
The state value function $V^{\pi, R}$ is similarly defined as
\begin{align*}
    V^{\pi, R}(s) =
    \expct{\sum_{t=1}^{\infty} \gamma^{t-1} R(s_t, a_t) | \pi, s_1 = s} = 
    Q^{\pi, R}(s, \pi(s))
\end{align*}
We define $Q^{*, R}$ and $V^{*, R}$ as the maximum of these values over all policies, i.e., 
\begin{align*}
    &Q^{*, R}(s, a) = 
    \max_{\pi\in \PiDet} Q^{\pi, R}(s, a)
    \\
    &V^{*, R}(s) = \max_{\pi\in \PiDet} V^{\pi, R}(s)
\end{align*}
The optimal policy in an MDP can be calculated by setting
$\pi(s) \in \argmax_{a}Q^{*, R}(s,a)$ and satisfies
$Q^{\pi, R}=Q^{*, R}$.
For $R=\overline{R}$, we denote this policy with $\optpi$. 
~\\
We define the state occupancy measure $\occstate^\pi$ as 
\begin{align*}
    \occstate^{\pi}(s) = \expct{ (1-\gamma)\sum_{t=1}^{\infty} \gamma^{t-1}  \ind{s_t = s} | \pi, \sigma}.
\end{align*}
$\occstate^\pi$ can be efficiently calculated as it is the unique solution to the Bellman flow constraint
\begin{align}
    \occstate^\pi(s) = 
    (1-\gamma)\cdot\sigma(s) + 
    \gamma \sum_{s'} P(s', \pi(s'), s)\occstate^\pi(s').
    \label{eq.bellman.occstate}
\end{align}
An important result that we utilize repeatedly in our proofs, is the following lemma 
that relates the score difference $\score^{\pi_1, R} - \score^{\pi_2, R}$ for two policies $\pi_1, \pi_2$ to their Q-values through the state occupancy measure
$\occstate$.

\begin{lemma}{\citep{schulman2015trust}}\label{lm.score_diff_q_value}
    Any two deterministic policies, $\pi_1$ and  $\pi_2$, and reward function $R$  satisfy:
    \begin{align*}
     \score^{\pi_1, R}-
        \score^{\pi_2, R}
        =\sum_{s\in S} \occstate^{\pi_1}(s)
        \big(
        Q^{\pi_2, R}(s, \pi_1(s))-
            Q^{\pi_2, R}(s, \pi_2(s)) 
        \big).
    \end{align*}
\end{lemma}
For a policy $\pi$, we define $\posStates^\pi$ as
\begin{align}
    \posStates^\pi = 
     \{s| \occstate^{\pi}(s) > 0\}.
     \label{eq.s1}
\end{align}

We also prove the following lemma, which we use in several results in the following sections.

\begin{lemma}\label{lm.same_occupancy}
Let $\pi, \pi'$ be deterministic policies such that $\pi(s) = \pi'(s)$ for all $s\in \posStates^\pi \cap \posStates^{\pi'}$.
Then $\occstate^{\pi} = \occstate^{\pi'}$.
\end{lemma}
\begin{proof}
    \textbf{Part 1:}
    We first prove a simpler version of the Lemma; we assume that 
    $\pi(s) = \pi'(s)$ for all $s\in \posStates^\pi$.
    We then show how to extend the result to the general case.

    We prove by induction on $t$ that for all states $s$,
    \begin{align*}
    \Pr{s_t = s | \pi} = \Pr{s_t = s | \pi'}
    \end{align*}
    where $s_t$ denotes the state visited at time $t$.
    
    The claim holds for $t=1$ as the initial probabilities are sampled from $\sigma$. Assuming the claim holds for $t$,
    \begin{align*}
        \Pr{s_t = s | \pi} &= \sum_{s'} \Pr{s_{t-1}=s'|\pi}\Pr{s', \pi(s'), s}
        \\&=
         \sum_{s': \Pr{s_{t-1} = s' | \pi} > 0} \Pr{s_{t-1}=s'|\pi}\Pr{s', \pi(s'), s}
        \\&\overset{(i)}{=}
        \sum_{s': \Pr{s_{t-1} = s' | \pi} > 0} \Pr{s_{t-1}=s'|\pi'}\Pr{s', \pi(s'), s}
        \\&\overset{(ii)}{=}
        \sum_{s': \Pr{s_{t-1} = s' | \pi} > 0} \Pr{s_{t-1}=s'|\pi'}\Pr{s', \pi'(s'), s}
        \\&\le 
        \sum_{s'}
        \Pr{s_{t-1}=s'|\pi'}\Pr{s', \pi'(s'), s}
        \\&=\Pr{s_{t}=s | \pi'},
    \end{align*}
    where $(i)$ follows from the induction hypotheses
    and $(ii)$ follows from the fact that
    if $\Pr{s_{t-1} = s' | \pi} > 0$, then 
    $\occstate^{\pi}(s') > 0$ and therefore
    $\pi(s) = \pi'(s)$.
    Since
    \begin{align*}
        \sum_{s}\Pr{s_t=s|\pi} = 
        \sum_{s}\Pr{s_t=s|\pi'} = 1,
    \end{align*}
    it follows that
    \begin{align*}
    \Pr{s_t = s | \pi} = \Pr{s_t = s | \pi'}.
    \end{align*}
    Therefore,
    \begin{align*}
        \occstate^{\pi}(s) = 
        \sum_{t=1}^{\infty}
        \Pr{s_t = s | \pi} = 
        \sum_{t=1}^{\infty}
        \Pr{s_t = s | \pi'} = 
        \occstate^{\pi'}(s).
    \end{align*}

    \textbf{Part 2:}
    Now, in order to obtain the general case, define $\tilde{\pi}$ as follows.
    \begin{align*}
      \tilde{\pi}(s)
      := \begin{cases}
        \pi(s) \quad&\text{if} \quad s\in \posStates^{\pi}\\
        \pi'(s) \quad&\text{otherwise}
     \end{cases},
    \end{align*}
    By the simpler version just proved, since
    $\tilde{\pi}(s)= \pi(s)$ for all $s\in \posStates^{\pi}$, 
    $\occstate^{\tilde{\pi}} = \occstate^\pi$.
    ~\\
    Furthermore, 
    for all $s\in \posStates^{\pi'}$,
    either $s\in \posStates^{\pi}$, in which case
    $\pi(s)=\pi'(s)$ by assumption and therefore
    by definition of $\tilde{\pi}=\pi(s)=\pi'(s)$,
    or $s\notin \posStates^{\pi}$, in which case,
    $\tilde{\pi}(s) = \pi'(s)$. Since $\tilde{\pi}(s) = \pi'(s)$
    in both cases,
    it follows that by the simpler version just proved,
    $\occstate^{\tilde{\pi}} = \occstate^{\pi'}$. 
    ~\\
    Therefore
    $\occstate^{\pi} = \occstate^{\tilde{\pi}} = \occstate^{\pi'}$ as claimed.
\end{proof}
\noindent



\section{Approximately Solving the Optimization Problem 
\eqref{prob.reward_poisoning_attack}
}\label{app.approx.reward_poisoning_attack}

In this section, we show to efficiently approximate the optimization problem \eqref{prob.reward_poisoning_attack}.
In order to obtain the approximate solution, we will consider the following optimization problem
\begin{align}
	\label{prob.rp_approx}
	\tag{P5-ATK}
	\min_{R} \quad &\norm{R - \overline R}_{2}\\
	\label{constraint.rqv}
    \text{s.t.} \quad &\forall s, a: 
    Q(s, a) = R(s, a) + \gamma \sum_{s'}P(s, a, s') V(s')\\
    \label{constraint.ge}
    &
    \forall s \in \posStates^\targetpi, a \ne \targetpi(s): 
    Q(s, \targetpi(s)) \ge
    Q(s, a)  + \epsilon'(s, a)\\
    \label{constraint.vqone}
    &
    \forall s \in \posStates^\targetpi: 
    V(s) = Q(s, \targetpi(s))\\\label{constraint.vqzero}
    &\forall s \notin \posStates^\targetpi, a:
    V(s) \ge Q(s, a),
\end{align}
where $\epsilon'(s, a) \ge 0$ are arbitrary non-negative values that will be specified later.
We first show that the constraints of the optimization problem effectively ensure that $V$ and $Q$ can be thought of as the $V^{*, R}$ and $Q^{*, R}$ vectors respectively. Note that this is not trivial since for $s\notin \posStates^\targetpi$, the constraint $V(s) = \max_{a}Q(s, a)$ is not explicitly enforced. Formally, we have the following lemma.
\begin{lemma}\label{lm.effective_q_v}
Let $\epsilon'(s, a)$ be a non-negative vector.
If the vectors $(R, Q, V)$ satisfy the constraints of the optimization problem \eqref{prob.rp_approx}, the vectors
$(R, Q^{*, R}, V^{*, R})$ satisfy the constraints as well.
\end{lemma}
\begin{proof}
    Starting with $Q, V$, we run the standard value iteration algorithm for finding $Q^{*, R}, V^{*, R}$ and claim that at
  the end of each step, all the constraints would still be satisfied.
  Concretly, We set $V^{0} = V$ and $Q^{0} = Q$ and for all $t \ge 0$: 
  \begin{gather*}
    \label{eq:value_iter}
    V^{t + 1}(s, a) = \max_{a} Q^t(s, a)\\
    Q^{t + 1}(s, a) = R(s, a) + \gamma \sum_{s'} P(s, a, s') V^{t + 1}(s').
  \end{gather*}
  We claim that for all $t \ge 0$,  the vectors $(R, Q^t, V^t)$ satisfty the constraints
  \eqref{constraint.rqv} to \eqref{constraint.vqzero}.
  We prove that claim by induction on $t$.

  For $t=0$, the claim holds by
  assumption. Assume that the claim holds for $t-1$; we will show
  that it holds for $t$ as well by proving
  the constraints \eqref{constraint.rqv}, \eqref{constraint.vqone},
  \eqref{constraint.ge} and \eqref{constraint.vqzero} repsectively.
  Constraint \eqref{constraint.rqv} holds by definition of $Q^{t}(s, a)$.
  For constraint \eqref{constraint.vqone},
  observe that for all $s \in \posStates^\targetpi$,
  \begin{align}
    V^{t}(s) &\overset{(i)}{=} \max_{a} Q^{t - 1}(s, a) \notag
           \\&\overset{(ii)}{=} Q^{t-1}(s, \targetpi(s)) \notag
           \\&\overset{(iii)}{=} V^{t-1}(s)\label{eq.V_t_equals_v_t_minus_1}
  \end{align}
  where $(i)$ follows from the definition of $V^t$, 
  $(ii)$ follows from \eqref{constraint.ge} 
  and $(iii)$ follows from
  \eqref{constraint.vqone} for $V^{t-1}$ and $Q^{t-1}$.

  Now observe that if $s\in \posStates^\targetpi$ and $s'\notin \posStates^\targetpi$, 
  then
  $P(s, \targetpi(s), s') = 0$ as otherwise
  given \eqref{eq.bellman.occstate},
  $\occstate^\targetpi(s')$ would be lower bounded by
  $\occstate^\targetpi(s) \cdot P(s, \targetpi(s), s') > 0$, contradicting
  the assumption $s'\notin \posStates^{\targetpi}$.
  Therefore, for all $s\in \posStates^\targetpi$,
  \begin{align}
    Q^{t}(s, \targetpi(s)) \notag
    &= R(s, \targetpi(s)) + \gamma \sum_{s'} P(s, \targetpi(s), s') V^{t}(s'). \notag
  \\&= R(s, \targetpi(s)) + \gamma \sum_{s' \in \posStates^\targetpi} P(s, \targetpi(s), s') V^{t}(s'). \notag
  \\&\overset{(i)}{=}
  R(s, \targetpi(s)) + \gamma \sum_{s' \in \posStates^\targetpi} P(s, \targetpi(s), s') V^{t - 1}(s').  \notag
  \\&=
  R(s, \targetpi(s)) + \gamma \sum_{s'} P(s, \targetpi(s), s') V^{t - 1}(s').  \notag
  \\&=
  Q^{t-1}(s, \targetpi(s)), \label{eq.Q_t_equals_q_t_minus_1}
  \end{align}
  where $(i)$ follows from \eqref{eq.V_t_equals_v_t_minus_1}.
  Together with \eqref{eq.V_t_equals_v_t_minus_1}, \eqref{eq.Q_t_equals_q_t_minus_1}
  implies that 
  the constraint \eqref{constraint.vqone} still holds.

  Now observe that
  for $s\notin \posStates^\targetpi$, 
  \begin{align*}
      V^{t}(s) = \max_{a}Q^{t-1}(s, a)
      \le V^{t - 1}(s),
  \end{align*}
  where the inequality follows from the induction hypothesis; namely, 
  constraint \eqref{constraint.vqzero} for $V^{t-1}$ and $Q^{t-1}$.
  This means that for all $s$ (both when $s\notin \posStates^\targetpi$ and when
      $s\in \posStates^\targetpi$),
    $V^{t}(s) \le V^{t-1}(s)$
    and therefore given 
    \eqref{constraint.rqv},
    \begin{align*}
      Q^{t}(s, a)\le Q^{t - 1}(s, a).
    \end{align*}
   for all $s, a$. This implies that the constraint \eqref{constraint.ge} still
  holds because the LHS has stayed the same and RHS hasn't increased.
  Finally, constraint \eqref{constraint.vqzero} holds as well because it holds
  with $V^{t}, Q^{t - 1}$ by definition of $V^{t}$ and $Q^t(s, a)\le Q^{t - 1}(s,
      a)$. 

  Since $( Q^t, V^t )$ converge to $( Q^{*, R}, V^{*,R} )$ and the constraints
  characterize a closed set, $( R, Q^{*, R}, V^{*,R} )$ also satisfy the constraints.
\end{proof}
While the value of $\epsilon'$ can be arbitrary in \eqref{constraint.ge}, in our analysis we will mainly consider
$\epsilon'_{\targetpi}$, which for
a non-negative number $\epsilon \ge 0$ and policy $\targetpi$ we define as
\begin{align}
    {\epsilon'_{\targetpi} (\tilde{s}, \tilde{a})}
     := \begin{cases}
     \dfrac{\epsilon}{
     \min_{\pi \in D(\targetpi, \tilde{s}, \tilde{a})}\occstate^{\pi}(\tilde{s})
     } \quad &\text{if}\quad
     \tilde{s} \in \posStates^\targetpi \text{ and }
     \tilde{a} \ne \pi(s)\\
     0 \quad & \text{otherwise}.
     \end{cases},
     \label{eq.epsilon_prime}
\end{align}
where
\begin{align*}
D(\targetpi, \tilde{s}, \tilde{a}) = 
\big\{\pi:
\pi(\tilde{s}) = \tilde{a}
\text{ and }
\pi(s) = \targetpi(s) \text{ for all } s\in \posStates^\targetpi \backslash \{\tilde{s}\}
\big\}.
\end{align*}
Of course, in order for the above definition to be valid, we need to ensure that the denominator is non-zero, i.e. 
$\min_{\pi \in D(\targetpi, \tilde{s}, \tilde{a})}\occstate^\pi(\tilde{s}) > 0$. The following lemma ensures that this is the case.

\begin{lemma}\label{lm.mu_min_pos}
Let $\targetpi\in \PiDet$ be a deterministic policy.
Define
$\posStates^\targetpi$ as in \eqref{eq.s1}.
Let $\tilde{s}$ be an arbitrary state in $\posStates^\targetpi$ and $\pi$ be a deterministic policy such that
\begin{align*}
    \pi(s) = \targetpi(s)\quad  \text{ for all } \quad
    s\in \posStates^\targetpi \backslash \{\tilde{s}\}.
\end{align*}
Then
\begin{align*}
    \occstate^{\pi}(\tilde{s}) > 0.
\end{align*}
\end{lemma}
\begin{proof}
Assume that this is not the case and
$\occstate^{\pi}(\tilde{s}) = 0$.
Then
$s\notin \posStates^{\pi}$ and therefore
$\pi(s) = \targetpi(s)$ for all $\posStates^{\targetpi} \cap \posStates^{\pi}$.
Lemma 
\ref{lm.same_occupancy} (from section \nameref{app.background}) implies that
$\occstate^{\pi} = \occstate^{\targetpi}$ which is a contradiction since
$\occstate^{\targetpi}(s) > 0 = \occstate^{\pi}(\tilde{s})$.
Therefore the initial assumption was wrong and $\occstate^\pi(s) > 0$.
\end{proof}
\begin{proposition}\label{lm.rp_approx}
Let $\epsilon$ be a non-negative number\footnote{The case of $\epsilon=0$ is also covered by the lemma.}.
Denote by $\widehat{R}^{\targetpi}$ the solution of the optimization problem \eqref{prob.reward_poisoning_attack} and let
$\widehat{R}'$ be the solution to \eqref{prob.rp_approx} with 
$\epsilon'=\epsilon'_{\targetpi}$ where
$\epsilon'_{\targetpi}$ is defined as in Equation \eqref{eq.epsilon_prime}.
Then $\widehat{R}'$ satisfies the constraints of \eqref{prob.reward_poisoning_attack} and
\begin{align*}
0 \le \norm{\widehat{R}' -\overline{R}}_2 - 
\norm{\widehat{R}^{\targetpi} - \overline{R}}_2
\le 
\norm{\epsilon'}_2
\end{align*}
\end{proposition}
\begin{proof}
Before we proceed with the proof, note that
$\min_{\pi \in D(\targetpi, \tilde{s}, \tilde{a})}\occstate^{\pi}(\tilde{s}) > 0$ because of Lemma
\ref{lm.mu_min_pos}.
~\\
\textbf{Part 1:} 
We first prove that $\widehat{R}'$ satisfies the constraints of \eqref{prob.reward_poisoning_attack}, which automatically proves the left inequality by optimality of $\widehat{R}^{\targetpi}$. 
Set $Q=Q^{*, \widehat{R}'}$ and $V=V^{*, \widehat{R}'}$. Given Lemma \ref{lm.effective_q_v}, 
$(R, Q, V)$ satisfy the constraints of \eqref{prob.rp_approx}.
Define $\tilde{\pi}_{\dagger}\in \PiDet$ as
\begin{align}
    \tilde{\pi}_{\dagger}(s) = \argmax Q(s, a).
    \label{eq.def_tilde_pi_dagger}
\end{align}
It is clear that $Q^{\tilde{\pi}_{\dagger}, \widehat{R}'}=Q$.
Now note that since $\argmax Q(s, a) = \targetpi(s)$ for all states $s\in \posStates^\targetpi$, 
Lemma \ref{lm.score_diff_q_value} implies that
\begin{align*}
    \score^{\targetpi, \widehat{R}'} - 
    \score^{\tilde{\pi}_{\dagger}, \widehat{R}'} 
    &=
    \sum_{s}\occstate^{\targetpi}(s)\big(
    Q(s, \targetpi(s)) - Q(s, \tilde{\pi}_{\dagger}(s))
    \big)
    \\&=
    \sum_{s\in \posStates^\targetpi}\occstate^{\targetpi}(s)\big(
    Q(s, \targetpi(s)) - Q(s, \tilde{\pi}_{\dagger}(s))
    \big)
    \\&= 0.
\end{align*}
Given this, it suffices to prove that
\begin{align*}
  \score^{\tilde{\pi}_{\dagger}} \ge \score^\pi + \epsilon
  \quad \text{if} \quad
  \exists s \in \posStates^\targetpi: 
  \pi(s) \ne \targetpi(s),
\end{align*}
Our proof now proceeds in a similar fashion to the proof of Lemma 1 in \cite{rakhsha2020policy}: 
We start by proving the claim policies that effectively, differ from $\targetpi$ in only a single state.
We then generalise the claim for other policies via induction.

Concretely, we first claim that if
$\pi \in D(\targetpi, \tilde{s}, \tilde{a})$ for some $\tilde{s}\in \posStates^\targetpi$
and $\tilde{a}$, 
then
$\score^{\tilde{\pi}_{\dagger}, \widehat{R}'}
-\score^{\pi, \widehat{R}'} \ge \epsilon$.
To see why, note that
\begin{align*}
\score^{\tilde{\pi}_{\dagger}, \widehat{R}'}
-\score^{\pi, \widehat{R}'}
&=
\sum_{s}\occstate^{\pi}(s)\big(
Q(s, \tilde{\pi}_{\dagger}(s)) - 
Q(s, \pi(s)
\big)
\\&=
\occstate^{\pi}(\tilde{s})\big(
Q(\tilde{s}, \tilde{\pi}_{\dagger}(\tilde{s})) - 
Q(\tilde{s}, \pi(\tilde{s})
\big)
+
\sum_{s\ne \tilde{s}}\occstate^{\pi}(s)\big(
Q(s, \tilde{\pi}_{\dagger}(s)) - 
Q(s, \pi(s)
\big)
\\&\overset{(i)}{\ge}
\occstate^{\pi}(\tilde{s})\big(
Q(\tilde{s}, \tilde{\pi}_{\dagger}(\tilde{s})) - 
Q(\tilde{s}, \pi(\tilde{s})
\big)
\\&\overset{(ii)}{\ge}
\occstate^{\pi}(\tilde{s})\epsilon'(\tilde{s}, \pi(\tilde{s}))
\\&\overset{(iii)}{\ge}
\epsilon.
\end{align*}
where $(i)$ follows from the definition of
$\tilde{\pi}_{\dagger}$,
$(ii)$ follows from \eqref{constraint.ge} since $\tilde{s}\in \posStates^\targetpi$
and $(iii)$ follows from the definition of $\epsilon'_{\targetpi}$ in
\eqref{eq.epsilon_prime}.

We now generalize the above result by showing that if $\pi$ is a policy such that there exists state $\tilde{s}\in \posStates^\targetpi$ satisfying 
$\pi(\tilde{s})\ne \targetpi(\tilde{s})$, then
$\score^{\tilde{\pi}_{\dagger}, \widehat{R}'}
-\score^{\pi, \widehat{R}'} \ge \epsilon$.
We do this by induction on 
$d_{\posStates^\targetpi}(\pi, \targetpi)$ where
we define
$d_{\tilde{S}}(\pi, \pi')$ for
$\tilde{S}\subseteq S$ 
as
\begin{align*}
    d_{\tilde{S}}(\pi, \pi')
    := 
    \left|\left\{
    s\in \tilde{S}: \pi(s)\ne \pi'(s)
    \right\}\right|.
\end{align*}
For $d_{\posStates^\targetpi}(\pi, \targetpi)=1$, the claim is already proved since this is equivalent to
$\pi \in D(\targetpi, \tilde{s}, \tilde{a})$ for some $\tilde{s} \in \posStates^\targetpi$ and $\tilde{a}$. Suppose the claim holds for
all $\pi$ satisfying
$d_{\posStates^\targetpi}(\pi, \targetpi)\le k$ where $k\ge 1$. We prove it holds for all $\pi$ satisfying
$d_{\posStates^\targetpi}(\pi, \targetpi)=k+1 $.
Let $\pi$ be one such policy and note that by Lemma \ref{lm.score_diff_q_value}, 
\begin{align*}
\score^{\tilde{\pi}_{\dagger}, \widehat{R}'}
-\score^{\pi, \widehat{R}'}
&=
\sum_{s}\occstate^{\pi}(s)\big(
Q(s, \tilde{\pi}_{\dagger}(s)) - 
Q(s, \pi(s)
\big) \ge 0.
\end{align*}
On the other hand, again by Lemma \ref{lm.score_diff_q_value},
\begin{align*}
\score^{\tilde{\pi}_{\dagger}, \widehat{R}'}
-\score^{\pi, \widehat{R}'}
&=
\sum_{s}\occstate^{\tilde{\pi}_{\dagger}}(s)\big(
Q^{\pi, \widehat{R}'}(s, \tilde{\pi}_{\dagger}(s)) -
Q^{\pi, \widehat{R}'}(s, \pi(s))
\big) 
\\&=
\sum_{s\in \posStates^{\tilde{\pi}_{\dagger}}}\occstate^{\tilde{\pi}_{\dagger}}(s)\big(
Q^{\pi, \widehat{R}'}(s, \tilde{\pi}_{\dagger}(s)) -
Q^{\pi, \widehat{R}'}(s, \pi(s))
\big).
\\&\overset{(i)}{=}
\sum_{s\in \posStates^\targetpi}\occstate^{\tilde{\pi}_{\dagger}}(s)\big(
Q^{\pi, \widehat{R}'}(s, \tilde{\pi}_{\dagger}(s)) -
Q^{\pi, \widehat{R}'}(s, \pi(s))
\big).
\end{align*}
Where $(i)$ follows from the fact that
$\posStates^\targetpi = \posStates^{\tilde{\pi}_{\dagger}}$ by
Lemma \ref{lm.same_occupancy}. Therefore, there exists a state $s\in \posStates^\targetpi$ such that 
\begin{align*}
    Q^{\pi, \widehat{R}'}(s, \tilde{\pi}_{\dagger}(s)) -
Q^{\pi, \widehat{R}'}(s, \pi(s)) \ge 0.
\end{align*}
Define the policy
$\tilde{\pi}$ as
\begin{align*}
    \tilde{\pi}(\tilde{s}) = 
    \begin{cases}
    \targetpi(\tilde{s}) \quad&\text{if} \quad\tilde{s} = s\\
    \pi(\tilde{s}) \quad&\text{otherwise}
    \end{cases}.
\end{align*}
It follows that
\begin{align*}
\score^{\tilde{\pi}, \widehat{R}'}
-\score^{\pi, \widehat{R}'}
= 
\occstate^{\tilde{\pi}}(s)
\big(
Q^{\pi, \widehat{R}'}(s, \tilde{\pi}_{\dagger}(s)) -
Q^{\pi, \widehat{R}'}(s, \pi(s))\big)
\ge 0.
\end{align*}
However, $d_{\posStates^\targetpi}(\targetpi, \tilde{\pi})\le k-1$ and therefore
\begin{align*}
\score^{\tilde{\pi}_{\dagger}, \widehat{R}'} -
    \score^{\tilde{\pi}, \widehat{R}'}\ge \epsilon.
\end{align*}
Which proves the claim.

\textbf{Part 2: }
For the right inequality, define
$\widehat{R}"$ as
$\widehat{R}" := \widehat{R}^{\targetpi} - \epsilon'_{\targetpi}$.
We claim that 
\begin{align}
  \targetpi(s) \in  \argmax_a Q^{*, \widehat{R}^{\targetpi}}(s, a)
  \quad \text{for all} \quad
  s\in \posStates^\targetpi.
  \label{eq.dec9_1920}
\end{align}
To prove this,
define $\tilde{\pi}_{\dagger}\in \PiDet$ as
\begin{align*}
    \tilde{\pi}_{\dagger}(s) = \argmax_a Q^{*, \widehat{R}^{\targetpi}}(s, a).
\end{align*}
Now note that
\begin{align*}
  0 &\overset{(i)}{=}
  \score^{\targetpi, \widehat{R}^{\targetpi} } - 
  \score^{\tilde{\pi}_{\dagger}, \widehat{R}^{\targetpi}}
  \\&= 
  \sum_{s} \occstate^\targetpi(s) \cdot \left(
    Q^{\tilde{\pi}_{\dagger},\widehat{R}^{\targetpi}}(s, \targetpi(s))
    - Q^{\tilde{\pi}_{\dagger}, \widehat{R}^{\targetpi}}(s, \tilde{\pi}_{\dagger}(s))
  \right)
\\&\overset{(ii)}{=}
  \sum_{s} \occstate^\targetpi(s) \cdot \left(
    Q^{*,\widehat{R}^{\targetpi}}(s, \targetpi(s))
    - Q^{*, \widehat{R}^{\targetpi}}(s, \tilde{\pi}_{\dagger}(s))
  \right)
\end{align*}
where $(i)$ follows from optimality of $\targetpi$ in $\widehat{R}^{\targetpi}$ and
$(ii)$ follows from the fact that $Q^{*, \widehat{R}^\targetpi} = Q^{\tilde{\pi}_{\dagger}, \widehat{R}^{\targetpi}}$.
Therefore
$ Q^{*,\widehat{R}^{\targetpi}}(s, \targetpi(s)) = Q^{*, \widehat{R}^{\targetpi}}(s, \tilde{\pi}_{\dagger}(s))
$ for all $s \in \posStates^\targetpi$ which is equivalent to \eqref{eq.dec9_1920}.

  Given this result,
  we further assume that $\tilde{\pi}_{\dagger}(s)=\targetpi(s)$ for all
  $s\in \posStates^\targetpi$ since we did not originially specify how to break ties in the definition
  of $\tilde{\pi}_{\dagger}$.
  This also implies $\posStates^{\targetpi} = \posStates^{\tilde{\pi}_{\dagger}}$
  since $\occstate^{\targetpi} = \occstate^{\tilde{\pi}_{\dagger}}$
  by Lemma \ref{lm.same_occupancy}.

Now note that
given the construction of
$\widehat{R}"$, 
\begin{align*}
  Q^{\tilde{\pi}_{\dagger}, \widehat{R}^{\targetpi}}(s, \tilde{\pi}_{\dagger}(s)) =
  Q^{\tilde{\pi}_{\dagger}, \widehat{R}"}(s, \tilde{\pi}_{\dagger}(s)) 
\end{align*}
 for all $s$. This is because the rewards for the state-action pairs $(s, \tilde{\pi}_{\dagger}(s))$ were not modified.
 Since no reward has increased, this further implies that
$V^{\tilde{\pi}_{\dagger}, \widehat{R}"}=
V^{\tilde{\pi}_{\dagger}, \widehat{R}^{\targetpi}}
$.
Now note that
for all $s, a$,
\begin{align*}
    Q^{\tilde{\pi}_{\dagger}, \widehat{R}"}(s, a) - 
    Q^{\tilde{\pi}_{\dagger}, \widehat{R}^{\targetpi}}(s, a)
    &=
    \widehat{R}"(s, a) - \widehat{R}^{\targetpi}(s, a)
    \le -\epsilon'(s, a) \cdot \ind{s \in \posStates^{\targetpi}}
\end{align*}
Recall however that by definition of
$\tilde{\pi}_{\dagger}$,
\begin{align*}
  Q^{\tilde{\pi}_{\dagger}, \widehat{R}^{\targetpi}}(s, \tilde{\pi}_{\dagger}(s)) \ge 
  Q^{\tilde{\pi}_{\dagger}, \widehat{R}^{\targetpi}}(s, a).
\end{align*}
We can therefore conclude that
\begin{align*}
  Q^{\tilde{\pi}_{\dagger}, \widehat{R}"}(s, \tilde{\pi}_{\dagger}(s)) \ge 
  Q^{\tilde{\pi}_{\dagger}, \widehat{R}"}(s, a) + \epsilon'(s, a)\cdot \ind{s \in \posStates^{\targetpi}}.
\end{align*}
This means that 
$\left(
  R=\widehat{R}",
  Q=Q^{\tilde{\pi}_{\dagger}, \widehat{R}"},
  V=V^{\tilde{\pi}_{\dagger}, \widehat{R}"},
  \targetpi=\tilde{\pi}_{\dagger}\right)$
satisfy all of the constraints of \eqref{prob.rp_approx}.
Since $\targetpi(s) = \tilde{\pi}_{\dagger}(s)$
for all $s\in \posStates^{\targetpi}$, this means that
$\left(
  R=\widehat{R}",
  Q=Q^{\tilde{\pi}_{\dagger}, \widehat{R}"},
  V=V^{\tilde{\pi}_{\dagger}, \widehat{R}"},
  \targetpi=\targetpi\right)$
satisfy the constraints of \eqref{prob.rp_approx} as well.
Therefore, by optimality of $\widehat{R}'$,
\begin{align*}
    \norm{\overline{R} - \widehat{R}'}_2 &\le 
    \norm{\overline{R} - \widehat{R}"}_2
    \\&\le 
    \norm{\overline{R} - \widehat{R}^{\targetpi}}_2 + 
    \norm{\widehat{R}^{\targetpi} - \widehat{R}"}_2
    \\&=
    \norm{\overline{R} - \widehat{R}^{\targetpi}}_2 + 
    \norm{\epsilon'}_2.
\end{align*} 
\end{proof}


\section{\qgreedy~Algorithm}\label{app.qgreedy_algo}

In this section, we present the \qgreedy~algorithm and prove its correctness. Recall that this algorithm finds a solution to the optimization problem:
\begin{align*}
&\minmaxgapqval = 
 	\min_{\pi
 	\in \allowedpi
 	}\max_{s \in S_{\textnormal{pos}}^{\pi}}
 	\big(
 	Q^{*, \overline{R}}(s, \optpi(s)) - 
 	Q^{*, \overline{R}}(s, \pi(s))
 	\big).
\end{align*}
The intuition behind the \qgreedy~algorithm as follows. It starts by finding the state $s$ which has the highest value of 
$\delta(s) = \min_{a\in A_s^{\adm}} Q^{*, \overline{R}}(s, \optpi(s)) - 
Q^{*, \overline{R}}(s,a)$ among admissible state-action pairs---this value provides an upper bound on  $\minmaxgapqval$ and is tight if $s$ is reachable by any admissible policy. If this is not the case, then there might exists a policy $\pi$ which results in lower value of $\maxgapqval$ by not reaching $s$ (making $\occstate^{\pi}(s) = 0$). Therefore, after finding $s$, the algorithm proceeds by finding the set of state-action pairs that are ``connected'' to $s$ in that policies defined on these pairs reach $s$ with strictly positive probability.
These state action pairs are removed from the admissible set of state-action pairs, and the algorithm proceeds with the next iteration.
The output is defined by the minimum of all of the gaps $\delta$ found in each iteration,
and the policy can be reconstructed from the set of state-action pairs that are ``connected'' to the state that defines this gap.
~\\
The pseudo-code of \qgreedy~can be found in Algorithm \ref{alg:algorithm} and provides a more detailed description of the algorithm. 
\begin{algorithm}[tb]
\caption{\qgreedy}
\label{alg:algorithm}
\textbf{Input}: MDP $\overline{M}$,
admissible action set $A_{s}^{\adm}$ for each state $s$.
\\
\textbf{Output}: 
$\minmaxgapqval$, 
Policy $\pi \in \argmin_{\pi} \Delta_{Q}^\pi$.
\begin{algorithmic}[1] 
\STATE Calculate Q-values $Q^{*, \overline{R}}$.
\STATE Let $\optpi(s) = \max_{a}Q^{*, \overline{R}}(s, a)$.
\STATE Let $\text{ADM}^{(0)}=
\{
(s, a) | a\in A_s^{\adm}
\}
$.
\STATE Let $\tilde{S}^{(0)} = S$.
\STATE Let $S_{\sigma} = 
\{s| \sigma(s) \ne 0\}
$.
\STATE Let $t=0$.
\WHILE{$S_{\sigma}\subseteq \tilde{S}^{(t)}$}
\STATE Let $\text{ADM}^{(t)}_s = 
\{a| (s, a) \in \text{ADM}^{(t)}\}$ for all $s \in \tilde{S}^{(t)}$.
\STATE Let $\delta^{(t)}(s) = \min_{a\in \text{ADM}_s^{(t)}}
Q^{*, \overline{R}}(s, \optpi(s)) - 
Q^{*, \overline{R}}(s,a)
$ for all 
$s\in \tilde{S}^{(t)}$.
\STATE Let $s_{t} = \argmax_{s\in \tilde{S}^{(t)}} \delta^{(t)}(s)$ and
$\Delta_t=\delta^{(t)}(s_t)$.
\STATE
Let $\pi_t(s) = \argmin_{a\in \text{ADM}_s^{(t)}}
Q^{*, \overline{R}}(s, \optpi(s)) - 
Q^{*, \overline{R}}(s,a)
$ for all $s\in \tilde{S}^{(t)}$
and choose $\pi_t(s)$ arbitrarily otherwise.
\STATE Let $\tilde{S}^{(t + 1)} = \tilde{S}^{(t)}$ and
$\text{ADM}^{(t + 1)} = 
\text{ADM}^{(t)}
$.
\STATE Let $S_{\delta} = \{s_t\}$
\WHILE{$S_{\delta} \ne \emptyset$}
\label{alg:inner_loop_begin}
\STATE
$\tilde{S}^{(t + 1)} = \tilde{S}^{(t + 1)} \backslash S_{\delta}$.
\label{alg:remove_s}
\STATE
$\text{ADM}^{(t + 1)} = 
\text{ADM}^{(t + 1)}
\backslash \{(s, a)|  P(s, a, \tilde{s}) > 0 \text{ for some } \tilde{s}\notin \tilde{S}^{(t + 1)}\}$.
\label{alg:remove_sa}
\STATE
Let $S_{\delta} = \{
s \in \tilde{S}^{(t + 1)}| (s, a) \notin \text{ADM}^{(t + 1)} \text{ for all } a
\}
$
\ENDWHILE
\label{alg:inner_loop_end}
\STATE $t=t + 1$.
\ENDWHILE
\STATE Let $t^*=\argmin_{t}\Delta_t$
\STATE \textbf{return} $\Delta_{t^*}, \pi_{t^*}$
\end{algorithmic}
\end{algorithm}
The following result formally shows that \qgreedy~outputs a correct result. 
\begin{lemma}
Let $\Delta_{A}, \pi_{A}$ denote the output of the algorithm \ref{alg:algorithm}. Then it holds that
\begin{align*}
    \minmaxgapqval = \Delta_{A} = \Delta_{Q}^{\pi_{A}},
\end{align*}
where $\Delta_{Q}^{\pi}$ is defined as
\begin{align}
    \maxgapqval = \max_{s\in \posStates^{\pi}}\left(
Q^{*, \overline{R}}(s, \optpi(s))-
Q^{*, \overline{R}}(s, \pi(s))
\right).
\label{eq.q_gap}
\end{align}
\end{lemma}\begin{proof}
Before we discuss the algorithm's correctness,  observe that it guaranteed to terminate since $S$ is finite, $S_{\sigma}$ is nonempty, and $|\tilde{S}^{(t)}|$ strictly decreases in each iteration. We now prove the correctness of the algorithm. We divide the proof into three parts.
~\\
\textbf{Part 1:} We show that
$\Delta_{Q}^{\pi_{A}} \le \Delta_{A}$.
In order to understand why this is the case, for each iteration $t$, consider the policy $\pi_t$. 
We first claim that
$\posStates^{\pi_{t}} \subseteq \tilde{S}^{(t)}$. Note that this is the reason we have defined $\pi_t$ only for the states $s\in \tilde{S}^{(t)}$ in the algorithm, the value of $\pi_t(s)$ for $s\notin \tilde{S}^{(t)}$ is not important and can be chosen arbitrarily.
~\\
In order to prove the claim, note that it is obviously true for $t=0$ because $\tilde{S}^{(0)}=S$.
As for $t\ge 1$, observe that 
an agent following $\pi_t$ initially starts in a state 
in $S_{\sigma} \subseteq\tilde{S}^{(t)}$
and therefore the agent is in $\tilde{S}^{(t)}$ in the beginning.
Therefore, in order to reach a state in $S \backslash \tilde{S}^{t}$,
at some point the agent would need to take an action that could lead to, i.e., with
strictly positive probability would transition to, 
$S \backslash \tilde{S}^t$.
This is not possible however as 
$\pi_t(s)\in \text{ADM}_s^{(t)}$
and 
all state-action pairs in $\tilde{S}^{(t)}$ that could lead to  $S \backslash \tilde{S}^{(t)}$ were removed from 
$\text{ADM}^{(t)}$ during its construction in Line \ref{alg:remove_sa}.
~\\
Given this result, it is clear that 
\begin{align*}
    \Delta_t
    &=
    \delta^{(t)}(s_t)
    \\&=
    \max_{s\in \tilde{S}^{(t)}}
        \delta^{(t)}(s)
    \\&\ge 
    \max_{s\in \posStates^{\pi_t}}
        \delta^{(t)}(s)
    \\&=
    \max_{s\in \posStates^{\pi_t}}
        Q^{*, \overline{R}}(s, \optpi(s)) - 
        Q^{*, \overline{R}}(s, \pi_t(s))
    \\&=
    \Delta^{\pi_t}_{Q}.
\end{align*}
Therefore, 
$\Delta_{Q}^{\pi_t}\le \Delta_{t}$.
Since this holds for all $t$, the claim is proved.
~\\
\textbf{Part 2:}
We show that 
$\minmaxgapqval \ge  \Delta_{A}$.
We use proof by contradiction.
Assume that this is not the case and
$\minmaxgapqval <  \Delta_{A}$.
This means that
$\Delta_{Q}^\pi < \Delta_{A}$ for some $\pi\in \allowedpi$. 
Now, observe that since the loop terminates for some $t$,
there exists a $t$ such that
$S_{\sigma} \not\subseteq \tilde{S}^{(t)}$. Since
$S_{\sigma}\subseteq\posStates^\pi$, this means that there exists a $t$ such that $\posStates^\pi \not\subseteq \tilde{S}^{(t)}$. We claim that this contradicts the assumption $\Delta_{Q}^\pi < \Delta_{A}$. 
Concretely, we will use the assumption 
$\Delta_{Q}^\pi < \Delta_{A}$ and
show by induction on $t$ that 
$\posStates^\pi \subseteq \tilde{S}^{(t)}$ and $(s, \pi(s))\in \text{ADM}^{(t)}$ for all $s\in \posStates^\pi$. 
~\\
For $t=0$, the claim holds because since $\pi \in \allowedpi$.
Assume the claim holds for $t$, we will show that it holds for $t + 1$ as well.
We first claim that $s_t\notin \posStates^\pi$. Concretely,
$(s_t, \pi(s_t)) \in \text{ADM}^{(t)}$ by the induction hypotheses. If $s_t \in \posStates^\pi$, this would imply that
\begin{align*}
    \Delta_{Q}^\pi
    &\ge 
    Q^{*, \overline{R}}(s_t, \optpi(s_t))
    -
     Q^{*, \overline{R}}(s_t, \pi(s_t))
     \\&\ge 
     \min_{a\in \text{ADM}^{t}_{s_t}}
     Q^{*, \overline{R}}(s_t, \optpi(s_t))
    -
     Q^{*, \overline{R}}(s_t, a)
     \\&=
     \delta^{(t)}(s_t)
     \\&= \Delta_{t}
     \\&
     \ge \Delta_{A},
\end{align*}
which contradicts the assumption $\Delta_{Q}^{\pi} < \Delta_{A}$. Therefore, 
$s_{t}\notin \posStates^{\pi}$.
Now consider the inner loop in lines \ref{alg:inner_loop_begin}-\ref{alg:inner_loop_end}. 
We claim by induction that the loop does not remove any states $s\in \posStates^\pi$ 
from $\tilde{S}^{(t+1)}$ and
does not remove
 any state-action pairs $(s, \pi(s))$ such that $s\in \posStates^\pi$ from 
 $\text{ADM}^{(t+1)}$.
 Since $s_t\notin \posStates^\pi$, this is true the first time
line \ref{alg:remove_s} is executed.
In each execution of the loop,
assuming the constraint is not violated in line
\ref{alg:remove_s}, then it will not be violated in line \ref{alg:remove_sa} either.
This is because if $(s, \pi(s))$ is removed from $\text{ADM}^{(t+1)}$ for some $s\in \posStates^\pi$,
then $P(s, a, s') > 0$ for some $s'\notin \tilde{S}^{(t+1)}$. 
However, $P(s, a, s') > 0$ implies that  $s'\in \posStates^\pi$ and therefore $s'\in \posStates^\pi$ was already removed from $\tilde{S}^{(t + 1)}$. Likewise, if a state $s\in \posStates^\pi$ is removed in later executions of \ref{alg:remove_s}, then it must be the case that $(s, a)\notin \text{ADM}^{(t+1)}$ for all $a$. This means that at some point, the state-action pair $(s, \pi(s))$ must have been removed from $\text{ADM}^{(t+1)}$ which means the constraint must have already been violated. Therefore, the constraint is not violated at any point. This means that the induction is complete and we have reached a contradiction using the assumption
$\Delta_{Q} < \Delta_{A}$. Therefore, the assumption was wrong and
$\Delta_{Q} \ge \Delta_{A}$.
~\\
\textbf{Part 3:}
Putting both parts together, note that
\begin{align*}
    \Delta_{A} \le \Delta_{Q} \le
    \Delta_{Q}^{\pi_{A}} \le \Delta_{A},
\end{align*}
where the first inequality follows from Part 2, the second inequality follows from the definition of
$\Delta_{Q}$ and the final inequality follows from
Part 1. Therefore the proof is complete.
\end{proof}


\section{Proofs of the Results in Section \nameref{sec.setting}}\label{app.setting_proofs}

In this section, we provide proofs of our results in Section \nameref{sec.setting}, namely, Proposition \ref{prop.solvability.reward_design}.

\subsection{Proof of Proposition \ref{prop.solvability.reward_design}}
\textbf{Statement:} {\em If $\allowedpi$ is not empty, there always exists an optimal solution to the optimization problem \eqref{prob.reward_design}.}
\begin{proof} to prove the statement, we first show that the following three claims hold.
\begin{claim}
Consider a function that evaluates the objective of the optimization problem
	\eqref{prob.reward_design} for a given $R$:
\begin{align*}
    l(R) = \max_{\pi \in \optEpsDet(R)} \norm{\overline{R} - R}_2 -\lambda\score^{\pi, \overline{R}}.
\end{align*} 
This function, $l(R)$, is lower-semi continuous.
\end{claim}
 \begin{proof}
 First note that the function is real-valued (i.e., $l(R) \notin\{\infty, -\infty\}$) as the set of all deterministic policies is finite. To prove the claim, we need to show that
for all $R$:
\begin{align*}
    \forall \delta : \exists \alpha : \forall \tilde R: 
    \norm{\tilde{R} - R}_2 \le \alpha \implies 
    l(\tilde{R}) \ge l(R) - \delta.
\end{align*}
Assume to the contrary that there exists an $R$ such that
\begin{align*}
    \exists \delta: \forall \alpha: 
    \exists \tilde{R}_\alpha:
    \norm{\tilde{R}_\alpha - R}_2 \le \alpha \land
    l(\tilde{R}_\alpha) < l(R) - \delta.
\end{align*}
By setting $\alpha_i=\frac{1}{2^i}$, we obtain a series
$\{R_i\}_{i=1}^\infty$ such that
$R_i$ tends to $R$ and
\begin{align*}
    \forall i: 
    l(R_i) < l(R) - \delta.
\end{align*}
Take $\pi$ to be an arbitrary policy in 
\begin{align*}
\argmax_{\pi \in \optEpsDet(R)}
\norm{\overline{R} - R}_2 -\lambda\score^{\pi, \overline{R}}.
\end{align*}
We claim that there exists $N$ such that
\begin{align*}
    \forall i \ge N: 
    \pi \in \optEpsDet(R_i).
\end{align*}
If this would not be the case, then there would exist an infinite sub-sequence of $R_i$s such that for all of them 
$\pi \notin \optEpsDet(R_i)$
and therefore
there exists a deterministic $\tilde{\pi}_i$ such that
$\score^{\tilde{\pi}_i, R_i} \ge \score^{\pi, R_i} + \epsilon$. Since the number of deterministic policies is finite, at least one of these deterministic policies would occur infinitely often. This would mean that there is a policy $\tilde{\pi}$ and an infinite subsequence of $R_j$s that tends to $R$, and for all of the $j$s in the subsequence
\begin{align*}
    \score^{\tilde{\pi}, R_j} \ge \score^{\pi, R_j} + \epsilon.
\end{align*}
Since $\score(R, \tilde{\pi})$ is continuous in $R$ for fixed $\tilde \pi$, this would imply that
\begin{align*}
    \score^{\tilde{\pi}, R} \ge \score^{\pi, R} + \epsilon,
\end{align*}
which contradicts the assumption 
$\pi \in \optEpsDet(R)$. Therefore, as we stated above, there exists $N$ such that
\begin{align*}
    \forall i \ge N: 
    \pi \in \optEpsDet(R_i).
\end{align*}

Now, note that
\begin{align*}
    \forall i \ge N: 
    l(R_i) = \max_{\pi \in \optEpsDet(R_i)} \norm{\overline{R} - R_i}_2 -\lambda\score^{\pi, \overline{R}} \ge \norm{\overline{R} - R_i}_2 -\lambda\score^{\pi, \overline{R}}.
\end{align*}
Therefore, we have that
\begin{align*}
-\delta &>
    l(R_i) - l(R) \\&\ge 
    \norm{\overline{R} - R_i}_2 -\lambda\score^{\pi, \overline{R}} - 
    \norm{\overline{R} - R}_2 + \lambda\score^{\pi, \overline{R}} = \norm{\overline{R} - R_i}_2 - \norm{\overline{R} - R}_2,
\end{align*}
which is a contradiction, since $\norm{\overline{R} - R_i}_2 - \norm{\overline{R} - R}_2$ goes to 0 as $i \to \infty$. This proves the claim.
 \end{proof}
 
 \begin{claim}
 The set $\{R: \optEpsDet(R) \subseteq \allowedpi \}$ is closed. 
 \end{claim}
 \begin{proof}
To see why, note that $R$ is in this set if and only if
\begin{align*}
\exists \pi \in \allowedpi: \forall
\tilde{\pi} \in \PiDet \backslash \allowedpi: 
\score^{\pi, R} \ge \score^{\tilde{\pi}, R} + \epsilon
\end{align*}
For a fixed $\pi, \tilde \pi$, 
the set
$\{R: \score^{\pi, R} \ge \score^{\tilde{\pi}, R} + \epsilon\}$
is closed. Since the set $\{R: \optEpsDet(R) \subseteq \allowedpi\}$ is a finite union of a finite intersection of such sets, $\{R: \optEpsDet(R) \subseteq \allowedpi\}$ is closed as well.
 \end{proof}
\begin{claim}
If 
$\allowedpi$ is not empty, \eqref{prob.reward_design} is feasible.
\end{claim}
\begin{proof}
    Assume $\pi \in \allowedpi$. Let $t\ge 0$ be an arbitrary positive number. Define reward function $R$ as
    \begin{align*}
      R(s, a) = t\cdot\ind{ 
      \occstate^\pi(s) > 0 \land 
      a=\pi(s)
      }.
    \end{align*}
    Note that this implies
    \begin{align*}
        \score^{\pi, R} &= \sum_{s}\occstate^\pi(s) R(s, \pi(s)) 
        \\&=
        \sum_{s\in \posStates^\pi}\occstate^\pi(s) R(s, \pi(s))
        \\&=
        \sum_{s\in \posStates^\pi}\occstate^\pi(s) \cdot t
        \\&=
        t.
    \end{align*}
    We show that if $t$ is large enough, $R$ is feasible, i.e., for all $\tilde{\pi}\notin\allowedpi$, $\tilde{\pi}\notin\optEpsDet(R)$. Since the number of deterministic policies is finite, it suffices to show  for a fixed that $\tilde{\pi}\notin\allowedpi$, $\tilde{\pi}\notin\optEpsDet(R)$ for large enough $t$. To prove this, note that since $\tilde{\pi}\notin \allowedpi$, there exists a state $\tilde{s}$ such that
    $\occstate^{\tilde{\pi}}(\tilde{s}) > 0$ and
    $\tilde{\pi}(\tilde{s})\notin A^{\adm}_{\tilde{s}}$. Since $\pi$ was admissible, this means that either
    $\pi(\tilde{s})\ne \tilde{\pi}(\tilde{s})$ or $\occstate^\pi(\tilde{s}) = 0$. Either way, $R(s, \tilde{\pi}(\tilde{s})) = 0$. Therefore,
    \begin{align*}
        \score^{\pi, R} &= \sum_{s}
        \occstate^{\tilde{\pi}}(s)\cdot R(s, \tilde{\pi}(s)) 
        \\&=
        \sum_{s\ne \tilde{s}}
        \occstate^{\tilde{\pi}}(s)\cdot
        R(s, \tilde{\pi}(s)) 
        \\&\le
        t - \occstate^{\tilde{\pi}}(\tilde{s}) \cdot t.
    \end{align*}
    Setting $t> \frac{\epsilon}{\occstate^{\tilde{\pi}}(\tilde{s})}$ proves the claim.
\end{proof}

Let us now prove the statement of the proposition.  Since the optimization problem 
\eqref{prob.reward_design} has no limits on $R$, in other words the set of all feasible $R$ in the optimization problem is not bounded, we cannot claim that the feasible set is compact. Note however that since the optimization problem is feasible for any fixed $\overline{R}$, there is an upper bound on its value. Furthermore, the second term in the objective, i.e, $\score^{\pi, \overline{R}}$ is bounded for any fixed $\overline{R}$. This means that for every fixed $\overline{R}$, there exists a number $\Theta$ such that the optimization problem \eqref{prob.reward_design} is equivalent to
\begin{align*}
\notag
&\quad
\min_{R}\max_{\pi}
\norm{\overline{R}-R}_2 -\lambda\score^{\pi, \overline{R}}
\\
&\quad \mbox{ s.t. }  \quad
\optEpsDet (R) \subseteq \allowedpi
\\&\quad \quad \quad\quad
\pi \in\optEpsDet(R)
\\&\quad\quad\quad\quad
\norm{R - \overline{R}}_2 \le \Theta,
\end{align*}
This turns the problem into minimizing a lower-semi-continous function over a compact set which has an optimal solution (i.e., the infimum is attainable).
\end{proof}


\section{Proofs of the Results in Section \nameref{sec.computational_challenges} and Additional Results}\label{app.computational_challenges}

In this section, we provide proofs of our results in Section \nameref{sec.computational_challenges}, namely, 
Theorem \ref{thm.copmutational_hardness}, and Proposition \ref{prop.reward_design.approx}.
We also provide
an additional computational complexity result for the optimization problem \eqref{prob.reward_design.approx}.


\subsection{Proof of Theorem \ref{thm.copmutational_hardness}}

{\bf Statement:} {\em 
For any constant $p \in (0,1)$, it is NP-hard to distinguish between instances of \eqref{prob.reward_design_simple} that have optimal values at most $\xi$ and instances that have optimal values larger than $\xi \cdot \sqrt{(|S| \cdot |A|)^{1-p}}$.
The result holds even when the parameters $\epsilon$ and $\gamma$ in \eqref{prob.reward_design_simple} are fixed to arbitrary values subject to $\epsilon > 0$ and $\gamma \in (0,1)$.
}

\begin{proof}

\begin{figure}[h]
\centering
\tikzset{
->, 
>=Stealth, 
node distance=3cm, 
every state/.style={thick, fill=gray!10}, 
}
\resizebox{1\linewidth}{!}{
    \begin{tikzpicture}[scale=0.5]
    \tikzstyle{every node}=[font=\footnotesize] 
    
    \node[state, initial, initial where=above] (s0) {$s_0$};
    \node[state, below of=s0, yshift=16mm, scale=0.2] (a0) {};
    
    \foreach \i in {0,...,9}
    {
    \pgfmathsetmacro\j{\i*4}
    \node[state, below left of=a0, draw=black!\j, fill=black!\i, xshift=0mm -\i mm, yshift=10mm -\i mm] {};
    \node[state, below left of=a0, draw=black!\j, fill=black!\i, xshift=15mm -\i mm, yshift=10mm -\i mm] {};
    \node[state, below left of=a0, draw=black!\j, fill=black!\i, xshift=30mm -\i mm, yshift=10mm -\i mm] {};
    \node[state, below left of=a0, draw=black!\j, fill=black!\i, xshift=60mm -\i mm, yshift=10mm -\i mm] {};
    \node[state, below left of=a0, draw=black!\j, fill=black!\i, xshift=85mm -\i mm, yshift=5mm -\i mm] {};
    }
    
    \node[state, below left of=a0, xshift=-10mm] (s1) {$s_1$};
    \node[state, below left of=a0, xshift=5mm] (s2) {$s_2$};
    \node[state, below left of=a0, xshift=20mm] (s3) {$s_3$};
    \node[below left of=a0, xshift=35mm] (sdots) {$\dots~\dots$};
    \node[state, below left of=a0, xshift=50mm] (sm) {$s_{3k}$};
    
    \node[state, below left of=a0, xshift=75mm, yshift=-5mm] (sx) {$s_*$};
    \node[state, below of=sx, xshift=-5mm, yshift=5mm, scale=0.2] (ax) {};
    
    \foreach \i in {0,...,9}
    {
    \pgfmathsetmacro\j{\i*4}
    \node[state, below of=sm, draw=black!\j, fill=black!\i, xshift = 50mm -\i mm, yshift=0mm -\i mm] {};
    \node[state, below left of=s1, draw=black!\j, fill=black!\i, xshift = 2mm -\i mm, yshift=33mm -\i mm] {};
    }
    
    
    \node[state, below of=sm, xshift = 40mm, yshift=-10mm] (sxx) {$\tilde{s}_1$};
    \node[state, below left of=s1, xshift=25mm, yshift=-5mm] (t1) {$t_1$};
    \node[state, below left of=s1, xshift=40mm, yshift=-5mm] (t2) {$t_2$};
    \node[below left of=s1, xshift=57mm, yshift=-5mm] (dots*) {$\dots~\dots$};
    \node[state, below left of=s1, xshift=75mm, yshift=-5mm] (tl) {$t_l$};
    
    \node[state, below of = dots*, yshift=-5mm] (final) {$\tilde{s}_f$};
    \node[state, below left of=s1, xshift=-8mm, yshift=23mm] (s1x) {$\tilde{s}_0$};

    \draw 
    (s0) edge[left] node{} (a0)
    (a0) edge[dashed, thick, bend right, above] node{} (s1)
    (a0) edge[dashed, thick, bend right=20, left] node{} (s2)
    (a0) edge[dashed, thick, bend right=5, left] node{} (s3)
    (a0) edge[dashed, thick, bend left, above] node{} (sm)
    (a0) edge[dashed, thick, bend left=30, above] node{} (sx)
    
    (ax) edge[dashed, thick, bend left=31, above] node{} (t1)
    (ax) edge[dashed, thick, bend left=21, left] node{} (t2)
    (ax) edge[dashed, thick, bend left=15, left] node{} (tl)
    
    (s1) edge[left] node{} (t1)
    (sm) edge[right, pos=0.5] node{} (dots*)
    (s2) edge[left, pos=0.2] node{} (t1)
    (s1) edge[left, pos=0.7] node{} (dots*)
    (s2) edge[, pos=0.5] node{} (t2)
    (s3) edge[above, pos=0.3] node{} (t2)
    (sdots) edge[right, pos=0.27] node{} (tl)
    (sdots) edge[left, pos=0.5] node{} (t1)
    (sdots) edge[right, pos=0.5] node[xshift=1mm, fill=white,inner sep=1pt]{$x_{ij} = \frac{m}{\gamma}(\frac{\epsilon}{1-\gamma} + \delta) - \gamma$} node[left,fill=white,xshift=0.1mm,inner sep=0pt]{\fbox{$a_{j}$}} (dots*)
    
    (s2) edge[above, draw=red, bend right=30] node{} (s1x)
    (s3) edge[above, draw=red, bend right] node{} (s1x)
    (sm) edge[above, draw=red, bend right, pos=0.75] node{\fbox{$a_\dagger$} $z_{3k} = 0$} (s1x)
    (s1) edge[below, draw=red, bend right=25] node[below, fill=white,xshift=0.2mm,yshift=-0.2mm,inner sep=0pt]{\fbox{$a_\dagger$}} node[yshift=-5mm]{$z_{1} = 0$} (s1x)
    
    (sx) edge[right, bend left=50, right] node[right,pos=0.35,fill=white,xshift=-0.1mm,inner sep=0pt]{\fbox{$a_1$}} node[pos=0.5]{$y = \gamma \cdot \frac{k}{l} + \frac{m}{ \gamma} (\frac{\epsilon}{1-\gamma} + \delta)$} (sxx)
    (sx) edge[right, thick, draw=red, bend left, pos=0.55] node{$z_*=0$} node[left, pos=0.55, fill=white,xshift=0.2mm,inner sep=0pt]{\fbox{$a_\dagger$}} (ax)
    
    (t1) edge[left, pos=0.5] node{$\omega_1=0$} (final)
    (t2) edge[right, pos=0.6] node{$\omega_2=0$} (final)
    (tl) edge[right, pos=0.65] node{$\omega_l=0$} (final)
    (final) edge[loop below] node{$0$} (final)
    (sxx) edge[below, bend left, right, pos=0.5] node{\quad$r=0$} (final)
    (s1x) edge[left, bend right=50, left, pos=0.5] node{$u=0\quad$} (final)
    ;
    \end{tikzpicture}
}
\caption{Reduction. Solid edges represent actions and dashed edges represent non-deterministic transitions. Red edges are {\em not} admissible.
Labels in boxes are names of some important actions in the reduction, of the corresponding edges; other values on the edges denote rewards for the corresponding actions. Each of the states $s_1, \dots, s_{3k}$ and $s_*$ has $N$ copies, and copies of each state are connected to other states in the same way. \label{fig:reduction-P1}}
\end{figure}
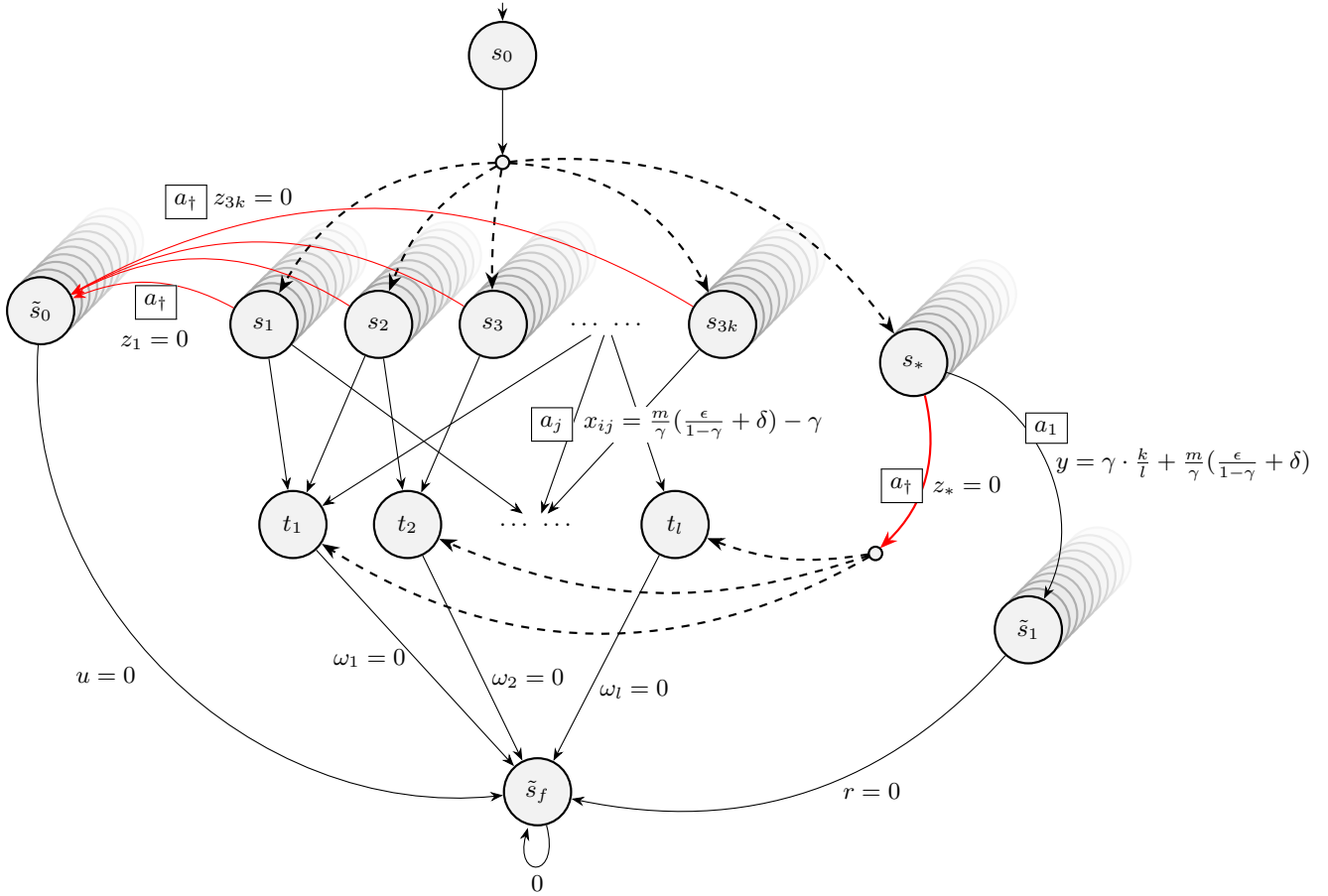

We show a reduction from the NP-complete problem {\sc Exact-3-Set-Cover} (X3C) \cite{karp1972reducibility,garey1979computers}.
An instance of X3C is given by a set $E= \{e_1, \dots, e_{3k}\}$ of $3k$ elements and a collection $\mathcal{S}$ of $3$-element subsets of $E$.
It is a yes-instance if there exists a sub-collection $\mathcal{Q} \subseteq \mathcal{S}$ of size $k$ such that $\cup_{S \in \mathcal{Q}} S = E$, and a no-instance otherwise.

Given an X3C instance, we construct the following instance of \eqref{prob.reward_design_simple}.
The underlying MDP is illustrated in Figure~\ref{fig:reduction-P1} with the following specifications, where we let 
$N = \left\lceil 3 k \cdot \varphi^{1-p} \cdot \left(\frac{9l}{\gamma}\right)^2 \right\rceil
$, $\delta = \frac{\gamma^2}{8 \cdot m \cdot l}$, and $m=(3k+1)N$; the value of $\varphi$ will be defined shortly.
Intuitively, we need $N$ to be sufficiently large and $\delta >0$ to be sufficiently close to $0$.

\begin{itemize}

\item
$s_0$ is the starting state, in which taking the only action available leads to a non-deterministic transition to each of the states $s_1,\dots, s_{3k}$ and $s_*$ with probability $\frac{1}{m}$.

\item
In each state $s_i$, $i =1 ,\dots, 3k$, taking action $a_\dagger$ leads to the state transitioning to $\tilde{s}_0$, yielding a reward $z_i := 0$.
Action $a_\dagger$ is {\em not} admissible in any of these states.

\item
Suppose there are $l$ subsets $S_1, \dots, S_l$ in the collection $\mathcal{S}$.
For each subset $S_j$, we create a state $t_j$.
We also create $l$ actions $a_1,\dots, a_l$.
If $e_i \in S_j$, then taking action $a_j$ in state $s_i$ leads to the state transitioning to $t_j$ and yields a reward $x_{ij} := \frac{m}{\gamma}(\frac{\epsilon}{1-\gamma} + \delta) - \gamma$.
From each state $t_j$, the only action available leads to state $\tilde{s}_f$, yielding a reward $\omega_j := 0$.

\item
In state $s_*$, taking action $a_\dagger$ leads to the state transitioning to each $t_j$ with probability $1/l$, yielding a reward $z_* := 0$; this action is {\em not} admissible.
Taking the other available action --- let it be $a_1$ --- leads to $\tilde{s}_1$ and a reward $y := \gamma \cdot \frac{k}{l} + \frac{m}{ \gamma} \left(\frac{\epsilon}{1-\gamma} + \delta\right)$ is yielded.
\item 
In states $\tilde{s}_0$ there is a single action yielding a reward of $u=0$ and transitioning to the state $\tilde{s}_f$. Similarly, in $\tilde{s}_1$ there is a single action yielding a reward of $r=0 $ and transitioning to 
$\tilde{s}_f$.
\item
In state $\tilde{s}_f$, there is a single action, yielding a reward $0$ and transitioning back to $s_0$
\end{itemize}
After the above construction is done, we start to create copies of some of the states, which will be essential for the reduction.
We repeat the following step $N-1$ times to create $N-1$ sets of copies:
\begin{itemize}
\item
Create a copy of each state
$s \in \{s_1, \dots, s_{3k}\} \cup \{s_*, \tilde{s}_0,\tilde{s}_1\}$.
Connect these new copies in the following way:
Connect the copy $\bar{s}$ of each state $s$ to the copy $\bar{s}'$ of every other state $s'$ (only those created in this step) the same way $s$ and $s'$ are connected;
In addition, connect $\bar{s}$ to every other state $s'$ which does not have a copy the same way 
$s$ and $s'$ are connected.
\end{itemize}
Finally, we let
\begin{align*}
    \varphi = \left(
6k \cdot \left(\frac{9l}{\gamma}\right)^2
\cdot
(3k + l + 5) \cdot (l + 1)
\right)^{1/p}.
\end{align*}
Now note that
$\varphi\ge 1$ and therefore $N\ge 3$ and $N \le 2(N-1) \le 2 \cdot \left( 3k \cdot \varphi^{1-p} \cdot
\left(\frac{9l}{\gamma}\right)^2 \right)$, which implies
\begin{equation}
\label{eq:bound-instance-size}
|S| \cdot |A| \le 
    N \cdot (3k + l + 5) \cdot (l + 1) 
\le \varphi^{1-p}\cdot 2 \cdot \left( 3k \cdot 
\left(\frac{9l}{\gamma}\right)^2 \right)
\cdot (3k + l + 5) \cdot (l + 1) 
\le 
\varphi.
\end{equation}

Without loss of generality, we can also assume that
$k \le |\mathcal{S}|$;
the X3C problem is always a no-instance when $|\mathcal{S}| < k$ since this implies $|E| > 3|\mathcal{S}|$. 
Therefore
the X3C problem remains NP-hard with this restriction.\footnote{We can further assume that $k \ge 1/\epsilon$ without loss of generality given that $\epsilon$ is fixed, in which case we have $x_{ij} = \frac{m}{\gamma} \left(\frac{\epsilon}{1-\gamma} + \delta\right) - \gamma > 0$, so the reduction would not rely on negative rewards.}
The restriction implies that $k + 1 \le 2|\mathcal{S}| = 2l$, which will be useful in the sequel.

Observe that by the above construction, regardless of what policy is chosen, the agent will end up in $\tilde{s}_f$ in exactly three steps and will stay there forever. Therefore, the reward of this state does not matter as it cancels out when considering
$\score^{\pi, R} - \score^{\tilde{\pi}, R}$ for all
$\pi, \tilde{\pi}, R$. We now proceed with the proof.

\paragraph{Correctness of the Reduction}
Let $\xi = \sqrt{k}$.
We will show next that if the X3C instance is a yes-instance, then this \eqref{prob.reward_design_simple} instance admits an optimal solution $R$ with $\norm{\overline{R} - R}_2 \le \xi$; otherwise, any feasible solution $R$ of \eqref{prob.reward_design_simple} is such that $\norm{\overline{R} - R}_2 > \xi \cdot \sqrt{(|S| \cdot |A|)^{1-p}}$.

\medskip

{\em First, suppose that the X3C instance if a yes-instance.}
By definition, there exists a size-$k$ set $Q \subseteq \{1,\dots, l\}$, such that $\cup_{j \in Q} S_j = E$.
Consider a solution $R$ obtained by increasing each $\omega_j$ to $1$ if $j \in Q$.
Since $|Q| = k$ , we have $\left\| \overline{R} - R \right\|_2 = \sqrt{k} = \xi$.
We verify that $R$ is a feasible solution, with $\optEpsDet (R) \subseteq \allowedpi$. 
In particular, in each state $s_i$, $i= 1, \dots, 3k$, we verify that taking action $a_\dagger$ would results in a loss of greater than $\epsilon$ in the policy score, as compared with taking an action $a_j$ such that $e_i \in S_j$ and $j \in Q$; we know such an $a_j$ exists because $Q$ is an exact set cover.
Similarly, we show that in state $s_*$, taking action $a_1$ is at least $\epsilon$ better than taking action $a_\dagger$. Note this proves the claim that $\optEpsDet (R) \subseteq \allowedpi$: if a policy takes an inadmissible action in one of the states, changing that action to $a_j$ or $a_1$ will cause an increase of at least $\epsilon$ in its score. This means that the inadmissible policy could not have been in $\optEpsDet(R)$.
Note that since there may be many copies of some states in the MDP, when referring to a state
$s_i$ or a state $s_*$, we are technically referring to one of these copies. For convenience however,
we will not consider this dependence in our notation.

Formally, let $\hat{\pi}$ be a policy that chooses that action $a_\dagger$ in some $s_i$ and denote by ${\pi}$ a policy such that $\pi(s)=\hat{\pi}(s)$ for all $s\ne s_i$ and $\pi(s)=a_j$ for $s=s_i$ where $j$ is chosen such that $e_i\in S_j$ and $j\in Q$. It follows that
\begin{align*}
    \frac{\score^{\pi, R} - \score^{\hat{\pi}, R}}{1-\gamma}
    &=
    \frac{1}{m}\cdot\gamma\left(
        V^{\pi, R}(s_i) - V^{\hat{\pi}, R}(s_i)
    \right)
    \\&=
    \frac{1}{m}\cdot\gamma\left(
        R(s_i, a_j) + \gamma V^{\pi, R}(t_j) - R(s_i, a_\dagger) 
         -
        \gamma R(\tilde{s}_0)
    \right)
    \\&=
    \frac{1}{m}\cdot\gamma\left(
        x_{ij} + \gamma \cdot 1 -
        0 - \gamma \cdot 0
    \right)
    \\&=
    \frac{1}{m}\cdot\gamma\left(
        \frac{m}{\gamma} \left(\frac{\epsilon}{1-\gamma} + \delta \right) - \gamma + \gamma - 0
    \right)
    \\&=
    \frac{\epsilon}{1-\gamma} + \delta
    \\
    &\ge \frac{\epsilon}{1-\gamma}
\end{align*}
Similarly, if $\hat{\pi}(s_*)=a_\dagger$, by defining $\pi$ such that
$\pi(s)=\hat{\pi}(s)$ for $s\ne s_*$ and
$\pi(s) = a_1$ for $s=s_*$, 
\begin{align*}
     \frac{\score^{\pi, R} - \score^{\hat{\pi}, R}}{1-\gamma}
    &=
    \frac{1}{m}\cdot\gamma\left(
        V^{\pi, R}(s_*) - V^{\hat{\pi}, R}(s_*)
    \right)
    \\&=
    \frac{1}{m}\cdot\gamma\left(
        R(s_*, a_1) + \gamma R(\tilde{s}_1) - R(a_{\dagger}) -  \frac{\gamma}{l}\sum_{j}V^{\pi, R}(t_j)
    \right)
    \\&=
    \frac{1}{m}\cdot\gamma\left(
        \gamma\cdot\frac{k}{l} + \frac{m}{\gamma} \left(\frac{\epsilon}{1-\gamma} + \delta \right) - \frac{\gamma k}{l}
    \right)
    \\&=
    \frac{1}{m}\cdot\gamma\big( \frac{m}{\gamma}(\frac{\epsilon}{1-\gamma} + \delta)
    \big)
    \\&=
    \frac{\epsilon}{1-\gamma} + \delta \ge \frac{\epsilon}{1-\gamma}
\end{align*}
It is therefore clear that in both cases, $\hat{\pi}$ cannot be in $\optEpsDet(R)$ as $\score^{\pi, R} - \score^{\hat{\pi}, R}\ge \epsilon$.

\medskip
{\em Conversely, suppose that the X3C instance is a no-instance.}
Consider an arbitrary feasible solution $R$, i.e., $\optEpsDet (R) \subseteq \allowedpi$.
Suppose that the parameters $\omega_j$, $x_{ij}$, $y$, $z_j$, $u$ and $r$ are modified in $R$ to $\tilde{\omega}_j$, $\tilde{x}_{ij}$, $\tilde{y}$, $\tilde{z}_j$, $\tilde{u}$ and $\tilde{r}$ respectively. Note that while technically these values may
be modified differently for different copies of the copied states, our results focus on one
copy and obtain bounds on these parameters. Since our choice of copy is arbitrary, the bound holds for any copy.
~\\
We will consider two cases and we will show that in both cases it holds that
\begin{align*}
  (\tilde{r} - r)^2 +
  (\tilde{y} - y)^2 +
  (\tilde{z}_* - z_*)^2 +
  (\tilde{u} - u)^2 +
  \sum_{ij}(\tilde{x}_{ij} - x_{ij})^2 +
  \sum_{i}(\tilde{z}_{i} - z_{i})^2
  \ge \frac{1}{3} \cdot (\frac{\gamma}{8l})^2.
\end{align*}
Since there are $N$ copies of each of these values, it follows that
\begin{align*}
  \norm{\overline{R} - R}_2 
  &\ge \sqrt{\frac{N}{3} \cdot \left(\frac{\gamma}{8l}\right)^2}  
  \\&\ge
  \sqrt{k \cdot \varphi^{1-p}\cdot\left(\frac{9l}{\gamma}\right)^2
  \cdot \left(\frac{\gamma}{8l} \right)^2
  }
  \\&> \sqrt{k \cdot \varphi^{1-p}}
  \\&
  \overset{\eqref{eq:bound-instance-size}}{\ge}
  \sqrt{k \cdot (|S| \cdot |A|)^{1-p}}
  \\&=
  \xi \cdot \sqrt{(|S| \cdot |A|)^{1-p}}
\end{align*}
which completes the proof.
\paragraph{Case 1.} $\sum_{j = 1}^l \tilde{\omega}_j \ge k + 1/2$.
Let $\hat{\pi}$ be an optimal policy under $R$. 
Hence, $\hat{\pi} \in \optEpsDet (R) \subseteq \allowedpi$, which means that $\hat{\pi}(s_*) \neq a_\dagger$ as $a_\dagger$ is not admissible.
Now, consider an alternative policy $\pi$, such that $\pi(s) = \hat{\pi}(s)$ for all $s\neq s_*$, and $\pi(s_*) = a_\dagger$.
Since $\pi$ is not admissible, we have $\pi \notin \optEpsDet (R)$, which means $\rho^{\pi, R} \le \rho^{\hat{\pi}, R} - \epsilon$. 
Note that 
\begin{align*}
\frac{\rho^{\pi, R} - \rho^{\hat{\pi}, R}
}{1-\gamma}
&=  \frac{1}{m} \cdot \gamma \left( V^{\pi,R}(s_*) - V^{\hat{\pi},R}(s_*) \right) \\
&= \frac{\gamma}{m} \left( R\left(s_*, a_\dagger \right) + \frac{\gamma}{l} \sum_{j = 1}^l V^{\pi, R}(t_j) - R \left(s_*, a_1 \right) - 
\gamma R(\tilde{s}_1)
\right) \\
&= \frac{\gamma}{m} \left( \tilde{z}_* + \frac{\gamma}{l} \sum_{j = 1}^l  \tilde{\omega}_j -  \tilde{y} - \gamma \tilde{r} \right) \\
&= \frac{\gamma}{m} \left( (\tilde{z}_* - z_*) +
\frac{\gamma}{l} \sum_{j = 1}^l  \tilde{\omega}_j -  (\tilde{y} - y) - \gamma (\tilde{r} - r) 
+ \left(z_{*} - y - \gamma \cdot r\right)
\right) \\
&= \frac{\gamma}{m} \left( (\tilde{z}_* - z_*) +
\frac{\gamma}{l} \sum_{j = 1}^l  \tilde{\omega}_j -  (\tilde{y} - y) - \gamma (\tilde{r} - r) 
+ \left(0 - \gamma \cdot \frac{k}{l} - \frac{m}{\gamma} \cdot (\frac{\epsilon}{1-\gamma} + \delta) + \gamma \cdot 0\right)
\right) \\
&= \frac{\gamma}{m} \left( \frac{\gamma}{l} \left(\sum_{j = 1}^l  \tilde{\omega}_j -  k\right) + \left(\tilde{z}_* - z_* \right) - \left(\tilde{y} - y \right) -
\gamma\left(\tilde{r} -r\right)
\right) -
(\frac{\epsilon}{1-\gamma} + \delta).
\end{align*}
Now that $\rho^{\pi, R} \le \rho^{\hat{\pi}, R} - \epsilon$, plugging this in the above equation and rearranging the terms leads us to the following result:
\begin{align*}
\left( \sum_{j = 1}^l  \tilde{\omega}_j -  k \right) + \frac{l}{\gamma} \left( \tilde{z}_* - z_* \right) -  \frac{l}{\gamma} \left(\tilde{y} - y \right) 
- l\cdot\left(\tilde{r} - r \right) 
\le \delta \cdot \frac{m \cdot l}{\gamma^2} \le 1/4.
\end{align*}
Given the assumption that $\sum_{j = 1}^l \tilde{\omega}_j \ge k + 1/2$ with this case, we then have
\begin{align*}
\gamma \left(\tilde{r} - r \right) +
    \left(\tilde{y} - y \right) - \left(\tilde{z}_* - z_*\right) \ge \frac{\gamma}{4l}.
\end{align*}
Now note that for any three real numbers $a, b, c$, it holds by Cauchy–Schwarz that 
\begin{align*}
    a^2 + b^2 + c^2 \ge \frac{(\gamma a+ b - c)^2}{1 + 1 + \gamma} \ge 
\frac{(\gamma a+ b - c)^2}{3}.
\end{align*}
Applying this result, we have
\begin{align}
\left(\tilde{r} - r \right)^2 
+
\left(\tilde{y} - y \right)^2 + \left(\tilde{z}_* - z_*\right)^2 
&\ge
\frac{1}{3} \cdot \left(\frac{\gamma}{4l}\right)^2
\label{eq:reduction-y-z-lb}
\\\notag&\ge
\frac{1}{3} \cdot \left(\frac{\gamma}{8l}\right)^2.
\end{align}

\paragraph{Case 2.} $\sum_{j = 1}^l \tilde{\omega}_j < k + 1/2$.
Let $Q = \left\{j: \tilde{\omega}_j \ge \frac{k+ 1/2}{k+1} \right\}$.
Hence, the size of $Q$ is at most $k$: otherwise, there are at least $k+1$ numbers in $\tilde{\omega}_1, \dots, \tilde{\omega}_l$ bounded by $\frac{k+ 1/2}{k+1}$ from below, which would imply that $\sum_{j = 1}^l \tilde{\omega}_j \ge k + 1/2$.

By assumption, the X3C instance is a no-instance, so by definition, $Q$ cannot be an exact cover, which means that there exists an element $e_\ell \in E$ not in any subset $S_j$, $j \in Q$.
Accordingly, in the MDP constructed, besides $\tilde{s}_0$, state $s_\ell$ only connects subsequently to states $t_\eta$, with $\eta \notin Q$ (and hence, $\tilde{\omega}_\eta < \frac{k+ 1/2}{k+1}$).

Similarly to the analysis of Case 1, 
let $\hat{\pi}$ be an optimal policy under $R$, so $\hat{\pi}(s_\ell) \neq a_\dagger$ as $a_\dagger$ is not admissible.
Hence, $\hat{\pi} (s_\ell) = a_\eta$ for some $\eta \in \{1,\dots,l\}$ such that $e_\ell \in S_\eta$; in this case, $\eta \notin Q$.

Consider an alternative policy $\pi$, such that $\pi(s) = \hat{\pi}(s)$ for all $s\neq s_\ell$, and $\pi(s_\ell) = a_\dagger$, so $\pi$ is not admissible.
By assumption $R$ is a feasible solution, which means $\pi \notin \optEpsDet (R)$ and hence, $\rho^{\pi, R} \le \rho^{\hat{\pi}, R} - \epsilon$. 
We have
\begin{align*}
- \frac{\epsilon}{1-\gamma} 
&\ge \frac{\rho^{\pi, R} - \rho^{\hat{\pi}, R}}{1-\gamma}
\\&=  \frac{1}{m} \cdot \gamma \left( V^{\pi,R}(s_\ell) - V^{\hat{\pi},R}(s_\ell) \right) \\
&=  \frac{\gamma}{m} \Big( R\left(s_\ell, a_\dagger \right)
+ 
\gamma R\left(\tilde{s}_0 \right)
- R \left(s_\ell, a_\eta \right) - \gamma \cdot \tilde{\omega}_\eta \Big) \\
&= \frac{\gamma}{m} \left( \tilde{z}_\ell + \gamma \tilde{u} - \tilde{x}_{\ell,\eta} - \gamma \cdot \tilde{\omega}_\eta \right) \\
&= \frac{\gamma}{m} \left( \tilde{z}_\ell + \gamma \tilde{u} - \tilde{x}_{\ell,\eta} - \gamma \cdot \tilde{\omega}_\eta \right) 
+\frac{\gamma}{m} \cdot (z_{\ell} + \gamma \cdot  u - x_{\ell, \eta})
\\
&= \frac{\gamma}{m} \left( \left( \tilde{z}_\ell - z_\ell \right)
+ 
\gamma\left( \tilde{u} - u \right)
- \left(\tilde{x}_{\ell,\eta} - x_{\ell, \eta} \right) \right) + \frac{\gamma^2}{m} (1 - \tilde{\omega}_\eta) - \left(\frac{\epsilon}{1-\gamma} + \delta \right),
\end{align*}
which means
\begin{align*}
\left(\tilde{x}_{\ell,\eta} - x_{\ell, \eta} \right) - \left( \tilde{z}_\ell - z_\ell \right) 
- 
\gamma \left( \tilde{u} - u \right) 
&\ge \gamma \left( 1 - \tilde{\omega}_\eta \right) - \frac{m}{\gamma} \cdot \delta \\
&> \frac{1/2}{k+1} \cdot \gamma - \frac{m}{\gamma} \cdot \delta \\
&= \gamma \left( \frac{1}{2(k+1)} - \frac{1}{8l} \right) \\
&\ge \gamma \cdot \frac{1}{8l},
\end{align*}
where we use the inequality $k+1 \le 2l$, which is implied by the assumption that $k \le |\mathcal{S}|$ as we mentioned previously.
In the same way we derived \eqref{eq:reduction-y-z-lb}, we can obtain the following lower bound:
\begin{align*}
 \left(\tilde{u} - u \right)^2
 + 
    \left(\tilde{x}_{\ell,\eta} - x_{\ell, \eta} \right)^2 + \left( \tilde{z}_\ell - z_\ell \right)^2 \ge \frac{1}{3} \cdot \left(\frac{\gamma}{8l}\right)^2,
\end{align*}
which completes the proof.
\end{proof}

We now show that the hardness result in Theorem \ref{thm.copmutational_hardness}, holds for the optimization problem
\eqref{prob.reward_design.approx} as well.
\begin{theorem}
\label{thm.copmutational_hardness-P2}
For any constant $p \in (0,1)$, it is NP-hard to distinguish between instances of \eqref{prob.reward_design.approx} that have optimal values at most $\xi$ and instances that have optimal values larger than $\xi \cdot \sqrt{(|S| \cdot |A|)^{1-p}}$.
The result holds even when the parameters $\epsilon$ and $\gamma$ in \eqref{prob.reward_design.approx} are fixed to arbitrary values subject to $\epsilon > 0$ and $\gamma \in (0,1)$.
\end{theorem}

\begin{proof}
Throughout the proof, we assume that $\lambda = 0$.
To prove the theorem, we use the same reduction 
as the one used in the proof of Theorem~\ref{thm.copmutational_hardness} with slight modification.
Fomrally, given an instance of the X3C problem $(E, \mathcal{S})$
and parameter $\epsilon$,
let $\overline{R}_{\epsilon}$ be the reward function $\overline{R}$ as 
defined in the proof of Theorem \ref{thm.copmutational_hardness-P2}
with the same underlying MDP. We have made the dependence on $\epsilon$
explicit here for reasons that will be clear shortly.
As before, the underlying MDP is shown in Figure \ref{fig:reduction-P1}.
Similarly, for instances of the X3C problem where an exact cover $Q$ exists,
let $R_{\epsilon}$ be the reward function defined in Part 1 of the same proof.
Note that $R_{\epsilon}$ was constructed using $Q$ and therefore implicitly depends on it.
Define $\xi = \sqrt{k}$ as before. 

Given these definitions, recall that our proof showed 
that $(\overline{R}_{\epsilon}, R_{\epsilon}, \xi)$
satisfied the following two properties.
\begin{itemize}
  \item For instances of the problem where exact cover is possible,
  $\norm{R_{\epsilon} - \overline{R}_{\epsilon}}_2\le \xi$ and
  $R_{\epsilon}$ satisfied the constraints of \eqref{prob.reward_design}
  with parameter $\epsilon$, i.e,
  \begin{align*}
    \optEpsDet(R_{\epsilon}) \subseteq \allowedpi,
  \end{align*}
  which is equivalent to
  \begin{align*}
    \forall \pi \notin \allowedpi: \score^{\pi, R_{\epsilon}} \le \max_{\pi'} \score^{ \pi', R_{\epsilon}} - \epsilon.
  \end{align*}
  \item For instances of the X3C problem where an exact cover does not exist,
  for any reward function $\widetilde{R}$ satisfying
  \begin{align*}
    \forall \pi \notin \allowedpi: \score^{\pi, \widetilde{R}} \le \max_{\pi'} \score^{ \pi', \widetilde{R}} - \epsilon,
  \end{align*}
  it follows that $\norm{\widetilde{R} - \overline{R}_{\epsilon}}_2 > \sqrt{(|S| \cdot |A|)^{1-p}} \cdot \xi$.
\end{itemize}
Now, for $\eta > 0$, define the reward functions 
$\overline{R'}_{\epsilon, \eta} := \eta \cdot \overline{R}_{\frac{\epsilon}{\eta}}$
and
${R'}_{\epsilon, \eta} := \eta \cdot {R}_{\frac{\epsilon}{\eta}}$
and set $\xi'_{\eta} := \eta \cdot \xi$.
Since the score function $\score^{\pi, R}$ is linear in $R$, it follows that
$(\overline{R'}_{\epsilon, \eta}, R'_{\epsilon, \eta}, \xi'_{\eta})$
satisfy the same two properties listed above.

In order to prove the theorem statement, for a given instance
of X3C and a given parameter $\epsilon$, we will need to provide an MDP with reward function
$\overline{R'}$ and a parameter $\xi'$ such that:
\begin{itemize}
  \item If exact cover of the X3C instance is possible,
  the cost of the optimization problem \eqref{prob.reward_design.approx} 
  with parameter $\epsilon$ is less than or equal to $\xi'$.
  \item If exact cover of the X3C instance is not possible,
  the cost of the optimization problem \eqref{prob.reward_design.approx} 
  with parameter $\epsilon$ is more than $\sqrt{(|S| \cdot |A|)^{1-p}} \cdot \xi'$.
\end{itemize}
Set $\eta=\frac{\epsilon \cdot m}{\gamma^2 \cdot (1-\gamma)}$.
We claim that $\overline{R'} = \overline{R'}_{\epsilon, \eta}$
and $\xi'=\xi'_{\eta}$ satisfy the above properties.
The second property is easy to check.
Formally, if the cost of \eqref{prob.reward_design.approx} is less than
equal to
$\sqrt{(|S| \cdot |A|)^{1-p}} \cdot \xi'$, so is the cost of \eqref{prob.reward_design}. By the second
property of $(\overline{R'}_{\epsilon, \eta}, R'_{\epsilon, \eta}, \xi'_{\eta})$
discussed before, it follows that the X3C instance has an exact cover.
~\\
For the first property, assume that the X3C instance has an exact cover
$Q$ and consider
$R'=R'_{\epsilon, \eta}$. 
By the first property of
$(\overline{R'}_{\epsilon, \eta}, R'_{\epsilon, \eta}, \xi'_{\eta})$,
it follows that $\norm{R' - \overline{R'}}_2\le \xi'$.
Furthermore,
\begin{align}
  \forall \pi \notin \allowedpi: \score^{\pi, R'} \le \max_{\pi'} \score^{ \pi', R'} - \epsilon.
  \label{eq:dec28_1504}
\end{align}
Now, consider the policy $\hat{\pi}$ as follows:
for each $i = 1,\dots, 3k$, $\hat{\pi}(s_i) = a_j$ such that
$e_i \in S_j$ and $j \in Q$ (such $j$ exist
and is unique given that $Q$ is an exact set cover);
and $\hat{\pi}(s_*) = a_1$.
We claim that $(\overline{R'}, \hat{\pi})$ are feasible for \eqref{prob.reward_design.approx}
which would prove the claim.
~\\
Formally, let $\pi\ne \hat{\pi}$ be a deterministic policy.
We need to show that $\score^{\pi, R'} \le \score^{\hat{\pi}, R'} - \epsilon$.
We assume without loss of generality that
$\pi\in \allowedpi$. If we prove this, then since the optimal policy 
in $R'$ is admissible by \eqref{eq:dec28_1504}, it follows that
$\hat{\pi}$ is the optimal policy and therefore the case of
$\pi \notin \allowedpi$ is covered by \eqref{eq:dec28_1504}.
~\\
With this assumption in mind, note that since $\pi$ and $\hat{\pi}$ are admissible,
there for each $i\in \{1, \dots, 3k\}$, there exist
$j(i)$ and $\hat{j}(i)$ such that $\pi(s_i)=t_{j(i)}$ and $\hat{\pi}(s_i) = t_{\hat{j}(i)}$.
Furthermore, $\pi(s_*) = \hat{\pi}(s_*)$ since both policies are admissible.
It therefore follows that
\begin{align*}
  \frac{\score^{\hat{\pi}, R'} - \score^{\pi, R'}}{1-\gamma} &=
  \frac{\gamma}{m} \cdot \left(
    \sum_{i=1}^{3k} \left(
      V^{\hat{\pi}, R'}(s_i)
      -
      V^{{\pi}, R'}(s_i)
    \right)
  \right)
  \\&=
  \frac{\gamma}{m} \cdot \left(
    \sum_{i=1}^{3k} \left(
      x_{ij} + \gamma \cdot V^{\hat{\pi}, R'}(t_{\hat{j}(i)}) -
      x_{ij} - \gamma \cdot V^{{\pi}, R'}(t_{{j}(i)}) 
    \right)
  \right)
  \\&=
  \frac{\gamma^2}{m} \cdot \left(
    \sum_{i=1}^{3k} \left(
      V^{\hat{\pi}, R'}(t_{\hat{j}(i)}) -
      V^{{\pi}, R'}(t_{{j}(i)}) 
    \right)
  \right)
\end{align*}
Note however that by construction of ${R'}$,
for any $1 \le j \le l$,
\begin{align*}
  V^{\hat{\pi}, R'}(t_{j}) &= {R}'_{\epsilon, \eta}(t_j)
  =
  \eta \cdot R_{\frac{\epsilon}{\eta}}(t_j)
  = 
  \eta \cdot \ind{S_j\in Q}.
\end{align*}
It therefore follows that $V^{\hat{\pi}, R'}(t_{\hat{j}(i)}) = \eta$
since $Q$ is a cover.
Furthermore, 
since $Q$ is an \emph{exact} cover,
$e_i$ is only covered by $S_{\hat{j}(i)}$ which means
$V^{\pi, R'}(t_{j(i)}) = \eta \cdot \ind{j(i) = \hat{j}(i)}$.
Threfore,
\begin{align*}
  \frac{\score^{\hat{\pi}, R'} - \score^{\pi, R'}}{1-\gamma} &=
  \frac{\gamma^2}{m} \cdot \left(
    \sum_{i=1}^{3k} \left(
      V^{\hat{\pi}, R'}(t_{\hat{j}(i)}) -
      V^{{\pi}, R'}(t_{{j}(i)}) 
    \right)
  \right)
  \\&=
  \frac{\gamma^2}{m} \cdot \eta \cdot \left(
  \sum_{i=1}^{3k}\ind{\pi(s_i) \ne \hat{\pi}(s_i)}
  \right)
\end{align*}
Note however that $\pi(s_i)\ne \hat{\pi}(s_i)$ for some
$1\le i \le 3k$. This is because $\pi \ne \hat{\pi}$
and since both are admissible policies, they can only disagree
on some state $s_i$.
Therefore, $  \sum_{i=1}^{3k}\ind{\pi(s_i) \ne \hat{\pi}(s_i)} \ge 1$ which implies
\begin{align*}
  \score^{\hat{\pi}, R'} - \score^{\pi, R'} \ge
  (1-\gamma) \cdot \frac{\gamma^2}{m} \cdot \eta = \epsilon,
\end{align*}
which completes the proof.
\end{proof}


\subsection{Proof of Proposition \ref{prop.reward_design.approx}}\label{sec.proof_of_prop_reward_design_approx}

{\bf Statement:} {\em 
Let $\RsolA$ and $\RsolB$ be the optimal solutions to
	\eqref{prob.reward_design} and \eqref{prob.reward_design.approx} respectively and let $l(R)$ be a function that outputs the objective of the optimization problem
	\eqref{prob.reward_design}, i.e.,
	\begin{align*}
	l(R)= \max_{\pi\in \optEpsDet(R)}\norm{\overline{R} - R}_{2} - \lambda\score^{\pi, \overline{R}}.
	\end{align*}
	Then $\RsolB$ satisfies the constraints of \eqref{prob.reward_design}, i.e, $\optEpsDet(\RsolB) \subseteq \allowedpi$, and
	\begin{align*}
	l(\RsolA) \le l(\RsolB) \le l(\RsolA) + \frac{\epsilon}{\occstate_{\min}}\cdot\sqrt{|S|\cdot|A|}.
	\end{align*}
}
\begin{proof}

 We first prove that $\widehat{R}_2$ is feasible for \eqref{prob.reward_design}. This would prove the left inequlity given the optimality of $\widehat{R}_1$.
 In order to prove this, we just need to show that for all $\pi \in \allowedpi$,
 \begin{align*}
     \{
     \pi': \pi'(s) = \pi(s) \text{ if } \occstate^{\pi}(s) > 0
     \} \subseteq
     \allowedpi
 \end{align*}
 Given Lemma \ref{lm.same_occupancy}, if 
 $\pi'(s) = \pi(s)$ for all $s$ such that
 $\occstate^\pi(s) > 0$, then 
 $\occstate^{\pi}=\occstate^{\pi'}$.
 Therefore, for all states $s$, either $\occstate^\pi(s) > 0$, in which case
 $\pi'(s) = \pi(s) \in A^{\adm}_s$, or
 $\occstate^\pi(s) = 0$, in which case $\occstate^{\pi'}(s) = 0$. Therefore, $\pi' \in \allowedpi$ and the claim is proved.
	
	 As for the right inequality,
	let $\pi_1$ be an optimal policy under
	$\RsolA$. Given the constraints of
	\eqref{prob.reward_design}, $\pi_1\in \allowedpi$.
	Define $R'$ as
	\begin{align*}
		R'(s,a) = \begin{cases}
			\RsolA(s,a) - \frac{\epsilon}{\occstate_{\min}}
			&\CasesIf a \ne \pi_1(s) \land 
		\occstate^{\pi_1}(s) > 0
			\\
			\RsolA(s,a) &\CasesOW
		\end{cases}.
	\end{align*}
  We first show that $R', \pi_1$ are feasible for \eqref{prob.reward_design.approx}. 
	Let $\pi$ be a policy such that
	$\pi(\tilde{s}) \ne \pi_1(\tilde{s})$ for some $\tilde{s}$ that satisfies $\occstate^{\pi_1}(\tilde{s}) > 0$.
  We claim that this means there exists a state $s\in \posStates^{\pi_1} \cap \posStates^{\pi}$
  such that $\pi(s) \ne \pi_1(s)$.
  If this is not the case, then 
  Lemma \ref{lm.same_occupancy} (from section \nameref{app.background}) implies that
  $\occstate^{\pi_1} = \occstate^{\pi}$.
  This further implies that $\posStates^\pi = \posStates^{\pi_1}$ and therefore
  since $\tilde{s} \in \posStates^{\pi_1}$ and $\pi(\tilde{s}) \ne \pi_1(\tilde{s})$,
  we have reached a contradiction.
  Therefore, 
  there exists a state state $s\in \posStates^{\pi_1} \cap \posStates^{\pi}$
  such that $\pi(s) \ne \pi_1(s)$.
  Without loss of generality, assume that $\tilde{s}$ is this state.
  ~\\
  It follows that
	\begin{align*}
	    \score^{\pi_1, R'} - \score^{\pi, R'}
	    &=
	    (\score^{\pi_1, \widehat{R}_1} - \score^{\pi, \widehat{R}_1})
	    +
	    (\score^{\pi_1, R'} - \score^{\pi_1, \widehat{R}_1})
	    + 
	    (\score^{\pi, \widehat{R}_1} - \score^{\pi, R'})
	\end{align*}
	The first term is non-negative since $\pi_1$ was assumed to be optimal under $\widehat{R}_1$. The second term equals zero given the definition of $R'$. As for the last term,
	\begin{align*}
	    \score^{\pi, \widehat{R}_1} - \score^{\pi, R'}
	    &=
	    \sum_{s}\occstate^{\pi}(s)(\widehat{R}_1(s, \pi(s)) - R'(s, \pi(s)))
	    \\&=
	    \occstate^{\pi}(\tilde{s})(\widehat{R}_1(\tilde{s}, \pi(\tilde{s})) - R'(\tilde{s}, \pi(\tilde{s}))
	    +
	    \sum_{s\ne \tilde{s}}
	    \occstate^{\pi}(s)(\widehat{R}_1(s, \pi(s)) - R'(s, \pi(s))) 
	    \\&=
	    \occstate^{\pi}(\tilde{s})\cdot\frac{\epsilon}{\occstate_{\min}} +
	    \sum_{s\ne \tilde{s}}
	    \occstate^{\pi}(s)(\widehat{R}_1(s, \pi(s)) - R'(s, \pi(s))) 
	    \\&\overset{(i)}{\ge}
	    \occstate^{\pi}(\tilde{s})\cdot\frac{\epsilon}{\occstate_{\min}}
	    \\&\overset{(ii)}{\ge}
        \epsilon
	\end{align*}
	where $(i)$ follows from the fact that
	$\widehat{R}_1-R'$ is non-negative and $(ii)$ follows from the definition of $\occstate_{\min}$ and the fact $\tilde{s}\in \posStates^{\pi}$.
  Therefore, $R', \pi_1$ are feasible for \eqref{prob.reward_design.approx}. 

	Now, note that
	\begin{align*}
		\norm{\overline{R} - R'}_2 &\le 
		\norm{\overline{R} - \RsolA}_2  
		+
		\norm{\RsolA-R'}_2 
		\\&\le 
		\norm{\overline{R} - \RsolA}_2  
		+
		\sqrt{\sum_{s,a\ne\pi'(s)}
			\left (\frac{\epsilon}{\occstate_{\min}} \right)^2	
		}
		\\&\le 
		\norm{\overline{R} - \RsolA}_2  
		+
		\frac{\epsilon}{\occstate_{\min}}\sqrt{|S| \cdot |A|}.
	\end{align*}
	This means that
	\begin{align*}
		\norm{\overline{R} - R'}_2 -\lambda \score^{\pi_1,\overline{R}}&\le 
	\norm{\overline{R} - \RsolA}_2 - \lambda\score^{\pi_1,\overline{R}} + \frac{\epsilon}{\occstate_{\min}}\sqrt{|S| \cdot |A|} 
	\\&\le 
	\max_{\pi \in \optEpsDet(\RsolA)}
		\norm{\overline{R} - \RsolA}_2 - \lambda\cdot\score^{\pi,\overline{R}} + \frac{\epsilon}{\occstate_{\min}}\sqrt{|S| \cdot |A|} 
		\\&= 
		l(\RsolA) + \frac{\epsilon}{\occstate_{\min}}\sqrt{|S| \cdot |A|},
	\end{align*}
	where the second inequality is due to the fact that $\pi_1$ is optimal under $\RsolA$, so it belongs to the set $\optEpsDet(\RsolA)$.
	
	Now, let $\pi_2$ be a deterministic optimal policy under 
	$\RsolB$.
  Since $\RsolB$ was the solution to \eqref{prob.reward_design.approx}, for any $\pi\in \optEpsDet(\widehat{R}_2)$,
  it holds that
  $\pi(s) = \pi_2(s)$ for all $s\in \posStates^{\pi_2}$
  and therefore by Lemma \ref{lm.same_occupancy}, 
	$\score^{\pi, \overline{R}}=\score^{\pi_2, \overline{R}}$.
	Denote the solution to the optimization problem \eqref{prob.reward_poisoning_attack} (reward poisoning attack) for the target policy $\targetpi = \pi_1$ by $\widehat{R}_{\mathcal{A}}$. We obtain
	\begin{align*}
		l(\RsolB) &= \max_{\pi\in \optEpsDet(\widehat{R}_2)}\norm{\overline{R} - \RsolB}_2 - \lambda\cdot \score^{\pi, \overline{R}}\\&= 
		\norm{\overline{R} - \RsolB}_2 - \lambda\score^{\pi_2, \overline{R}}\\&
		\overset{(i)}{\le}
		\norm{\overline{R} - \widehat{R}_{\mathcal{A}}}_2 - \lambda\cdot \score^{\pi_1, \overline{R}}
		\\&\overset{(ii)}{\le}
			\norm{\overline{R} - R'}_2 -\lambda \cdot \score^{\pi_1, \overline{R}}
			\\&\le l(\RsolA) + \frac{\epsilon}{\occstate_{\min}}\cdot \sqrt{|S|\cdot|A|}.
	\end{align*}
	where $(i)$ follows from the optimality of 
	$\RsolB$ for
	\eqref{prob.reward_design.approx} and $(ii)$ follows from the definition of the reward poisoning attack \eqref{prob.reward_poisoning_attack}.
  We have therefore shown the right inequality of the Lemma's statement holds and the proof is complete.
\end{proof}


\section{Proofs of the Results from Section \nameref{sec.special_mdps}}\label{app.proof.special_mdps}

In this section, we provide proofs of our results in Section \nameref{sec.special_mdps}: Lemma \ref{lm.special_mdp.cost_of_poisoning} and Theorem
\ref{thm.spec_mdp_attack_form}.
Recall that for special MDPs since the transition probabilities are independent of the agent's policy, so is the state occupancy measure.
Concretely, 
since the Bellman flow constraint
\eqref{eq.bellman.occstate} characterizing
$\occstate^\pi$ is independent of policy, 
so is $\occstate^\pi$.
We therefore
use $\occstate$ instead of $\occstate^\pi$ to denote the state occupancy measure for special MDPs. Similarly, since $\posStates^\pi$ depends on $\pi$ only through $\occstate^\pi$, we use $\posStates$ instead of $\posStates^\pi$.

Before we prove the main results, we present and prove the following two lemmas.


\begin{lemma}\label{lm.surplus_Equation}
Consider a special MDP with a reward function $\overline{R}$ and a policy of interest $\targetpi$.
For all $s \in S$ s.t. $\occstate(s) > 0$, there exists a unique $x$ such that:
\begin{align}\label{eq.lm.surplus_Equation}
    \sum_{a \ne \targetpi(s)} \pos{\overline{R}(s, a) - x} = x - \overline{R}(s, \targetpi(s)) + \frac{\epsilon}{\occstate(s)}.
\end{align}
\end{lemma}
\begin{proof}
Fix the state $s$. 
    Consider the following function
    \begin{align*}
    f: \mathbb{R} \to \mathbb{R}, \quad
        f(x) := 
         \sum_{a \ne \targetpi(s)} \pos{\overline{R}(s, a) - x} -x.
    \end{align*}
    This function is strictly decreasing because
    $x\to -x$ is strictly decreasing and
    $x\to \pos{\overline{R}(s, a) - x}$ is decreasing. Furthermore, $\lim_{x\to \infty} f(x)=-\infty$ and $\lim_{x\to-\infty}f(x) = 
    \infty
    $.
    Therefore given the intermediate value theorem, there exists a unique number $x$ such that
    \begin{align*}
        f(x)=-\overline{R}(s, \targetpi(s))+
        \frac{\epsilon}{\occstate(s)}.
    \end{align*}
\end{proof}

\begin{lemma}\label{lm.special_mdp_attack_character}
Consider a special MDP with reward function $\overline{R}$. Let $\targetpi \in \PiDet$ be an arbitrary deterministic policy.
 Define $\widehat{R}^{\targetpi}$ as
\begin{align*}
    \widehat{R}^{\targetpi}(s, a) = \begin{cases}
     x_s + \frac{\epsilon}{\occstate(s)} \quad &\mbox{ if } \occstate(s) > 0 \land a = \targetpi(s)\\
     x_s \quad &\mbox{ if } \occstate(s) > 0 \land a \ne \targetpi(s) \land \overline{R}(s, a) \ge x_s \\
     \overline{R}(a, s) \quad &\mbox{ otherwise}
    \end{cases},
\end{align*}
where $x_s$ is the solution to the Eq. \eqref{eq.lm.surplus_Equation}.
$\widehat{R}^{\targetpi}$ is an optimal solution to the optimization problem \eqref{prob.reward_poisoning_attack}.
\end{lemma}
\begin{proof}
In order to show feasibility, let $\tilde{\pi}\in \PiDet$ be a deterministic policy such that 
$\tilde{\pi}(\tilde{s})\ne \targetpi(s)$ for some $\tilde{s}$ such that
$\occstate({\tilde{s}}) > 0$. It is clear that
\begin{align*}
    \score^{\targetpi, \widehat{R}^{\targetpi}}-\score^{\tilde{\pi}, \widehat{R}^{\targetpi}}
    &=
    \sum_{s}
    \occstate(s)
    \big(
      \widehat{R}^{\targetpi}
    (s, \targetpi(s)) - \widehat{R}^{\targetpi}(s, \tilde{\pi}(s))
    \big)
    \\&=
    \occstate(\tilde{s})\big(
      \widehat{R}^{\targetpi}(\tilde{s}, \targetpi(\tilde{s})) - \widehat{R}^{\targetpi}(\tilde{s}, \tilde{\pi}(\tilde{s}))
    \big) + 
    \sum_{s\ne \tilde{s}}
    \occstate(s)\big(
      \widehat{R}^{\targetpi}(s, \targetpi(s)) - \widehat{R}^{\targetpi}(s, \tilde{\pi}(s))
    \big)
    \\&\overset{(i)}{\ge}
    \occstate(\tilde{s})\big(
      \widehat{R}^{\targetpi}(\tilde{s}, \targetpi(\tilde{s})) - \widehat{R}^{\targetpi}(\tilde{s}, \tilde{\pi}(\tilde{s}))
    \big)
    \\&\overset{(ii)}{\ge}
    \occstate(\tilde{s})\frac{\epsilon}{\occstate(\tilde{s})}
    \ge \epsilon
\end{align*}
where $(i)$ and $(ii)$ both follow from the definition of $\widehat{R}^{\targetpi}$;
(i) follows from the fact that $\widehat{R}^{\targetpi}(s, \targetpi(s))\ge \widehat{R}^{\targetpi}(s, a)$
for all $s, a$ such that $\occstate(s)> 0$ 
and $(ii)$ follows from the fact that $\tilde{\pi}(\tilde{s})\ne \targetpi(\tilde{s})$.

We now show that if $R$ is also feasible for the optimization problem \eqref{prob.reward_poisoning_attack}, then $\norm{\overline{R} - R}_2 \ge \norm{\overline{R} - \widehat{R}^{\targetpi}}_2$. 
The key point about $\widehat{R}^{\targetpi}$, is the definition of $x_s$ in Equation \eqref{eq.lm.surplus_Equation}. 
Concretely,
it is clear that for $s\in \posStates$:
\begin{align}
    \notag \overline{R}(s, \targetpi(s)) - \widehat{R}^{\targetpi}(s, \targetpi(s))
    &=
    \overline{R}(s, \targetpi(s)) -
    x_{s} - \frac{\epsilon}{\occstate(s)}
    \\\notag&=
    -\big(
        x_s + \frac{\epsilon}{\occstate(s)} - \overline{R}(s, \targetpi(s))
    \big)
    \\\notag&\overset{\eqref{eq.lm.surplus_Equation}}{=}
    -\sum_{a\ne \targetpi(s)}\pos{\overline{R}(s, a) - x_s}
    \\\notag&=
    -\sum_{a\ne \targetpi(s)}\ind{\overline{R}(s, a) > x_s}\cdot
    (\overline{R}(s, a) - x_s)
    \\\notag&=
    -\sum_{a\ne \targetpi(s)}(\overline{R}(s, a) - \widehat{R}^{\targetpi}(s, a))
    \\&=
    \sum_{a\ne \targetpi(s)}(\widehat{R}^{\targetpi}(s, a) - \overline{R}(s, a))
    \label{eq.key_point_special_R_hat}
\end{align}
Now note that
\begin{align*}\norm{\overline{R} - R}_2^2 - 
    \norm{\overline{R} - \widehat{R}^{\targetpi}}_2^2
    &=
    \vecdot{\overline{R} - R - \overline{R} + \widehat{R}^{\targetpi}}{\overline{R} - R + \overline{R} - \widehat{R}^{\targetpi}}
    \\&=
    \vecdot{\widehat{R}^{\targetpi} - R}{\overline{R} - R + \overline{R}  - \widehat{R}^{\targetpi} }
    \\&=
    \vecdot{\widehat{R}^{\targetpi} - R}{\overline{R} - R + \overline{R} + \widehat{R}^{\targetpi}  - 2\widehat{R}^{\targetpi} }
    \\&=
    \norm{\widehat{R}^{\targetpi} - R}_2^2 + 2\vecdot{\widehat{R}^{\targetpi} - R}{\overline{R} - \widehat{R}^{\targetpi}}
    \\&\ge 
    2\vecdot{\widehat{R}^{\targetpi} - R}{\overline{R} - \widehat{R}^{\targetpi}}
    \\&= 2\Big(
        \vecdot{\overline{R} - \widehat{R}^{\targetpi}}{\widehat{R}^{\targetpi}} -
        \vecdot{\overline{R} - \widehat{R}^{\targetpi}}{R}
    \Big)
\end{align*}
It therefore suffices to show that the above quantity is non-negative. Note however,
\begin{align}
    \notag
    \vecdot{\overline{R} - \widehat{R}^{\targetpi}}{R}
    &=
    \sum_{(s, a)}
    (\overline{R}(s, a) - \widehat{R}^{\targetpi}(s, a))\cdot R(s, a)
    \\\notag&=
    \sum_{s\in \posStates, a}
    (\overline{R}(s, a) - \widehat{R}^{\targetpi}(s, a))\cdot R(s, a)
    + 
    \sum_{s\notin \posStates, a}
    (\overline{R}(s, a) - \widehat{R}^{\targetpi}(s, a))\cdot R(s, a)
    \\\notag&=
    \sum_{s\in \posStates, a}
    (\overline{R}(s, a) - \widehat{R}^{\targetpi}(s, a))\cdot R(s, a) + 
    \sum_{s\notin \posStates, a} 0\cdot R(s, a)
    \\\notag&=
    \sum_{s\in \posStates, a}
    (\overline{R}(s, a) - \widehat{R}^{\targetpi}(s, a))\cdot R(s, a)
    \\\notag&=
    \sum_{s \in \posStates} (\overline{R}(s, \targetpi(s)) - \widehat{R}^{\targetpi}(s, \targetpi(s)))\cdot R(s, \targetpi(s)) +
    \sum_{s\in \posStates, a\ne \targetpi(s)}
    (\overline{R}(s, a) - \widehat{R}^{\targetpi}(s, a))\cdot R(s, a)
    \\\notag&\overset{(i)}{=}
    \sum_{s \in \posStates} R(s, \targetpi(s)) \cdot\sum_{a\ne \targetpi(s)} 
    (\widehat{R}^{\targetpi}(s, a) - \overline{R}(s, a)) +
    \sum_{s\in \posStates, a\ne \targetpi(s)}
    (\overline{R}(s, a) - \widehat{R}^{\targetpi}(s, a))\cdot R(s, a)
    \\&=
    \sum_{s\in \posStates, a\ne \targetpi(s)}
    \big(\overline{R}(s, a) - \widehat{R}^{\targetpi}(s, a)\big)\cdot\big(
    R(s, a) - R(s, \targetpi(s))
    \big).
    \label{eq.vecdot_decomposition}
\end{align}
where $(i)$ follows from \eqref{eq.key_point_special_R_hat}.
Now note that since $R$ was assumed to be feasible, if $s\in \posStates$ and $a\ne \targetpi(s)$, then by defining $\tilde{\pi}$ as the policy that chooses $\targetpi(\tilde{s})$
in states $\tilde{s}\ne s$ and chooses $a$ in state $s$, it follows that
\begin{align*}
    \frac{\epsilon}{\occstate(s)} &\le 
    \frac{\score^{\targetpi, R}- \score^{\tilde{\pi}, R}
    }{\occstate(s)}
    \\&=
    \sum_{\tilde{s}} \frac{
    \occstate(\tilde{s})\big(
    R(\tilde{s}, \targetpi(\tilde{s})) - R(\tilde{s}, \tilde{\pi}(\tilde{s}))\big)
    }{\occstate(s)}
    \\&=    
     \frac{
    \occstate(s)\big(
    R(s, \targetpi(s)) - R(s, \tilde{\pi}(s))\big)
    }{\occstate(s)}
    \\&=
    \frac{
    \occstate(s)\big(
    R(s, \targetpi(s)) - R(s,a)\big)
    }{\occstate(s)}
    \\&=
    R(s, \targetpi(s)) - R(s, a).
\end{align*}
Therefore,
\begin{align*}
    \vecdot{\overline{R} - \widehat{R}^{\targetpi}}{R}
    &= 
    \sum_{s\in \posStates, a\ne \targetpi(s)}
    \big(\overline{R}(s, a) - \widehat{R}^{\targetpi}(s, a)\big)\cdot\big(
    R(s, a) - R(s, \targetpi(s))
    \big)
    \\&\le
    -\sum_{s\in \posStates, a\ne \targetpi(s)}
    \big(\overline{R}(s, a) - \widehat{R}^{\targetpi}(s, a)\big)\cdot\big(
    \frac{\epsilon}{\occstate(s)}
    \big)
\end{align*}
Using \eqref{eq.vecdot_decomposition} for $\widehat{R}^{\targetpi}$ (note that since no assumptions on $R$ were made in deriving the identity, it is valid for $R=\widehat{R}^{\targetpi}$),
\begin{align*}
    \vecdot{\overline{R} - \widehat{R}^{\targetpi}}{\widehat{R}^{\targetpi}}
    &=\sum_{s\in \posStates, a\ne \targetpi(s)}
    \big(\overline{R}(s, a) - \widehat{R}^{\targetpi}(s, a)\big)\cdot\big(
    \widehat{R}^{\targetpi}(s, a) - \widehat{R}^{\targetpi}(s, \targetpi(s))
    \big)
    \\&\overset{(i)}{=}
    -\sum_{s\in \posStates, a\ne \targetpi(s)}
    \big(\overline{R}(s, a) - \widehat{R}^{\targetpi}(s, a)\big)\cdot\big(
    \frac{\epsilon}{\occstate(s)}
    \big)
\end{align*}
where $(i)$ follows form the definition of $\widehat{R}^{\targetpi}$. Concretely, for state action pairs $s\in \posStates, a\ne \targetpi(s)$ such that
$\widehat{R}^{\targetpi}(s, a)\ne\overline{R}(s, a)$,
$\overline{R}(s, a)\ge x_s$ and therefore
\begin{align*}
    \widehat{R}^{\targetpi}(s, a) = x_s=\widehat{R}^{\targetpi}(s, \targetpi(s)) - \frac{\epsilon}{\occstate(s)}
\end{align*}
We therefore obtain 
$
\vecdot{\overline{R} - \widehat{R}^{\targetpi}}{\widehat{R}^{\targetpi}}
\ge 
\vecdot{\overline{R} - \widehat{R}^{\targetpi}}{R}
$
and the proof is concluded.
\end{proof}


\subsection{Proof of Lemma \ref{lm.special_mdp.cost_of_poisoning}}
{\bf Statement:} {\em 
Consider a {\em special} MDP with reward function $\overline{R}$, and let $\optallowpi(s) = \argmax_{a \in A_{s}^{\adm}}\overline{R}(s,a)$.
Then the cost of the optimal solution to the optimization problem \eqref{prob.reward_poisoning_attack} with $\targetpi = \optallowpi$ is less than or equal to the cost of the optimal solution to the optimization problem \eqref{prob.reward_poisoning_attack} with $\targetpi = \pi$ for any $\pi \in \allowedpi$.
}
\begin{proof}
~\\
    \textbf{Part 1:} We first prove the claim for single-state MDPs which are equivalent to multi-arm bandits. Since there is a single state, for ease of notation we drop the dependence on the state $s$ when referring to quantities that would normally depend on $s$ such as $\overline{R}(s, a)$ and $\pi(s)$.
    ~\\
    Let $a_1, a_2$ be two actions such that
    $\overline{R}(a_1) \ge \overline{R}(a_2)$. Denote by $\pi_1$ and $\pi_2$ the policies that deterministically choose $a_1$ and $a_2$ respectively and
    and denote by
    $\widehat{R}^{\pi_1}, \widehat{R}^{\pi_2}$ the solutions to the optimization problem \eqref{prob.reward_poisoning_attack} 
    with $\targetpi=\pi_1$ and $\targetpi=\pi_2$ respectively. 
    We will show that
    \begin{align*}
        \norm{\overline{R} - \widehat{R}^{\pi_2}}_2 \ge \norm{\overline{R} - \widehat{R}^{\pi_1}}_2.
    \end{align*}
    ~\\
    Consider the reward function $R'$ defined as
	\begin{align*}
		R'(a) = \begin{cases}
			\widehat{R}^{\pi_2}(a_1) \CasesIf a=a_2\\
			\widehat{R}^{\pi_2}(a_2) \CasesIf a=a_1\\
			\widehat{R}^{\pi_2}(a) \CasesOW
		\end{cases};
	\end{align*}
	In other words, we have swtiched the reward for $a_1$ and $a_2$ in $\widehat{R}^{\pi_2}$. 
	
	We claim that
	\begin{align*}
		||\overline{R} - \widehat{R}^{\pi_2}||_2
    \ge
    ||\overline{R} - R'||_2
    \ge
    ||\overline{R} - \widehat{R}^{\pi_1}||_2.
	\end{align*}
	To prove the second inequality, note that 
	$a_1$ is $\epsilon$-robust optimal in $R'$ since
	$a_2$ was $\epsilon$-robust optimal in $R'$ and 
	$R'$ was obtained from switching $a_1, a_2$ in $\widehat{R}^{\pi_2}$.
	Therefore the inequality follows from the optimality of $\widehat{R}^{\pi_1}$.
	
	As for the first inequality, note that it can be rewritten as
	\begin{align*}
		& 
		||\overline{R} - \widehat{R}^{\pi_2}||_2^2 \ge ||\overline{R} - R'||_2^2
		\\\iff &
		\sum_{i} (\overline{R}(a_i) - \widehat{R}^{\pi_2}(a_i))^2\ge 
		\sum_{i} (\overline{R}(a_i) - R'(a_i))^2
		\\\overset{(i)}{\iff}& 
		(\overline{R}(a_1) - \widehat{R}^{\pi_2}(a_1))^2 + (\overline{R}(a_2) - \widehat{R}^{\pi_2}(a_2))^2
		\ge 
		(\overline{R}(a_1) - \widehat{R}^{\pi_2}(a_2))^2 + (\overline{R}(a_2) - \widehat{R}^{\pi_2}(a_1))^2
		\\\iff & 
		-2\overline{R}(a_1)\widehat{R}^{\pi_2}(a_1) - 2\overline{R}(a_2)\widehat{R}^{\pi_2}(a_2)\ge 
		-2\overline{R}(a_1)\widehat{R}^{\pi_2}(a_2) - 2\overline{R}(a_2)\widehat{R}^{\pi_2}(a_1)
		\\\iff &
		\overline{R}(a_1)\widehat{R}^{\pi_2}(a_2) + \overline{R}(a_2)\widehat{R}^{\pi_2}(a_1)
		- 
		\overline{R}(a_1)\widehat{R}^{\pi_2}(a_1) - \overline{R}(a_2)\widehat{R}^{\pi_2}(a_2)
		\ge 0
		\\\iff &
		(\overline{R}(a_1) - \overline{R}(a_2))(\widehat{R}^{\pi_2}(a_2) - \widehat{R}^{\pi_2}(a_1))\ge 0,
	\end{align*}
	where $(i)$ follows from the definition of $R'(a_i)$.
	Note however that the last equation holds trivially since
	$\overline{R}(a_1) \ge \overline{R}(a_2)$ by assumption and
	$\widehat{R}^{\pi_2}(a_2) > \widehat{R}^{\pi_2}(a_1)$ by definition of $\widehat{R}^{\pi_2}$.
	Note further that if
	$\overline{R}(a_1) > \overline{R}(a_2)$, then
	the inequality is strict.

    The statement of the lemma now follows by setting
    $a_1 = \argmax_{a\in A_{s}^{\adm}}\overline{R}(a)$ and
    $a_2$ to be any $a\in A_{s}^{\adm}$.
    
\textbf{Part 2:}
    We now extend this result to multi-state special MDPs.
	Since the MDP is special, 
	given Lemma \ref{lm.special_mdp_attack_character}
	we can view the attack as separate single-state attacks for $s\in \posStates$ with parameters
	$\frac{\epsilon}{\occstate(s)}$.
	Since $\optallowpi(s)$ equals
	$\argmax_{a\in A_{s}^{\adm}}\overline{R}(s, a)$ by definition, Part 1 implies that the cost of the optimization problem \eqref{prob.reward_poisoning_attack} with $\targetpi=\optallowpi$ is not more than the cost of the optimization problem with $\targetpi=\pi$ for all $\pi\in \allowedpi$.
\end{proof}


\subsection{Proof of Theorem \ref{thm.spec_mdp_attack_form}}

{\bf Statement:} {\em Consider a special MDP with reward function $\overline{R}$. Define $\widehat R(s, a) = \overline R(s, a)$ for $\occstate(s) = 0$ and otherwise
\begin{align*}
    \widehat R(s, a) = \begin{cases}
     x_s + \frac{\epsilon}{\occstate(s)} \quad &\mbox{ if } a = \optallowpi(s)\\
     x_s \quad &\mbox{ if } a \ne \optallowpi(s) \land \overline{R}(s, a) \ge x_s \\
     \overline{R}(s, a) \quad &\mbox{ otherwise}
    \end{cases},
\end{align*}
where $x_s$ is the solution to the equation
\begin{align*}
     \sum_{a \ne \optallowpi(s)} \pos{\overline{R}(s, a) - x} = x - \overline{R}(s, \optallowpi(s)) + \frac{\epsilon}{\occstate(s)}.
\end{align*}
Then, $(\optallowpi, \widehat R)$
is an optimal solution to \eqref{prob.reward_design.approx}.
}
\begin{proof}
The claim follows from Lemmas 
\ref{lm.special_mdp.cost_of_poisoning} and \ref{lm.special_mdp_attack_character}. Concretely, 
given Lemma \ref{lm.special_mdp.cost_of_poisoning}, the solution to the optimization problem \eqref{prob.reward_design.approx} is $\optallowpi$ and
$\widehat{R}^{\targetpi}$ where $\widehat{R}^{\targetpi}$ is the solution to \eqref{prob.reward_poisoning_attack} with $\targetpi=\optallowpi$. The claim now follows from Lemma \ref{lm.special_mdp_attack_character} which characterizes $\widehat{R}^{\targetpi}$.

\end{proof}


\section{Proofs of the Results in Section \nameref{sec.general_mdps}}\label{app.proofs.general_mdps}

In this section, we provide proofs of our results in Section \nameref{sec.general_mdps}, namely Theorem \ref{thm.general_mdp.charact.bounds.opt_const}
and Theorem \ref{thm.general_mdp.charact.bounds.qgreedy}.
Before we present the proofs, we introduce and prove two auxiliary lemmas.

\begin{lemma}\label{lm.constructive_answer}
    For an arbitrary policy $\pi$, define $R', V, Q$ as
    \begin{align*}
        &R'(s, a) =
        \begin{cases}
            \overline{R}(s, a) +\delta^\pi(s) \quad\text{if}\quad 
            \occstate^{\pi}(s)> 0
            \text{ and }
            a=\pi(s)\\
            \overline{R}(s, a) - \epsilon'_{\pi}(s, a)\quad\text{if}\quad
            \occstate^{\pi}(s)> 0 \text{ and } a\ne \pi(s)\\
            \overline{R}(s, a)\quad\text{o.w.}
        \end{cases}\\
        &V(s) = V^{*, \overline{R}}(s)\\
        &Q(s, a) = R'(s, a) + \gamma\sum_{s'}P(s, a, s')V(s'),
    \end{align*}
    where
    \begin{align*}
    \label{eq.def_delta_pi}
        \delta^\pi(s) = 
        \begin{cases}
            Q^{*, \overline{R}}(s, \optpi(s)) - Q^{*, \overline{R}}(s, \pi(s)) \quad&\text{if}\quad \occstate^\pi(s) > 0
            \\
            0 \quad&\text{otherwise}
        \end{cases},
    \end{align*}
    and $\epsilon'_\pi$ is defined as in
    \eqref{eq.epsilon_prime}.
    The vectors $R', V, Q$ are feasible in 
    \eqref{prob.rp_approx} with $\epsilon'$ set to
    $\epsilon'_\pi$ and $\targetpi=\pi$.
    Furthermore, given Proposition \ref{lm.rp_approx},
    $R'$ is feasible for
    \eqref{prob.reward_poisoning_attack} with $\targetpi=\pi$ and therefore $(R', \pi)$ are feasible for
    \eqref{prob.reward_design.approx}.
\end{lemma}\begin{proof}
We check all of the conditions. \eqref{constraint.rqv} holds by definition of $Q$.
Now note that since $V=V^{*, \overline{R}}$ and $R'(s, a)=\overline{R}(s, a)$ for all $s\notin \posStates^{\pi}$, 
we conclude that $Q(s, a)=Q^{*, \overline{R}}(s, a)$ for all $s\notin \posStates^{\pi}$.
Therefore, since
$V^{*, \overline{R}}(s)\ge Q^{*, \overline{R}}(s, a)$ for all $s, a$, the constraint \eqref{constraint.vqzero} holds as well.
The constraint \eqref{constraint.vqone} holds by definition of $\qgapqval^\pi$ as
\begin{align*}
Q(s, \pi(s)) - V(s)
&=
R'(s, \pi(s)) + \gamma
\sum_{s'}P(s, \pi(s), s')V^{*, \overline{R}}(s')
- V^{*, \overline{R}}(s)
\\&=
\overline{R}(s, \pi(s)) + 
\delta^\pi(s) +
\gamma
\sum_{s'}P(s, \pi(s), s')V^{*, \overline{R}}(s')
- V^{*, \overline{R}}(s)
\\&=
\overline{R}(s, \pi(s)) + 
V^{*, \overline{R}}(s) -
Q^{*, \overline{R}}(s, \pi(s))
 +
\gamma
\sum_{s'}P(s, \pi(s), s')V^{*, \overline{R}}(s')
- V^{*, \overline{R}}(s)
\\&=
\left(
\overline{R}(s, \pi(s)) + 
\gamma
\sum_{s'}P(s, \pi(s), s')V^{*, \overline{R}}(s')
 -
Q^{*, \overline{R}}(s, \pi(s))
\right)
 +
 \left(
 V^{*, \overline{R}}(s)
- V^{*, \overline{R}}(s)
  \right)
  \\&=0.
\end{align*}
Furthermore, \eqref{constraint.ge} holds because for all $s\in \posStates^\pi, a\ne \pi(s)$,
\begin{align*}
    Q(s, a) &=
    \overline{R}(s, a) - \epsilon'(s, a) + \gamma\sum_{s'}P(s, a, s')V^{*, \overline{R}}(s')
    \\&=
    Q^{*, \overline{R}}(s, a) - 
    \epsilon'(s, a)
    \\&\le
    V^{*, \overline{R}}(s) - 
    \epsilon'(s, a)
    \\&\overset{\eqref{constraint.vqone}}{=}
    Q(s, \pi(s))-\epsilon'(s, a).
\end{align*}
Therefore, all 4 sets of constraints are satisfied which proves feasibility.
Finally, given Proposition \ref{lm.rp_approx},
$R'$ is feasible for
\eqref{prob.reward_poisoning_attack} with $\targetpi=\pi$ and therefore $R', \pi$ are feasible for
\eqref{prob.reward_design.approx}.
\end{proof}
\begin{lemma}{Lemma 7 in \citep{ma2019policy}}\label{lm.ma2019policy_lemma_7}
    For arbitrary reward functions $R_1$ and $R_2$,
	\begin{align*}
    (1-\gamma) \cdot 
    \norm{Q^{*, R_1} - Q^{*, R_2}}_{\infty}
	 \le
	||R_1 - R_2||_\infty 
	\end{align*}
\end{lemma}
\begin{corollary}\label{cor.lower_bound}
Let $\widehat{R}^{\pi}$ be the solution to 
the optimization problem \eqref{prob.reward_poisoning_attack} with $\targetpi=\pi$.
Define $\maxgapqval$ as in \eqref{eq.q_gap}.
The following holds.
\begin{align*}
    \norm{\overline{R} -\widehat{R}^{\pi}}_{2} \ge 
    \frac{1-\gamma}{2}\cdot\maxgapqval
\end{align*}
\end{corollary}\begin{proof}
    Define $s_{\max}$ as 
    \begin{align*}
    \argmax_{s\in \posStates^\pi}
    \big(
    Q^{*, \overline{R}}(s, \optpi(s)) - 
    Q^{*, \overline{R}}(s, \pi(s))
    \big).
\end{align*}
    
    Since $s_{\max}\in \posStates^\pi$ and $\pi$ is optimal in $\widehat{R}^{\pi}$,
    it is clear that
    $Q^{*, \widehat{R}^{\pi}}(s_{\max}, \pi(s_{\max})) \ge 
    Q^{*, \widehat{R}^{\pi}}(s_{\max}, \optpi(s_{\max}))$. However, 
    $Q^{*, \overline{R}}(s_{\max}, \pi(s_{\max}))
    \le Q^{*, \overline{R}}(s_{\max}, \optpi(s_{\max})) -
    \maxgapqval
    $. 
    Summing up the two inequalities,
    \begin{align*}
        \maxgapqval &\le 
     Q^{*, \overline{R}}(s_{\max}, \optpi(s_{\max}))
     - 
     Q^{*, \overline{R}}(s_{\max}, \pi(s_{\max})) + 
     Q^{*, \widehat{R}^{\pi}}(s_{\max}, \pi(s_{\max})) 
      -  
      Q^{*, \widehat{R}^{\pi}}(s_{\max},\optpi(s_{\max}))
     \\&=
     \left[
      Q^{*, \overline{R}}(s_{\max}, \optpi(s_{\max}))
      - 
            Q^{*, \widehat{R}^{\pi}}(s_{\max},\optpi(s_{\max}))
     \right]
     + \left[
          Q^{*, \widehat{R}^{\pi}}(s_{\max}, \pi(s_{\max})) 
    - 
         Q^{*, \overline{R}}(s_{\max}, \pi(s_{\max})) 
     \right]
     \\&\le 
     2\cdot\norm{Q^{*, \overline{R}} - 
     Q^{*, \widehat{R}^{\targetpi}}
     }_{\infty}
     \\&\overset{(i)}{\le} 
     \frac{2}{1-\gamma}\cdot\norm{
     \overline{R} - \widehat{R}^{\targetpi}
     }_{\infty}
     \\&\le 
     \frac{2}{1-\gamma}\cdot\norm{
     \overline{R} - \widehat{R}^{\targetpi}
     }_{2}.
    \end{align*}
    where $(i)$ follows from Lemma \ref{lm.ma2019policy_lemma_7}.
\end{proof}


\subsection{Proof of Theorem
\ref{thm.general_mdp.charact.bounds.opt_const}}

To prove Theorem \ref{thm.general_mdp.charact.bounds.opt_const} we will utilize the following lemma.

\begin{lemma}\label{lm.general_mdp.charact.bounds.opt_const}
Let $\pi\in \PiDet$. be an arbitrary deterministic policy.
Define
$\scoregap$ as
$\scoregap = \score^{\optpi, \overline{R}} - \score^{\pi, \overline{R}}$ and let $\widehat{R}$ be the solution to the optimization problem
\eqref{prob.reward_poisoning_attack} with $\targetpi=\pi$. The following holds:
\begin{align*}
    \frac{1-\gamma}{2}\scoregap \le 
    \norm{\widehat{R} - \overline{R}}_2
    \le 
    \frac{1}{\occstate_{\min}^{\pi}} \cdot \scoregap + \frac{\epsilon}{\occstate_{\min}}\sqrt{|S|\cdot|A|}.
\end{align*}
\end{lemma}
\begin{proof}
~\\
{\bf Upper bound: } We prove the upper bound in a constructive manner, using the reward vector $R'$ as defined in Lemma \ref{lm.constructive_answer}.
~\\
Setting $\delta^{\pi}$ as in Lemma 
\ref{lm.constructive_answer}, 
 the cost of  modifying $\overline{R}$ to $R'$ is bounded by:
\begin{align*}
    \norm{R' - \overline{R}}_2 &\le \norm{\delta^\pi}_2 + 
    \norm{\epsilon'_{\pi}}_2
    \\&\le 
    \norm{\delta^\pi}_2 + 
    \frac{\epsilon}{\occstate_{min}} \cdot \sqrt{|S| \cdot |A|}
\end{align*}
It remains to bound the term $\norm{\delta^\pi}_2 $.
From Lemma \ref{lm.score_diff_q_value} and the definition of $\pi$, we have:
\begin{align*}
    \score^{\pi, \overline{R}} - \score^{\optpi, \overline{R}} &= \sum_{s} \occstate^{\pi}(s)
    \cdot 
    \big(
    Q^{\optpi, \overline{R}}(s, \pi'(s))
    -Q^{\optpi, \overline{R}}(s, \optpi(s))
    \big)
    \\
     &=\sum_{s\in \posStates^\pi}
     \occstate^{\pi}(s)
    \big(
    Q^{\optpi, \overline{R}}(s, \pi'(s))
    -Q^{\optpi, \overline{R}}(s, \optpi(s))
    \big)
    \\
    &=\sum_{s\in \posStates^\pi}
    \occstate^\pi(s)\cdot \delta^\pi(s)
    \\
    &\ge\sum_{s\in \posStates^\pi}
    \occstate_{\min}^\pi \cdot \delta^\pi(s)
    \\
    &\overset{(i)}=
    \sum_{s}
    \occstate_{\min}^\pi \cdot \delta^\pi(s)
    \\
    &=
    \occstate^\pi_{\min}\cdot \norm{\delta^\pi}_1
    \\&\ge \occstate^{\pi}_{min} \cdot \norm{\delta^\pi}_2,
\end{align*}
where $(i)$ follows from the fact that $\delta^\pi(s)=0$ for $s\notin \posStates^\pi$.
We can therefore conclude that
\begin{align*}
    \norm{\widehat{R} - \overline{R}}_2 \overset{(i)}{\le}\norm{R' - \overline{R}}_2 &\le \frac{1}{\occstate_{min}^{\pi}} \cdot \left [ \score^{\optpi, \overline{R}} - \score^{\pi, \overline{R}} \right] + \frac{\epsilon}{\occstate_{min}} \cdot \sqrt{|S| \cdot |A|}
    \\&=
        \frac{1}{\occstate_{\min}^{\pi}} \cdot \scoregap + \frac{\epsilon}{\occstate_{\min}}\sqrt{|S|\cdot|A|}.
\end{align*}
where $(i)$ follows from the fact that $R'$ is feasible in the optimization problem \eqref{prob.reward_poisoning_attack}
with $\targetpi = \pi$
by Lemma \ref{lm.score_diff_q_value} while $\widehat{R}$ is optimal for the problem.
~\\
\textbf{Lower Bound: }
Given Corollary \ref{cor.lower_bound}, 
\begin{align*}
    \norm{\overline{R} - \widehat{R}}_2\ge
    \frac{1-\gamma}{2}\cdot \maxgapqval.
\end{align*}
Note however that Lemma \ref{lm.score_diff_q_value} implies 
\begin{align*}
\score^{\optpi, \overline{R}} -
    \score^{\pi, \overline{R}}
 &=
 \sum_{s}\occstate^{\pi}(s)\big(
 Q^{\optpi, \overline{R}}(s, \optpi(s)) - 
 Q^{\optpi, \overline{R}}(s, \pi(s))
 \big)
 \\&=
 \sum_{s\in \posStates^\pi}\occstate^{\pi}(s)\big(
 Q^{*, \overline{R}}(s, \optpi(s)) - 
 Q^{*, \overline{R}}(s, \pi(s))
 \big)
 \\&\le 
 \sum_{s\in \posStates^\pi}\occstate^\pi(s) \maxgapqval = \maxgapqval,
\end{align*}
which proves the claim.
\end{proof}

We can now prove Theorem
\ref{thm.general_mdp.charact.bounds.opt_const}.

~\\
{\bf Statement:} 
{\em 
The relative value $\relval$ is bounded by
\begin{align*}
\alpha_{\score} \cdot \minscoregap \le \relval \le  \beta_{\score} \cdot \minscoregap + \frac{\epsilon}{\occstate_{\min}} \cdot \sqrt{|S| \cdot |A|},
\end{align*}
where $\alpha_{\rho} = \left (\lambda + \frac{1-\gamma}{2} \right )$ and $\beta_{\score} = \left (\lambda + \frac{1}{\occstate_{\min}} \right )$.
}
\begin{proof}
~\\
Given the Lemma \ref{lm.general_mdp.charact.bounds.opt_const}, it is clear that since $\min_{\pi} \scoregap = \minscoregap$,
setting $(\widehat{R}_2, \pi_2)$ as the solution to \eqref{prob.reward_design.approx},
\begin{align*}
    \frac{1-\gamma}{2}\cdot\minscoregap\le \norm{\overline{R} - \RsolB}_2 \le 
    \frac{1}{\occstate_{\min}} \cdot \minscoregap + \frac{\epsilon}{\occstate_{\min}}\cdot \sqrt{|S|\cdot|A|}
\end{align*}
Now note that for any $\pi\in \allowedpi$, including
$\pi_2$,
$\score^{\optpi, \overline{R}} - 
    \score^{\pi, \overline{R}}\ge \minscoregap$, which proves the lower bound. As for the upper bound,
    setting
    $\widehat{R}^{\optallowpi}$ as the solution to 
    \eqref{prob.reward_poisoning_attack} with 
    $\targetpi=\optallowpi$,
    \begin{align*}
        \Phi &= 
        \norm{\overline{R} - \widehat{R}_2}_2 + 
        \lambda\cdot
        [
        \score^{\optpi, \overline{R}}
        -\score^{\pi_2, \overline{R}}
        ]
        \\&\overset{(i)}{\le} 
        \norm{\overline{R} - \widehat{R}^{\optallowpi}}_2 + 
        \lambda\cdot
        [
        \score^{\optpi, \overline{R}}
        -\score^{\optallowpi, \overline{R}}
        ]
        \\&\overset{(ii)}{\le}
        \beta_{\score} \cdot \minscoregap + \frac{\epsilon}{\occstate_{\min}} \cdot \sqrt{|S| \cdot |A|},
    \end{align*}
    where $(i)$ follows from the optimality of
    $(\widehat{R}_2, \pi_2)$ and $(ii)$ follows from
    Lemma \ref{lm.general_mdp.charact.bounds.opt_const} and
    the fact that
    $\score^{\optpi, \overline{R}}
        -\score^{\optallowpi, \overline{R}}
        =\minscoregap$.
\end{proof}


\subsection{Proof of Theorem
\ref{thm.general_mdp.charact.bounds.qgreedy}}

To prove Theorem \ref{thm.general_mdp.charact.bounds.qgreedy}, we utilize the following result. 

\begin{lemma}\label{lm.general_mdp.charact.bounds.qgreedy}
Let $\widehat{R}$ be the solution to
\eqref{prob.reward_poisoning_attack} with $\targetpi=\pi$. The following holds:
\begin{align*}
    \frac{1-\gamma}{2} \cdot \maxgapqval \le \norm{\overline{R} - \widehat{R}}
    \le \sqrt{|S|} \cdot \maxgapqval + \frac{\epsilon}{\occstate_{min}} \cdot \sqrt{|S| \cdot |A|}.
\end{align*}
\end{lemma}
\begin{proof}
The lower bound follows from Corollary \ref{cor.lower_bound}.
As for the upper bound,
define $R'$ as in Lemma \ref{lm.constructive_answer} and note that
\begin{align*}
  \norm{\widehat{R} - \overline{R}}_2 
  \overset{(i)}{\le}
  \norm{R' - \overline{R}}_2 
  \le \norm{\delta^\pi}_2 + 
  \norm{\epsilon'_{\pi}}_2
  \le \sqrt{|S|} \cdot \maxgapqval + \frac{\epsilon}{\occstate_{min}} \cdot \sqrt{|S| \cdot |A|}.
\end{align*}
where $(i)$ follows from the fact that $R'$ is feasible in the optimization problem \eqref{prob.reward_poisoning_attack}
with $\targetpi = \pi$
by Lemma \ref{lm.score_diff_q_value} while $\widehat{R}$ is optimal for the problem.
\end{proof}

We can now prove Theorem
\ref{thm.general_mdp.charact.bounds.qgreedy}.

~\\
\textbf{Statement: }
\emph{
The relative value $\relval$ is bounded by
\begin{align*}
&\alpha_Q \cdot \minmaxgapqval \le 
    \Phi
    \le 
    \beta_Q \cdot \minmaxgapqval
    + \frac{\epsilon}{\occstate_{\min}}\sqrt{|S| \cdot |A|},
\end{align*}
where $\alpha_Q = \left (\lambda\cdot \occstate_{\min} + \frac{1-\gamma}{2}\right )$ and $\beta_Q = \left( \lambda + \sqrt{|S|} \right)$.
}
\begin{proof}
~\\
\textbf{Lower bound:}
Note that for any admissible $\pi$,
by Lemma \ref{lm.score_diff_q_value},
\begin{align*}
    \score^{\optpi, \overline{R}} - 
    \score^{\pi, \overline{R}}
    &=
    \sum_{s}\occstate^{\pi}(s)\big(
        Q^{*, \overline{R}}(s, \optpi(s)) - 
        Q^{*, \overline{R}}(s, \pi(s))
    \big)
    \\&\ge 
    \occstate_{\min}\cdot
    \maxgapqval
    \ge
    \occstate_{\min}\cdot 
    \minmaxgapqval.
\end{align*}
the claim now follows from Lemma
\ref{lm.general_mdp.charact.bounds.qgreedy} and the definition of $\relval$.
~\\
\textbf{Upper bound:} 
Recall that $\piqg=\argmin_{\pi\in \allowedpi} \maxgapqval$.
Using Lemma \ref{lm.score_diff_q_value},
\begin{align*}
    \score^{\optpi, \overline{R}} - 
    \score^{\piqg, \overline{R}}
    &=
    \sum_{s}\occstate^{\piqg}(s)\big(
        Q^{*, \overline{R}}(s, \optpi(s)) - 
        Q^{*, \overline{R}}(s, \piqg(s))
    \big)
    \\&=
    \sum_{s\in \posStates^{\piqg}}
    \occstate^{\piqg}(s)\big(
        Q^{*, \overline{R}}(s, \optpi(s)) - 
        Q^{*, \overline{R}}(s, \piqg(s))
    \big)
    \\&\le 
    \sum_{s\in \posStates^{\piqg}}
    \occstate^{\piqg}(s)\cdot
    \minmaxgapqval
    \\&=
    \minmaxgapqval.
\end{align*}
Similiar to the proof of
Theorem \ref{thm.general_mdp.charact.bounds.opt_const},
the claim now follows form Lemma
\ref{lm.general_mdp.charact.bounds.qgreedy}, the definition of $\relval$ and optimality of $\pi_2$. Concretely, setting
$\widehat{R}^{\piqg}$ as the solution to \eqref{prob.reward_poisoning_attack}
with $\targetpi=\piqg$,
\begin{align*}
        \Phi &= 
        \norm{\overline{R} - \widehat{R}_2}_2 + 
        \lambda\cdot
        [
        \score^{\optpi, \overline{R}}
        -\score^{\pi_2, \overline{R}}
        ]
        \\&\le
        \norm{\overline{R} - \widehat{R}^{\piqg}}_2 + 
        \lambda\cdot
        [
        \score^{\optpi, \overline{R}}
        -\score^{\piqg, \overline{R}}
        ]
        \\&\le
        \beta_{Q} \cdot \minmaxgapqval + \frac{\epsilon}{\occstate_{\min}} \cdot \sqrt{|S| \cdot |A|}.
    \end{align*}
\end{proof}

}
}
{}
\end{document}